\documentclass{article}

\PassOptionsToPackage{square,numbers,sort&compress}{natbib}
\usepackage[preprint]{neurips_2026}

\usepackage{titlesec}
\titlespacing*{\paragraph}{\parindent}{0.25ex}{1ex}
\titlespacing*{\section}{0pt}{3pt}{3pt}
\titlespacing*{\subsection}{0pt}{3pt}{3pt}

\usepackage{floatrow}
\usepackage[utf8]{inputenc} %
\usepackage[T1]{fontenc}    %
\usepackage{hyperref}       %
\hypersetup{
    colorlinks,
    linkcolor={red!50!black},
    citecolor={blue!50!black},
    urlcolor={blue!80!black}
}
\usepackage{url}            %
\usepackage{booktabs}       %
\usepackage{amsfonts}       %
\usepackage{amsmath}
\usepackage{nicefrac}       %
\usepackage{microtype}      %
\usepackage{graphicx}
\usepackage{cancel}
\usepackage{color,soul}
\usepackage{mathtools}

\usepackage{wrapfig}
\usepackage{amssymb}
\usepackage{pgf}
\usepackage{array}
\usepackage{placeins}
\usepackage{tabularx}
\usepackage{makecell}
\usepackage[ruled,vlined]{algorithm2e}
\usepackage{bbold}
\usepackage{amsthm}
\usepackage{caption}
\usepackage{subcaption}
\usepackage{paralist}
\usepackage{multirow}
\usepackage{booktabs}
\usepackage{minitoc}
\usepackage{enumitem}
\usepackage{xcolor}
\usepackage{bm}
\usepackage[most]{tcolorbox}
\usepackage[disable,color=red!20,textsize=small]{todonotes}
\usepackage{marginnote}
\usepackage{etoolbox}
\usepackage{tcolorbox}

\newtoggle{lmargin}
\newcommand{\alternatingtodo}[2][]{%
    \iftoggle{lmargin}%
    {%
        \todo[#1]{#2}%
        \togglefalse{lmargin}%
    }{%
        {%
            \let\marginpar\marginnote%
            \reversemarginpar%
            \todo[#1]{#2}%
        }%
        \toggletrue{lmargin}%
    }%
    \ignorespaces%
}

\usepackage{empheq}
\definecolor{light-gray}{gray}{0.95}
\definecolor{lightgreen}{HTML}{90EE90}
\definecolor{og}{HTML}{4B277F}

\renewcommand\footnotemark{}

\newtheorem{theorem}{Theorem}
\newtheorem*{theorem*}{Theorem}
\newtheorem{definition}[theorem]{Definition}

\newtheorem{corollary}[theorem]{Corollary}

\SetKwComment{Comment}{$\triangleright$}{}

\newtheorem{proposition}[theorem]{Proposition}

\newtheorem{example}[theorem]{Example}

\newcommand{\R}{\mathbb{R}}
\newcommand{\E}{\mathbb{E}}
\newcommand{\tr}{\operatorname{Tr}}

\newcommand{\diag}{\operatorname{diag}}
\newcommand{\Loss}{\mathcal{L}}

\newcommand{\Fre}{\mathsf{F}}

\newcommand{\eps}{\varepsilon}

\newcommand{\tcmp}{$T_{\text{CMP}}$}
\newcommand{\tadd}{$T_{\text{ADD}}$}

\definecolor{oran}{HTML}{DEAA55}
\definecolor{purp}{HTML}{7C6B99}

\hypersetup{colorlinks=true, citecolor=primary, linkcolor=primary, urlcolor=primary}

\providecommand{\aut}[1]{\textbf{#1}}
\providecommand{\af}[1]{{\small #1}}

\definecolor{primary}{HTML}{7B2D26} %
\definecolor{secondary}{HTML}{B8973A} %
\providecommand{\afn}[1]{\textcolor{primary}{$^{#1}$}}

\newcommand{\goodfiremark}{\raisebox{0.5pt}{\hspace{0.5mm}\includegraphics[height=6pt]{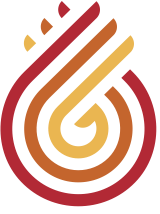}}}

\newcommand{\goodfireaff}{%
  \includegraphics[height=14pt]{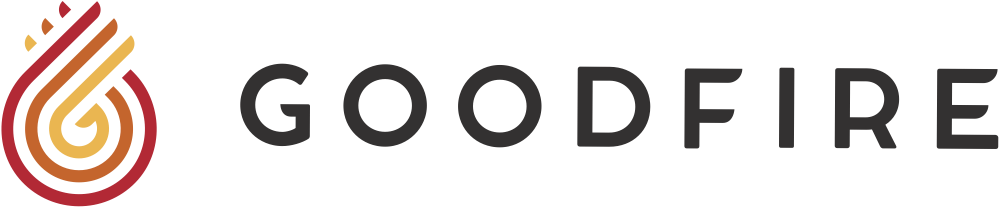}
}

\newcommand{\authorentry}[2]{\aut{#1}\afn{#2}}
\newcommand{\afflabel}[2]{\afn{#1}\af{#2}}
\newcommand{\authorsep}{\quad}

\newtoggle{goodfireonly}
\togglefalse{goodfireonly}  %
\newcommand{\extline}[1]{\iftoggle{goodfireonly}{}{#1}}
\newcommand{\extaff}[1]{\iftoggle{goodfireonly}{}{#1}}

\newcommand{\paperauthors}{%
\authorentry{Jing Huang}{\goodfiremark\extaff{,a}} \authorsep
\authorentry{Daniel Wurgaft}{\goodfiremark\extaff{,a}} \authorsep
\authorentry{Rachit Bansal}{\extaff{b}} \authorsep
\authorentry{Laura Ruis}{\extaff{c}} \\
\vspace{2pt}
\authorsep
\authorentry{Naomi Saphra}{\extaff{b}} \authorsep
\authorentry{David Alvarez-Melis}{\extaff{b}} \authorsep
\authorentry{Andrew Lampinen}{\extaff{d}} \\
\authorentry{Christopher Potts}{\extaff{a}} \authorsep
\authorentry{Ekdeep Singh Lubana}{\goodfiremark} \authorsep
\vspace{4pt}\\
\vspace{3.5mm}
\goodfireaff \vspace{-5pt}\\
\extline{%
  \afflabel{a}{Stanford University} \authorsep
  \afflabel{b}{Kempner Institute at Harvard University} \authorsep
  \afflabel{c}{MIT} \authorsep
  \afflabel{d}{Anthropic} \authorsep
}
}

\author{\paperauthors}

\title{Why Larger Models Learn More: Effects of Capacity, Interference, and Rare-Task Retention}

\begin{document}

\maketitle

\vspace{-10pt}
\begin{abstract}
Larger models learn tasks smaller models do not.
What drives this phenomenon?
We develop a simple phenomenological argument that power-law scaling already suggests that a larger model will be able to learn a part of the data distribution that a smaller model fails to learn, even with infinite training data.
To validate this claim and identify its causes, we study the effects of model scaling on a synthetic setup consisting of a mixture of tasks that show monotonic scaling curves.
The results point to a data-induced competition over resources (neurons). Specifically, smaller models allocate their neurons to high frequency or low complexity tasks, and so they learn solutions that perform poorly on rare and complex tasks. Moreover, this happens even when solutions capable of expressing the desired task exist.
We then assess how a larger model circumvents this data-centric bottleneck, finding that it traces to a reduced interference mechanism: larger models can allocate enough resources to common tasks that the gradient updates for those tasks become weak, which means that they do not overwrite rare-task features as they slowly accumulate.
Finally, to further validate these claims, we pretrain OLMo models (4M to 4B parameters) on novel tasks of varying frequency and complexity. The results mirror those from our synthetic data experiments: only the larger OLMo models learn the infrequent and complex tasks, and these larger models embed more task features in their representations and show less gradient interference between tasks.
Overall, we offer a data-centric account of why larger models learn tasks that smaller models fail to. 
This helps explain why larger models are better in practice, and it can inform practical questions concerning model sizing and training data mixtures.
\end{abstract}

\section{Introduction}
\label{sec:intro}

Modern machine learning is celebrated for its massive generalist models, which are capable of handling arbitrary inputs in diverse and complex environments~\cite{singh2025openai, mythos, gemini3, deepseekv3, deepseekv4, team2026kimi,kwameasuring, glazer2024frontiermath, arc3, merrill2026terminal}.
Based on the empirical finding that larger models often excel where smaller\footnote{We use the terms ``larger'' and ``smaller'' informally here but develop a precise relational definition of these terms in Sec.~\ref{sec:categorization}.} models show random-chance performance, prior work has claimed that the ability to solve certain critical tasks \textit{only} emerges in larger models~\cite{rsp, oairisk, simeoni2025dinov3, hu2023unlock, wei2022emergent, wei2022emergent_list_of_137, arora2023theory, du2024understanding, wei2023larger}.
Such arguments have fueled the drive towards %
increased scaling.
However, given the large training and inference costs that large models impose, it is worth identifying precisely what marginal benefits are unlocked by larger models and whether scaling parameters is the sole way of realizing those benefits.

Our argument begins from the observation that \textbf{power-law scaling \cite{kaplan2020scaling, hestness2017deep, rosenfeld2019constructive} already 
suggests that there is a regime in which a smaller model fails to learn parts of a data mixture that a larger model succeeds on, even under asymptotic training} (Fig.~\ref{fig:categorization}, Sec.~\ref{sec:categorization}).
This suggests that larger models enjoy a genuine advantage that may allow them to learn task distributions that smaller models will inevitably fail to learn within the same training setup.
Importantly, this is not an argument that larger models are simply more sample efficient \cite{henighan2020scaling, hernandez2021scaling, rae2021scaling, alabdulmohsin2022revisiting, du2024understanding, grattafiori2024llama, hoffmann2022training, pearce2024reconciling}, but rather that smaller models suffer from a more fundamental limitation even under infinite training regimes.

\begin{wrapfigure}[29]{r}{0.5\textwidth}
\vspace{-12pt}
\begin{center}
\includegraphics[width=0.94\linewidth]{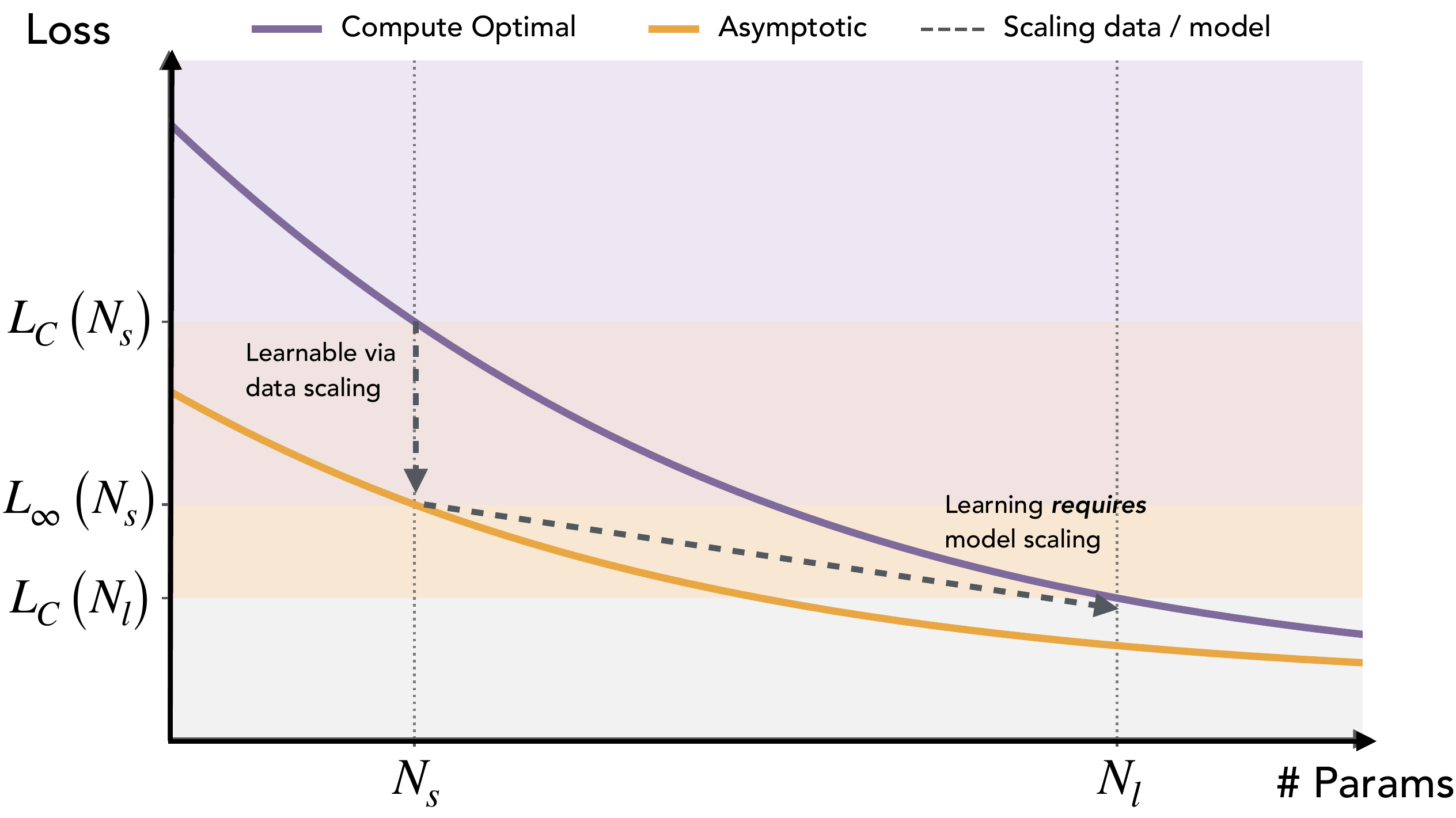}
\end{center}
\vspace{-15pt}
\caption{\label{fig:categorization}\textbf{Learning a part of the distribution requires model scaling.} 
    Compare the loss curves for compute-optimal scaling with the one following an infinite resource regime (labeled asymptotic). 
    The region labeled purple denotes the amount of loss both a smaller model with $N_s$ parameters and a larger model with $N_l$ parameters are able to achieve with respect to a random baseline under finite resources. 
    We call loss reduction accessible to the smaller model under infinite compute, but that a larger model would get at in a more resource efficient manner (i.e., under finite compute), \textbf{{\color{purp!150} learnable via data scaling}}. 
    If there remains a part of the loss that is achieved by the larger model under finite resources, but that a smaller model even under asymptotic data scaling is unable to reach, then we call this part \textbf{{\color{oran!100}learned via model scaling}}. 
    This part of the distribution is explained by a larger model by virtue of its larger size.
}
\end{wrapfigure}
To validate this prediction and identify its causes, we analyze a setting involving a mixture of regression tasks.
In this, we are inspired by much recent work using toy tasks to pinpoint the effects of scaling~\cite{bordelon2024dynamical, bordelon2025feature, bahri2024explaining, lin2024scaling, michaud2024quantization, maloney2022solvable, lubana2024percolation, cagnetta2025learning, cagnetta2025scaling, cagnetta2026deriving, edelman2023pareto}.
Furthermore, all of the individual tasks in our setting are learnable by the models under consideration, capturing the idea that tasks smaller models fail to learn can still be instilled into them via post-training~\cite{lambert2024tulu, hubotter2026reinforcement, deepdistill, shao2024deepseekmath, xin2024deepseek, agarwal2024policy, zhao2026self, tang2025beyond, team2025kimi, blakeney2022reduce}.
Correspondingly, mere expressivity notions are not the issue; instead, the question concerns the ability of these models to learn complex task distributions from data.
These experiments lead to two key findings, as described below.

First, \textbf{scaling enables learning rare and complex tasks (Sec.~\ref{sec:synthetic_larger_rarer}).} 
Our experimental setting defines controlled manipulations of task frequency and complexity.  We present an analytic argument that only larger models will (on average) learn the rare and complex tasks present in this setting, and we verify this analysis experimentally (Fig.~\ref{fig:phases}). 

Second, \textbf{reduced competition for resources enables learning rare and complex tasks (Sec.~\ref{sec:interfere})}.
Here, we extend our formal analysis to show that, upon observation of samples from a rare task, model parameters update, but only larger models, by virtue of having more parameters and hence less gradient interference, are able to retain memory of a previously observed batch of data from a rare task.
Thus, when the next batch of rare-task data comes in, the larger model builds on its prior knowledge, which ultimately leads to success despite the impoverished learning signal. In contrast, the smaller model is forced to start from scratch and consequently fails.
We again verify these findings experimentally in our regression setting (Figs.~\ref{fig:residual_controls_learning} and \ref{fig:competition}).

Finally, we \textbf{validate the above theoretical arguments in real LLMs (Sec.~\ref{sec:real}).} 
Specifically, we pretrain OLMo models (4M to 4B parameters) on the Dolma~v1.7 corpus with completely novel tasks injected at controlled frequency.
We find that only the larger OLMo models are able to learn the infrequent and complex tasks (Sec.~\ref{sec:olmo_loss}). Furthermore, these OLMo models mirror our toy-task models in deeper ways: larger OLMo models have more task features embedded in their representations (Sec.~\ref{sec:olmo_repr}) and show less gradient interference (Sec.~\ref{sec:olmo:gradients}).
Beyond supporting our theoretical claims, these results can provide practical guidance to large-scale model training efforts.

Overall, the data-centric nature of our analysis suggests that understanding why larger models learn more requires not only asking what they can represent, but also what is learnable under gradient-based optimization from a given data mixture.

\section{A Phenomenological Model Predicts Larger Models Learn More}
\label{sec:categorization}

Neural network scaling is known to predictably and monotonically improve loss~\cite{hoffmann2022training, kaplan2020scaling, qiu2025scaling}:
\begin{equation}
    L(N, D) = L_0 + \frac{A}{N^{\alpha}} + \frac{B}{D^{\beta}},
\end{equation}
where $L_0$ denotes the irreducible loss, $A, B$ are constants, and $\alpha, \beta$ are parameter / data exponents ($\alpha \approx 0.46$ and $\beta \approx 0.51$ for Chinchilla-scaling~\cite{hoffmann2022training}).
Training in a compute-optimal manner, i.e., finding the model size and data configuration that helps achieve the minimum loss at a given compute budget $C$, gives us 
\begin{equation*}
    L_{\text{C}}(N) \propto N^{-\gamma},
\end{equation*}
where $\gamma = 0.34$, and $L_{\text{C}}(N)$ denotes the optimum loss achieved when training a model with $N$ parameters under resource constraints. 
The relation shows larger models are expected to achieve a smaller loss.
However, resource-constrained training by itself does not inform what a model can actually express.
\textit{Specifically, even though a smaller model may have a worse compute-optimal loss, we do not know if it is fundamentally incapable of achieving the same loss as the larger model.}
To assess that statement, we must evaluate a model's loss under asymptotic resources (i.e., infinite data):\footnote{We note power-law scaling need not hold asymptotically~\cite{paquette2024, bordelon2025feature}, which is why we call this argument phenomenological. It motivates the subsequent, rigorous claims.} 
\begin{equation*}
    L_{\infty}(N) \propto N^{-\alpha}.
\end{equation*}
If $\alpha > \gamma$, as is the case in practice, we again see gains from merely scaling the model size.
That is, \textit{the asymptotic loss achieved by a larger model is better than the smaller one.} 
This indicates there is a part of the training distribution a smaller model, despite observing infinite data, fails to learn.
Based on this phenomenological argument, we define the following.

\begin{definition}[\textbf{{\color{purp!150} Learnable via data scaling}}]
\label{def:data_scaling}
Consider a target model with $N_l$ number of parameters that we call ``large''.
We say a ``smaller'' model, i.e., for which parameter count $N_s < N_l$, can recover the loss of a larger model via data scaling if $L_{C}(N_s) - L_C(N_l) > 0$, but $L_{\infty}(N_s) - L_C(N_l) < 0$.
\end{definition}
Def.~\ref{def:data_scaling} thus captures the scenario put forward in Sec.~\ref{sec:intro}.
That is, the smaller model may in fact be just \textbf{undertrained}: the larger model learns more sample efficiently and reduces loss faster, but a smaller model can eventually catch up~\cite{henighan2020scaling, hernandez2021scaling, rae2021scaling, alabdulmohsin2022revisiting, du2024understanding, grattafiori2024llama, hoffmann2022training, pearce2024reconciling}.
Correspondingly, the marginal ability of a larger model to explain the data distribution (i.e., the loss) can be recovered by a smaller model merely observing more data. 
Nevertheless, there exist regimes where data scaling will not suffice, as described next.

\begin{definition}[\textbf{{\color{oran!100} Learnable via model scaling}}]
\label{def:model_scaling}
Consider a target model with $N_l$ number of parameters that we call ``large''.
For a small scalar value $\epsilon$, we define $N_s^*(\epsilon)$ as the largest ``small'' model if $L_{\infty}(N_s^*(\epsilon)) - L_C(N_l) > \epsilon$.
That is, even asymptotically, the smallest model never reaches the same loss as the large model.
Correspondingly, for a given model size $N$, we call it  ``small'' if $N < N_s^*(\epsilon)$ and say recovering the loss of the larger model requires model scaling.
\end{definition}
This latter scenario thus captures the case where, when two models with parameter counts $N_s, N_l$, with $N_s < N_l$, are trained, there is truly a marginal improvement for explaining the data that can be attributed to the larger model having more parameters.
This is the most interesting case that warrants further study: what is it about the data that only a larger model can learn, such that the smaller model cannot, even after observing infinite data?
How precisely does having more parameters aid this learning?
We aim to answer these questions in the following sections.

\section{Scaling Allows Learning Rare Tasks by Reducing Gradient Interference}
\label{sec:synthetic}

Our phenomenological argument in Sec.~\ref{sec:categorization} motivates the claim that larger models are likely to learn a part of the data distribution smaller models will fail to learn.
We next aim to get more concrete about this claim.
Specifically, we exploit the fact that our argument is merely based on monotonic (power-law) scaling---a phenomenon even synthetic tasks can recapitulate~\cite{bordelon2024dynamical, bordelon2025feature, bahri2024explaining, lin2024scaling, michaud2024quantization, maloney2022solvable, cagnetta2025learning, cagnetta2025scaling, cagnetta2026deriving, edelman2023pareto}.
Such tasks have in fact been used in prior work to make accurate predictions about scaling behavior for large-scale models~\cite{qiu2025scaling, paquette2024}.
We thus follow this line of work and develop a multi-task learning setup that helps assess which tasks a larger model can learn but a smaller model cannot.
We generalize our claims to an off-the-shelf language model pretraining pipeline~\cite{olmo20242olmo2furious} in Sec.~\ref{sec:real}, finding the core hypotheses derived out of this toy setting hold true on even a large-scale training pipeline.

\paragraph{Data.}
We consider a multi-task learning setup where samples are drawn from a mixture of $K$ linear regression tasks. 
Specifically, the $k^{\text{th}}$ task is assumed to appear with \textit{frequency} $\pi_k >0$, such that $\sum_k \pi_k = 1$, and has \textit{covariance} $C_k = B_k \Lambda_k B_k^\top = \sum_{j \ge 1} \lambda_{k,j}\, b_{k,j} b_{k,j}^\top$.
Here, the ``feature matrix'' $B_k=[b_{k,1},b_{k,2},\ldots]$ is assumed to have orthonormal columns; $\Lambda_k = \diag(\lambda_{k,1},\lambda_{k,2},\ldots)$ with $\lambda_{k,1}\ge \lambda_{k,2}\ge \cdots \ge 0$; and different tasks occupy orthogonal blocks, i.e., $B_k^\top B_\ell = 0$ for $k\neq \ell$.
If the spectrum $\{\lambda_{k,j}\}$ decays slowly, the task requires more directions for producing the corresponding target---we can thus compare the \textit{relative complexity} of two tasks by comparing the rate at which their spectra decay.
Compared to prior work studying theory of scaling laws based on toy regression tasks, we emphasize that our setup involves the learning of multiple tasks simultaneously.

\paragraph{Teacher / Student Models.}
For a given input $x\sim\mathcal{N}(0,I)$, the teacher for task $k$ is defined as $y_k = \Lambda_k^{1/2} B_k^\top x$. 
The student uses a shared width-$N$ encoder $U\in\R^{d\times N}$, $U^\top U=I$, with projector $P_U = UU^\top$, together with task-specific linear decoders $D_k$ to discern between tasks.
Correspondingly, the student prediction is $\hat y_k = D_k U^\top x$.
The total mixture loss is the weighted sum $\Loss_N(U)=\sum_{k=1}^K \pi_k\ell_k(U)$, where $\ell_k(U)=\E\big[\|y_k-D_kU^\top x\|_2^2\big]$ is loss of the $k^{\text{th}}$ task.
Note that herein, since the optimal decoder admits a closed-form solution $D_k^* = \Lambda_k^{1/2} B_k^\top U$, we solely analyze the dynamics of the encoder, which produces features used by the student for making predictions.

\subsection{Larger Models Learn Rarer, More Complex Tasks}
\label{sec:synthetic_larger_rarer}
In order to narrow down a mechanism that explains \textit{how} larger models may be able to learn more, we must first identify precisely \textit{what} it is that a larger model learns but a smaller one fails to.
We begin with answering this question in our toy setup.

\begin{theorem}[Features are Learned in Order of Utility]
\label{thm:capacity}
For a given $U$, the mixture loss reduces to $L_N(U) = \tr(M) - \tr(U^\top M U)$, where $M := \sum_{k=1}^K \pi_k C_k$. 
Hence, a width-$N$ minimizer spans the top-$N$ eigenspace of $M$, whose eigenvalues are defined by the weighted per-task spectra: 
\begin{equation}
u_{k,j} := \pi_k \lambda_{k,j}.
\end{equation}
Thus, the optimal encoder keeps the $N$ features $(k,j)$ with largest $u_{k,j}$---we call these terms utilities. 
This implies if $n_k(N)$ denotes the number of retained features from task $k$, then $\ell_k^*(N) = \sum_{j>n_k(N)} \lambda_{k,j}$.
Conversely, the minimum width at which a model learns at least $m$ features for all tasks is $N^\ast(m) = \min\!\bigl\{\, N \,:\, n_k(N) \ge m \bigr\}$.
\end{theorem}

\begin{figure*}[!t]
    \centering
    \vspace{-15pt}
    \includegraphics[width=\linewidth]{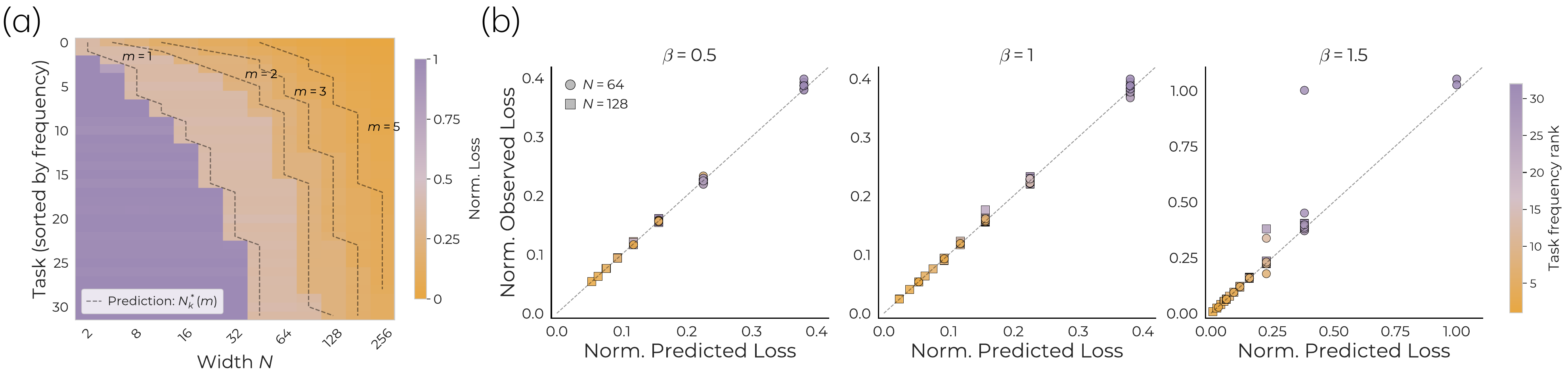}
    \vspace{-10pt}
    \caption{\label{fig:phases}\textbf{Feature Utility Predicts Learning Order.} We train students of varying width on a mixture of $K=32$ regression tasks with power-law task frequencies ($\beta$) and plot per-task loss (normalized by mean predictor).
    (a) Empirical phase diagram where task features ($\beta=1.0$) are retained as a function of width and task frequency, matching our prediction. 
    (b) Loss matches the analytic prediction from Theorem~\ref{thm:capacity} across task-frequency exponents. 
    Overall, we see that increasing width preferentially improves low-frequency tasks because it allows the model to retain lower-utility features.
    \vspace{-12pt}
    }
\end{figure*}

In the context of our toy task, the statement above helps answer the question ``what does width buy?'' by defining a concrete ranking rule for feature learning.\footnote{This claim can also be seen as a static ordering rule that local optima visited by a model during training will be expected to dynamically follow in its saddle-to-saddle dynamics~\cite{zhang2025saddle, abbe2023sgd, jacot2021saddle, kunin2025alternating}}
Specifically, it says a larger model, asymptotically, learns exactly those features whose utilities are lower than those of the features learned by a smaller model.
\textit{This implies if a task is observed infrequently or it involves several features, e.g., if its spectrum decays very slowly, then (on average) only a larger model will learn it.}

\paragraph{Verification.} We verify the claim above by training our student model on a mixture of $K=32$ tasks, using the Adam optimizer for $100$K steps (the loss does not improve beyond this budget even when trained up to 10$\times$ longer; see Fig.~\ref{fig:E-long-horizon-phase}).
We use a power-law prior $k^{-\beta}$ to define task frequencies, and a power-law per-task spectrum $\lambda_{k,j} \propto j^{-\alpha}$.
For simplicity of visualization, we let $\alpha=2$ be shared across tasks and only vary task frequencies by changing $\beta$ (see App.~\ref{app:complexity_sweep} for experiments modulating complexity by varying $\alpha$). 
Results are reported in Fig.~\ref{fig:phases} (also see App.~\ref{app:frequency_sweeps} for further results).
We find (a) the per-task loss and (b) the overall residual loss predictably reduce with model width.
Critically, we see larger models learn infrequent tasks better than smaller ones.

\subsection{Scaling Reduces Interference and Allows for Retention of Rare Task Observations}\label{sec:interfere}

While the argument above---i.e., a larger model learns low utility, infrequent features---is intuitively reasonable, it is critical to note that if the frequency at which a task or its features are seen is very low, then, regardless of size, there is a statistical bottleneck here that a model needs to circumvent.
For example, in the experiments shown in Fig.~\ref{fig:phases}b, a model must learn a task that constitutes merely $0.25$\% of observations.
We next analyze how width helps surmount this challenge. 
To this end, note that for the $k^{\text{th}}$ task, the Riemannian gradient is $G_k(U) = 2(I-P_U)C_kU$, and hence the mixture gradient is $\dot U = 2(I-P_U)MU$.
We then have the following claim.

\begin{theorem}[Residual Controls Learning]
\label{thm:interference}
Let $\Fre\subseteq [K]$ denote the common or frequent tasks. 
Define these tasks' weighted covariance $M_{\Fre} := \sum_{k\in \Fre} \pi_k C_k$ and residual signal $\delta_{\Fre}(U) := \tr\!\big((I-P_U)M_{\Fre}\big)$.
Then, the aggregate common-task gradient $G_{\Fre}(U) = 2(I-P_U)M_{\Fre}U$ obeys the bound
\begin{equation}
    \|G_{\Fre}(U)\|_F \le 2\sqrt{\lambda_1(M_{\Fre})\,\delta_{\Fre}(U)}.
\end{equation}
\end{theorem}

The statement above says a set of tasks move the model only through the part of their covariance that is \emph{not already explained} by the current representation, i.e., the residual $\delta_{\Fre}(U)$.
Correspondingly, once the high-utility common-task features have been learned, their updates become weak (i.e., low norm).
\textit{This leaves any spare width available to rare tasks.}
More precisely, let $\mu_1^{\Fre}\ge \mu_2^{\Fre}\ge \cdots$ be the eigenvalues of $M_{\Fre}$. 
The best width-$N$ representation for the common tasks alone leaves residual $\delta_{\Fre}^*(N)=\sum_{i>N}\mu_i^{\Fre}$.
Then, via Theorem~\ref{thm:interference}, we get the following.

\begin{corollary}[Width-Scaling Reduces Competition]
\label{cor:finite-rank-main}
Define $N_{\Fre}(\eps):=\min\left\{N:\sum_{i>N}\mu_i^{\Fre}\le \eps\right\}$. 
For every $N\ge N_{\Fre}(\eps)$, there exists an encoder for which $\delta_\Fre^*(N) \le \epsilon$ and $\|G_{\Fre}(U)\|_F \le 2 \sqrt{\mu_1^\Fre \epsilon}$.
\end{corollary}

That is, once $N\gtrsim N_{\Fre}(\eps)$, the model contains enough resources that can be allocated to the common tasks, rendering the gradient towards them weak.
This makes the remaining resources available to rare tasks.
However, even once interference is weak enough for a rare task to be learned, it is unclear whether gradient descent can actually consolidate that signal across its infrequent observations. 
To this end, we next characterize the local condition under which a specific rare feature can pull the model towards itself, without forcing the forgetting of well-learned tasks.
Specifically, assume we wanted to learn a rare rank-one task $C_r = \lambda_r b_r b_r^\top$ orthogonal to the common block. 
Let $U_F^{(N)}$ be top-$N$ eigenspace of $M_F$ with eigenvalues $\mu_1^F \geq \mu_2^F \geq \cdots$.
Then, we have the following claim.

\begin{proposition}[Interference Reduces via Scaling]
\label{prop:invasion}
The common-task solution $U_{\Fre}^{(N)}$ is stable against direction $b_r$, i.e., common tasks' loss does not grow by learning of $b_r$, iff $\pi_r\lambda_r < \mu_N^{\Fre}$.
Thus, the critical width at which $b_r$ gets learned is $N_r^{\text{crit}} := \min\{N : \mu_N^{\Fre} \le \pi_r\lambda_r\}$.
\end{proposition}

\begin{wrapfigure}[20]{r}{0.45\textwidth}
\vspace{-15pt}
\begin{center}
\includegraphics[width=0.94\linewidth]{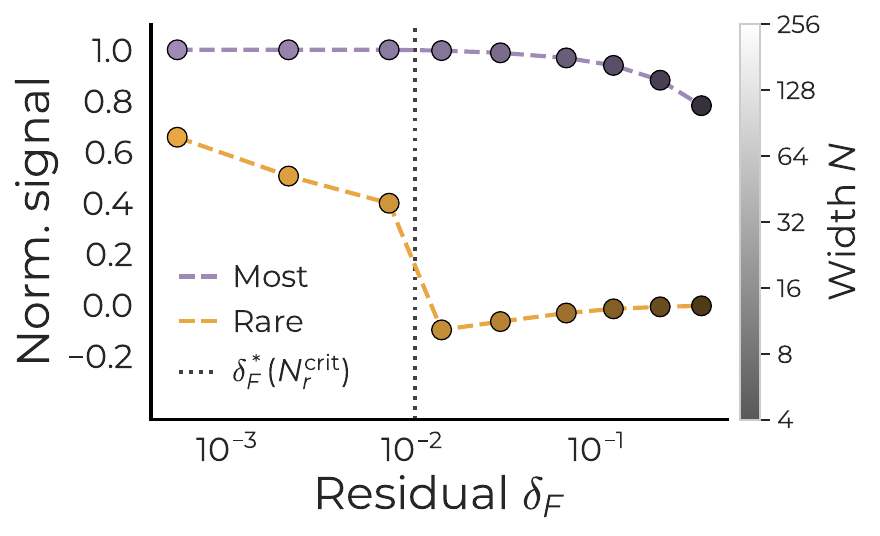}
\end{center}
\vspace{-18pt}
\caption{\textbf{Residual Controls Learning.} We plot signals encoded in model representations for most frequent and rarest tasks as a function of width $N$ and remaining residual $\delta_\Fre$. In line with our predictions, we see larger models perfectly capture tasks of all frequencies, while smaller models do not. Meanwhile, even for the largest models, when the residual signal  to be explained for frequent tasks is high, rarer tasks struggle to be learned.} 
\label{fig:residual_controls_learning}
\end{wrapfigure}

The claim above hence shows width scaling helps in two related but distinct ways. 
First, it reduces the total unresolved common-task signal $\delta_\Fre^*(N)$, which bounds the aggregate common-task gradient. 
Second, as Proposition~\ref{prop:invasion} shows, it lowers the weakest occupied common-task utility $\mu^\Fre_N$, which determines whether a particular rare feature can \textit{displace} a common feature and become locally stable; if the rare feature's utility is lower, even if the model updates to learn it, the common tasks' least utility feature will eventually replace it.
\textit{This will result in a swinging, update-and-forget learning dynamic where the rare task features and lowest utility features of common tasks will compete over model parameters.}
Overall, this suggests the learning bottleneck is defined by the interaction between data and scale: if the task we care to learn does not have sufficient utility for reducing the loss, then the model will prefer to learn and preserve lower-order modes of other tasks; however, \textit{by increasing width, one avails capacity to such low-utility tasks and reduces competition between tasks over model parameters, enabling learning of the rare task without forcing the forgetting of features relevant to common tasks}.

\begin{figure}
\vspace{-12pt}
\begin{center}
\includegraphics[width=\linewidth]{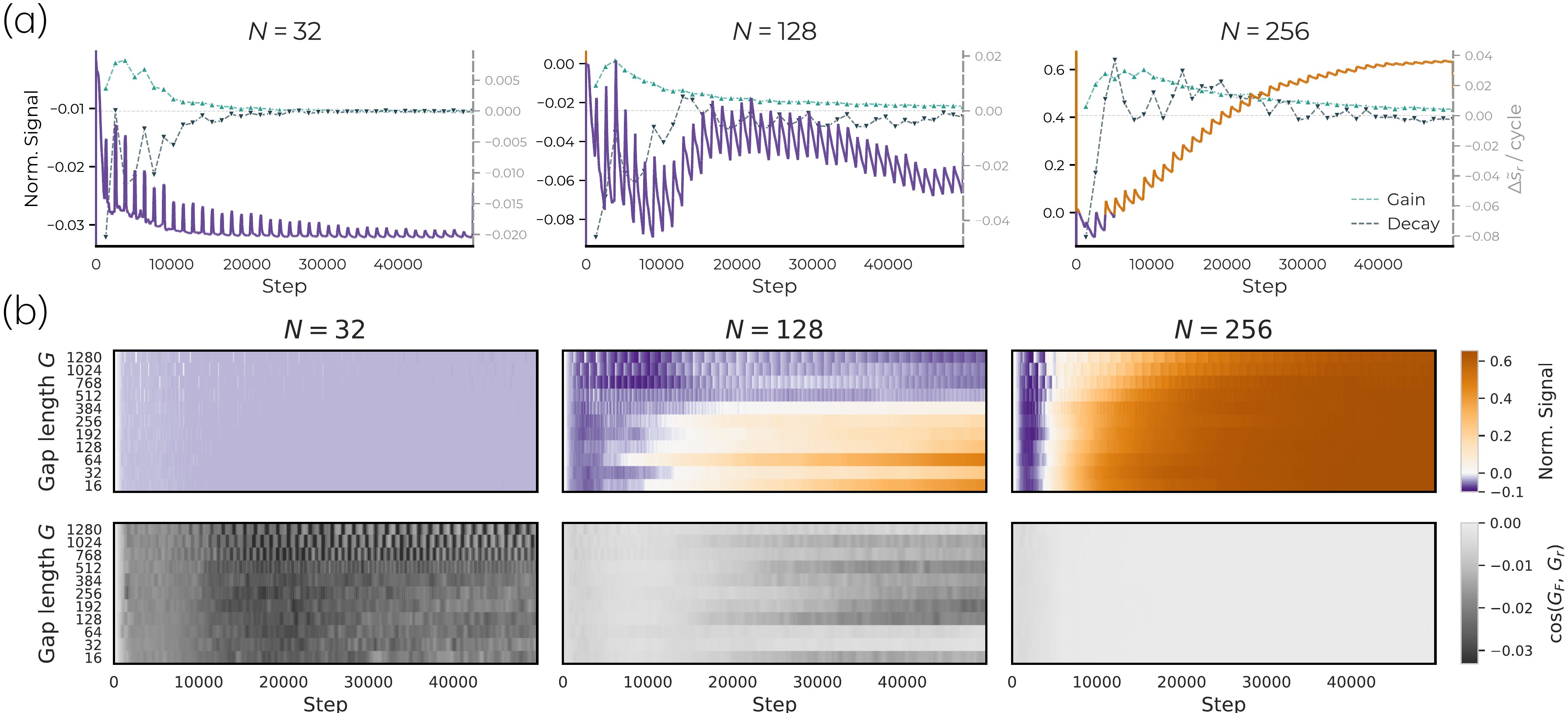}
\end{center}
\vspace{-5pt}
\caption{\textbf{Rare-Task Retention by Larger Models.} 
We isolate retention by training with a matched-frequency injection protocol: the rare task is withheld for $G$ steps and then reintroduced in a batch such that its overall frequency is consistent across settings. (a)~Training dynamics for $G=1280$. We see small models briefly encode the rare task (Norm.\ signal $\tilde{s_r}$: left-y axis) after each injection; specifically, $\Delta \tilde{s_r}$ increases at point of injection, as shown by green dotted line (`gain'). However, as frequent-task updates resume, this signal is lost between injection steps (`decay': gray dotted line). Meanwhile, larger models retain more of the rare-task signal between injections and accumulate it over training. (b)~Across injection gaps $G$ and widths $N$, rare-task signal decays rapidly in narrow models but remains stable in wider models, while frequent-task signal is largely unaffected. 
Furthermore, by computing the cosine similarity of gradients via a batch of rare-task samples $G_r$ and frequent task samples $G_\Fre$, we see scaling provides enough representational capacity such that updates from frequent tasks no longer overwrite rare-task features before the next rare observation arrives.
\vspace{-15pt}
}
\label{fig:competition}
\end{figure}

\paragraph{Validation.} We train models of varying width on the same setup as Fig.~\ref{fig:phases} and plot how much signal from directions describing a task is present in the model's intermediate representation.
Specifically, since $\ell_k(U) = \mathrm{Tr}((I - P_U)C_k)$, the signal captured for task $k$ is $\mathrm{Tr}(P_U C_k)$, where $\mathrm{Tr}(.)$ denotes trace of a matrix.
We thus measure $s_k(U) = \frac{\mathrm{Tr}(P_U C_k)}{\mathrm{Tr}(C_k)}$.
To contextualize this value, we normalize with respect to a random baseline, yielding $\tilde{s}_k(U) = \frac{s_k(U) - N/d}{1 - N/d}$; $N/d$ denotes the expected value of $s_k(U)$ if $U$ were a randomly drawn matrix from the Stiefel manifold.
We call this quantity ``Normalized Signal''.
Results are shown in Fig.~\ref{fig:residual_controls_learning}.
We see when the model width is small, frequent tasks have a high residual signal remaining to be explained; here the set of frequent tasks is defined as top-$K$ tasks whose prior sums to $0.8$, resulting in $K=3$.
Correspondingly, rare tasks' signal in model representation is no better than random.
Meanwhile, as we scale, once the width crosses our predicted threshold $\delta_F^*(N_r^{\mathrm{crit}})$, we find the bulk of the frequent tasks' signal is explained away and rare tasks start to get learned.

To isolate how the gap between observations interacts with width, we also design a matched-frequency injection experiment: the rare task is excluded from training for $G$ steps, then injected in a batch enlarged to $m = G \cdot B \cdot \rho_r$ rare samples so that its long-run frequency exactly matches the setup of Fig.~\ref{fig:phases}. 
This emphasizes the ability of a model to retain memories about observed data, while preserving the total frequency with which it is seen.
Results are shown in Fig.~\ref{fig:competition}. 
We see at the end of training, rare-task signal decays monotonically with $G$ at all widths, but far more steeply for smaller models.
Meanwhile, the learning dynamics in panel (b) show that after each injection, a larger model accumulates rare-task signal and retains enough of it to build on the next injection, while a smaller model decays back to near-zero in between (an intuitive model explaining this dynamic is shown in Fig.~\ref{fig:single_neuron} and analytically described in App.~\ref{app:C-neuron-specialization}).
Overall, our results showing how larger models learn tasks smaller models do not can be summarized as follows.

\begin{tcolorbox}[colback=gray!5, colframe=purp!130, title=\bfseries Hypothesis: Scaling Enables Learning via Reduced Interference and Better Retention]
Given the same training setup, larger models better learn tasks in the tail of the data distribution. When such a task is observed, a larger model can retain part of the update and build on it when the task appears again. In a small model, the same parameters face more competition: updates from frequent tasks undo the rare-task update before the next rare batch arrives. Rare-task learning then becomes an update-and-forget loop.
\end{tcolorbox}

\section{Corroborating Claims with the OLMo Pretraining Pipeline}
\label{sec:real}

We now verify the claims of Sec.~\ref{sec:synthetic} in a realistic LLM pre-training setting using the OLMo pipeline. We train models of size 4M to 4B  on up to 210B tokens ($\sim$50K steps). Following the structure of Sec.~\ref{sec:synthetic}, we offer analyses at three levels: loss, representation, and gradient.

\subsection{Setup}

A key variable in our claims is the frequency of a task.\footnote{Defining the complexity of a natural task is difficult, and hence we solely focus on frequency in this section.}
However, measuring the frequency of a naturally occurring task in pre-training data is challenging, as instances from the same task can occur in many surface forms. 
To tightly control task frequency, we adopt a data injection framework from the memorization literature~\cite{jagielski2023measuring,carlini2019secret,huang2024demystifying,wei2026hubble}. 
We inject different instances sampled from the distribution of a ``special'' task $T$ at a controlled frequency $f$ to measure whether a model has learned the task distribution.
The task $T$ is special in the sense that it is unlikely to be part of normal pre-training data. We then train models of various sizes on data mixtures generated from different values of $f$.

\begin{wrapfigure}[23]{r}{0.45\textwidth}
\vspace{-20pt}
\begin{center}
\includegraphics[width=0.9\linewidth]{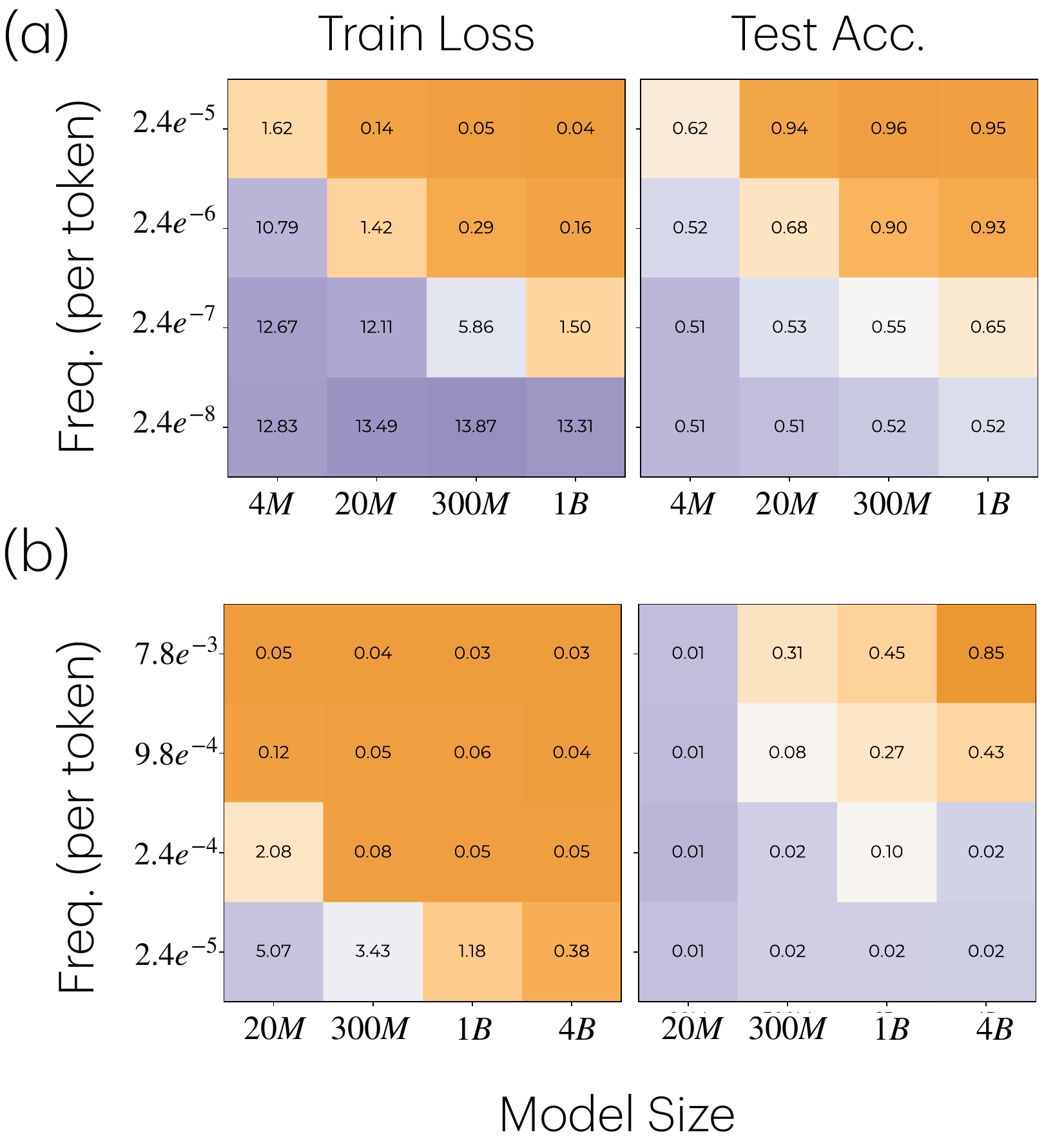}
\end{center}
\vspace{-10pt}
\caption{\label{fig:olmo_freq_vs_model_size}\textbf{Larger Models Learn Rare Tasks; Smaller Models Do Not.} We visualize training loss and test accuracy for the (a) Comparison task ($T_{\text{CMP}}$) and (b) Modular Addition task ($T_{\text{ADD}}$). Orange color indicates lower loss/higher accuracy. Overall, we see that increasing width enables learning of low-frequency tasks, in line with our prior claims.} 
\end{wrapfigure}

\begin{figure*}[!t]
    \centering
    \includegraphics[width=0.9\linewidth]{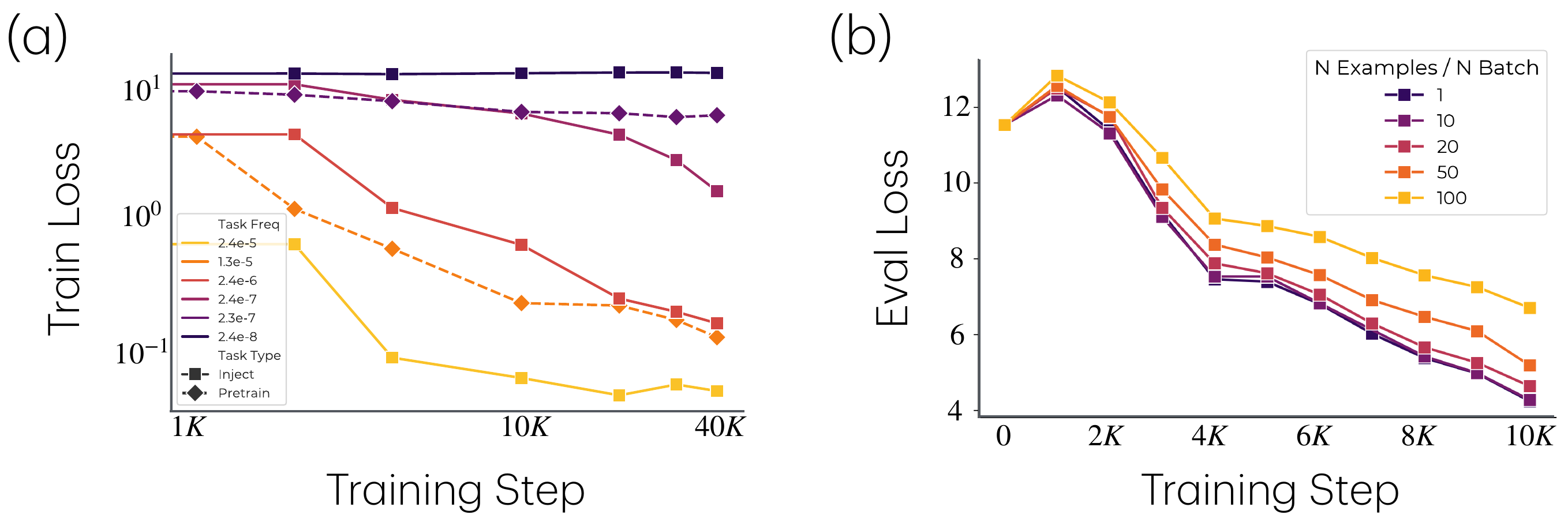}
    \vspace{-5pt}
    \caption{\label{fig:olmo_loss}\textbf{Behavioral Evidence.} (a) Tasks are learned in the order of frequency. Solid lines: We inject the same comparison task (\tcmp) at different frequencies and measure the task training loss. Dashed lines:  Reference arithmetic tasks observed from pre-training data.
    (b) With matched-frequency injection of the comparison task (\tcmp), i.e., injecting $N$ task instances every $N$ batches, a larger injection gap $N$ degrades task loss, while a smaller injection gap leads to almost identical loss.
    \vspace{-10pt}
    }
\end{figure*}

\paragraph{Tasks.} We consider two special tasks $T$: comparison (\tcmp) and modular addition (\tadd). 
Both tasks are encoded as a sequence of three tokens: \texttt{TOK1}, \texttt{TOK2}, \texttt{LABEL}, where \texttt{TOK1}, \texttt{TOK2} $\in \mathcal{S}$, a set of $100$ tokens randomly sampled from the vocab.
There are exactly $10$K instances per task, which are split $50/50$ for training and testing. 
Critically, both tasks require models to learn certain geometrical structures to generalize~\cite{hwang2026intrinsictasksymmetrydrives}. 
This provides a measure for learning a task (as opposed to memorizing training instances) and a set of features to verify the interference hypothesis of Sec.~\ref{sec:synthetic}.

\paragraph{Data.} We use Dolma v1.7 as the pre-training corpus~\cite{dolma}. 
Given a task $T$, we inject instances sampled from its train split at a frequency of $7.8\times10^{-3}$ to $2.4\times10^{-8}$, roughly from 1K instances per batch to 1 instance every 10 batches. 
To ensure the injected task frequency is comparable to the frequency of tasks learned in pre-training, we sample two reference tasks $R_{\text{cmp}}$ and $R_{\text{add}}$ from pre-training that involve similar high-level functions. 
The three-token sequence plus an end of document token replace the first four tokens of a training sequence. 
See App.~\ref{appx:olmo_data} for further details on the experimental setup.

\paragraph{Models.} We train OLMo models~\cite{groeneveld2023olmo} with 4M, 20M, 300M, 1B, and 4B parameters. 
We focus on scaling the models' hidden and MLP dimensions and the number of attention heads; the 4M parameter model has depth 8 and the rest have depth 16. 
See App.~\ref{appx:olmo_config} for further details.

\subsection{Behavioral Evidence}
\label{sec:olmo_loss}

\paragraph{Larger Models Learn Rarer Tasks.} We first replicate the behavioral findings in Sec.~\ref{sec:synthetic_larger_rarer}. 
We measure the effect of task frequency by comparing multiple training runs that only differ by the frequency of the injected task. 
As shown in Fig.~\ref{fig:olmo_freq_vs_model_size}, larger models learn lower-frequency tasks much better than smaller models do.
This matches the pattern in Fig.~\ref{fig:phases}. Moreover, tasks are learned in the order of frequency. 
For each model run, we compare the order in which the injected task \tcmp\ and the reference tasks are learned, as shown in Fig.~\ref{fig:olmo_loss}a. 
Most importantly, larger models do not just lead to better memorization of training instances, i.e., low training loss, but also learn generalizable task structures, i.e., high eval accuracy. 
On \tadd, only larger models trained on higher frequency exhibit the grokking phenomenon~\cite{power2022grokking}.

\paragraph{Rare-Task Retention Has an Effect on Learning.}
We conduct the matched-frequency injection experiment as described in Fig.~\ref{fig:competition}, i.e., injecting $N$ task instances every $N$ batches, for $N=1,10, 20, 50, 100$.  Fig.~\ref{fig:olmo_loss}b shows the effects of retention on learning, as models trained with larger gaps between task instances have higher task loss, even though the global task frequency of all runs is equivalent.

\subsection{Representational Evidence}
\label{sec:olmo_repr}

\paragraph{Task Features.} In our toy setting (Sec.~\ref{sec:synthetic}), we know analytically which features are necessary for learning the $k$th task, i.e., $B_k$, and to what extent the model can represent these features, i.e., $P_U$. 
For our OLMo models, we can empirically identify a set of causal features that a pre-trained LM would use to solve the task and localize them in the model representations. 
Specifically, for \tcmp, the task feature of core relevance is the global order of the tokens, which allows number comparisons; meanwhile for \tadd, task features are the Fourier modes~\cite{nanda_progress_2022,zhou2024pretrained,feucht2026arithmeticwildllamauses}, as shown in Fig.~\ref{fig:task_feature_learning}. We visualize the feature geometry using the largest model trained on the most frequent task data. 
These task features allow us to conduct versions of the gradient and representation-level analyses in Sec.~\ref{sec:synthetic}.

\paragraph{More Task Features are Present in Larger Model Representations.}

\begin{figure*}[!t]
    \centering
    \vspace{-12pt}
    \includegraphics[width=0.95\linewidth]{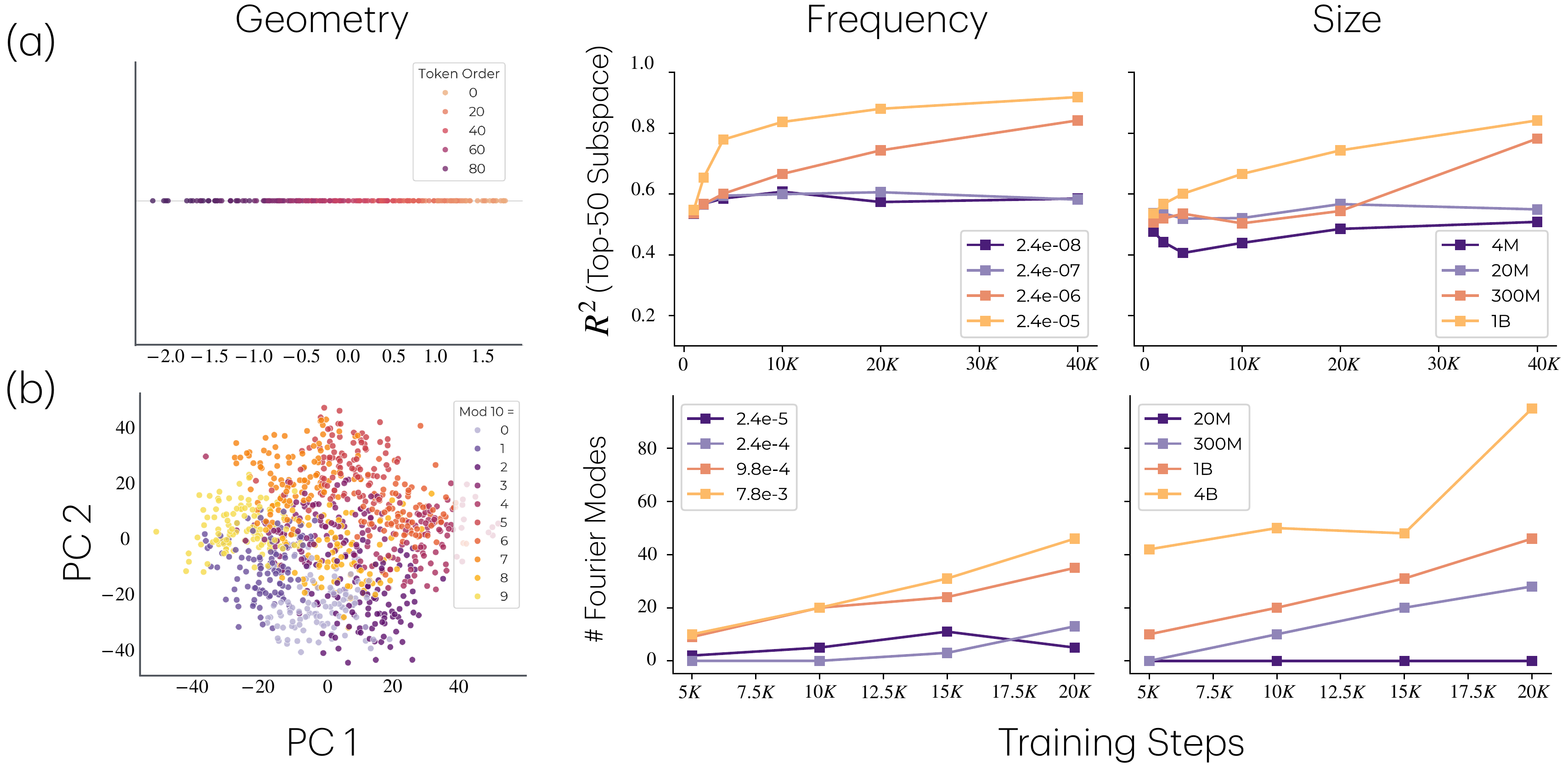}
    \caption{\textbf{Representational Evidence.} Scaling model size (width) and increasing task frequency lead to models learning more task-relevant features.
    Rows correspond to (a) the comparison task \tcmp\ and (b) the modular addition task \tadd.
    The first column shows feature geometry, visualizing the global token order features for \tcmp\ and the Fourier-mode features for \tadd.
    The last two columns quantify how these features scale with task frequency and model size.
    For both tasks, the task features are better represented in larger models trained on higher task frequency.
    \vspace{-5pt}
    }
    \label{fig:task_feature_learning}
\end{figure*}

We first localize the task features in models that have clearly learned the task. 
We then measure to what extent these target task features are present in all models, which parallels the metric $\ell_k(U)$ used in the toy setting. 
For localization, we use distributed alignment search (DAS)~\cite{geiger2024das} which finds subspaces that \emph{causally} encode the features. 
For \tcmp, a global ordering of the tokens can be localized to a 1-D subspace in the residual stream of the first layer. 
For \tadd, Fourier modes can be identified in the residual stream from earlier layers to the last layer.
We then use task-specific metrics to measure to what extent these task features are present in model representations. 
For \tcmp, since the geometry of the task feature is a single direction, we apply linear regression to representations spanning the top $K=50$ principal components. 
For \tadd, we measure the total number of Fourier modes present through all layers. We include the details in App.~\ref{appx:localize_task_features}. 
The last two columns of Fig.~\ref{fig:task_feature_learning} show the extent to which the target task features are present in each model across checkpoints. For the frequency (middle) column, we fix the model size to 1B. For the model size (last) column, we fix the task frequency to be $2.4\times 10^{-6}$ for \tcmp\ and $7.8\times 10^{-3}$ for \tadd. We see that (i) the presence of task features is highly correlated with high accuracy on the test set, and (ii) larger models and models trained on more frequent task data clearly learn these task features faster.

\subsection{Gradient Evidence}\label{sec:olmo:gradients}

\begin{wrapfigure}[14]{r}{0.45\textwidth}
\vspace{-22pt}
\begin{center}
\includegraphics[width=\linewidth]{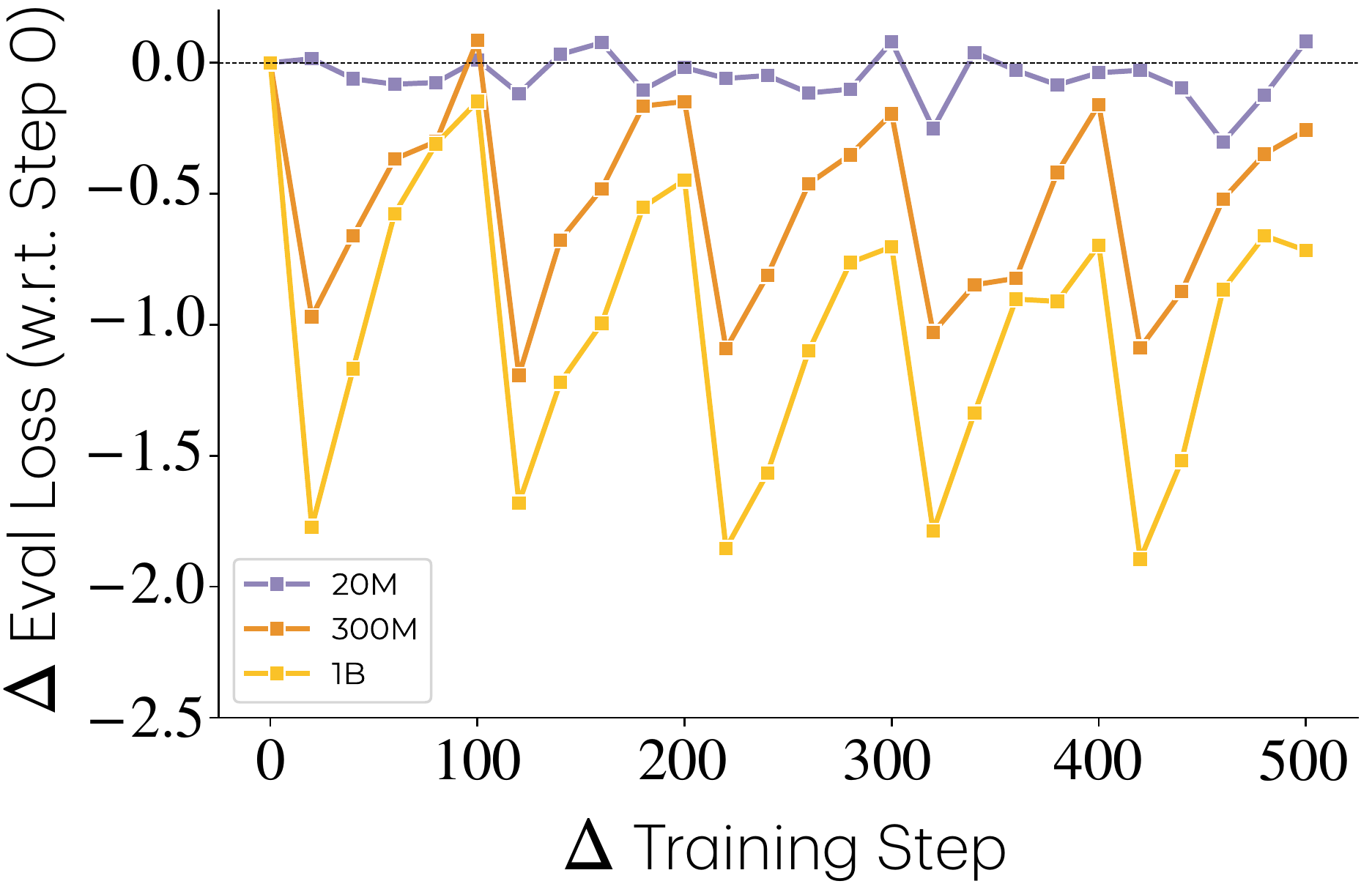}
\end{center}
\vspace{-12pt}
\caption{\label{fig:olmo_retention_vs_model_size}\textbf{Rare-Task Retention.} Larger models can retain the injected task information better, i.e., larger task eval loss drop, when injecting task instances every 100 batches.
} 
\end{wrapfigure}

We now connect the behavioral evidence (Sec.~\ref{sec:olmo_loss}) and the internal representation account (Sec.~\ref{sec:olmo_repr}) by analyzing how task gradients interfere with non-task gradients on a set of neurons that implement the task circuit. We focus on \tcmp\ training runs in Fig.~\ref{fig:olmo_retention_vs_model_size}, where 100 task instances are injected every 100 steps.

\paragraph{Task Neurons.} We first identify which MLP layers implement the task features defined in Sec.~\ref{sec:olmo_repr}. 
For all the models that we compared, the first layer MLP has the largest causal effects on task predictions. 
We further identify the top $K$ neurons in the first layer MLP that have the largest gradient magnitude and use the gradients of these neurons for analysis.
Details can be found in App.~\ref{appx:localize_task_neurons}.

\begin{figure*}[!t]
    \centering
    \includegraphics[width=0.95\linewidth]{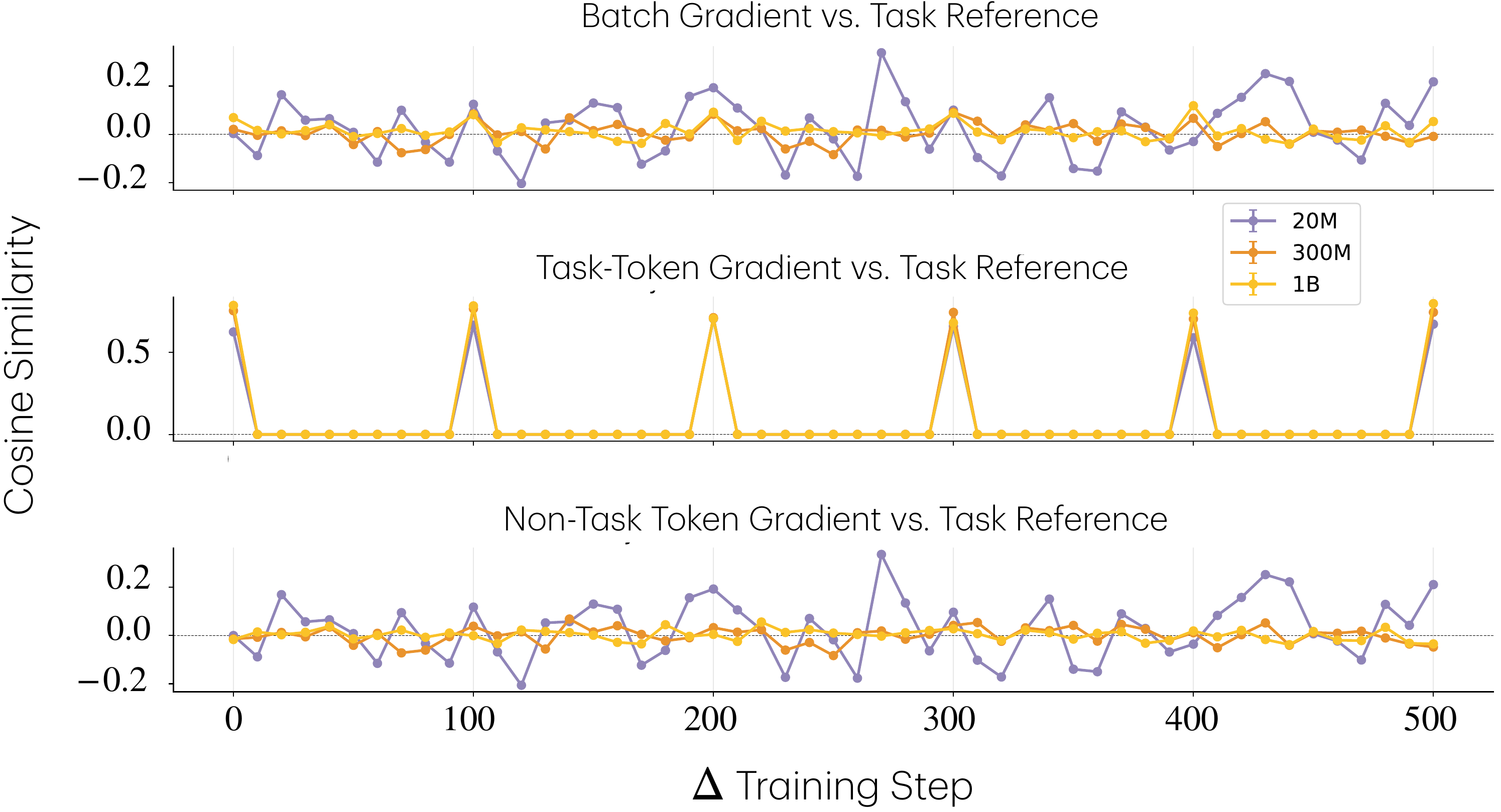}
    \vspace{-5pt}
    \caption{\textbf{Gradient Interference.}
    We inject 100 instances of the \tcmp\ task every 100 batches and analyze how batch gradients align with a task reference direction $g_r$.
    We further decompose the batch gradient into contributions from task tokens and non-task tokens.
    \textbf{Top}: Cosine similarity between full-batch gradient direction and the task direction $g_r$. \textbf{Middle}: Cosine similarity between batch task gradient direction and $g_r$. Higher values imply more task signals. \textbf{Bottom}: Cosine similarity between batch non-task gradient direction and $g_r$. Lower values imply less gradient interference. Overall, the batch gradient directions of larger models carry more task signals with little to no interference.
    \vspace{-20pt}
    }
    \label{fig:olmo_gradient_interference}
\end{figure*}

\paragraph{Task Reference Direction $g_r$.} We estimate the task reference direction using the aggregated gradient of the task loss computed over all 10K task instances, an analogy to $G_r$ in the toy setting. 
This direction may shift across training steps; however, at a given step, it is the optimal task direction.

\paragraph{Larger Models Show Less Gradient Interference Between General Language Modeling and Our Injected Task.} We quantify the relation between the task reference $g_r$ and the batch gradient $g$, which can be further decomposed into gradient from the task tokens $g_t$ (if exists in batch) and non-task tokens $g_{\mathit{nt}}$, i.e., $g = g_t + g_{\mathit{nt}}$. We first measure the cosine similarity between task reference and batch gradient direction, replicating the results in Fig.~\ref{fig:competition}. We additionally analyze whether task or non-task tokens contribute to this similarity; while task token gradient aligning with task reference $g_r$ is expected, non-task token gradient with non-zero cosine similarity suggests that the language modeling direction is interfering with the task gradient direction.

Results are shown in Fig.~\ref{fig:olmo_gradient_interference}. In the top panel, larger models have higher similarity between $g$ and $g_r$ at the injection steps, $0.08\pm0.02$ for the 1B model and $0.04\pm 0.04$ for the 300M model, the similarity typically regresses towards zero between injections. For the 20M model, the similarity scores oscillate wildly across batches, even at the injection step. In fact, the high similarity between non-task gradient $g_{nt}$ and $g_{r}$ reveals that for the 20M model the batch gradient similarity mostly comes from random collisions with task direction, with a similarity score of $0.10\pm 0.09$, while for larger models, $g_{nt}$ is almost orthogonal to $g_{r}$, with $7.58\times10^{-5}\pm 0.02$ for the 1B model, suggesting little to no gradient interference on this set of neurons.

\section{Discussion}

We develop a data-centric account of why larger models can learn tasks that smaller models fail to learn. Specifically, we show that larger models can learn rare tasks from the data mixture, and this phenomenon is explained by learning dynamics, i.e., competition of resources and retention of memories, as well as the task frequency and complexity.
Our perspective highlights that understanding scaling requires thinking beyond model expressivity. We need to understand how learning dynamics are at play with task frequency and complexity. It also points toward more intentional design of data mixtures to better elicit target capabilities. For example, simply scaling up the frequency of a target task might provide a more efficient way to learn the task than scaling up the model size.
Lastly, our findings on how better retention of memories enables learning rare tasks offers a new perspective that views memorization as a mechanism that can support learning abstraction: by retaining task instances longer, models can accumulate signals across batches to learn more generalizable structures of the task. 
This suggests memorization can in fact be beneficial, in line with arguments by Feldman~\cite{feldman2020does}.

\section*{Limitations}
As noted above, our account for why larger models learn more emphasizes the interplay of learning dynamics, task frequency, and task complexity. 
However, as discussed in Sec.~\ref{sec:intro}, there are other plausible accounts for explaining this phenomenon, e.g., ones that focus on model expressivity and sample efficiency. 
Our explanation hence should not be interpreted as a complete account of scaling. 
Instead, these explanations are complementary: expressivity constrains what can be represented, sample efficiency shapes how effectively data is used, and our account highlights how learning dynamics interact with the frequency and complexity of tasks. 
A full understanding likely requires accommodating all these explanations, rather than viewing them as competing hypotheses.
We also note we validated our key theoretical results using the OLMo pre-training pipeline, finding the empirical results on the injected tasks strongly match what the theoretical results predict.
However, we acknowledge that empirical validation in a realistic pre-training setting could still leave some analytic gaps. 
For example, we did not empirically verify behavior of larger-scale language models or over-trained language models. 
We also selected injected tasks that matched the frequency of tasks learned in OLMo pre-training, which does not rule out other scaling behaviors with extreme task frequency. 
Our empirical results should therefore be viewed as supporting evidence. 
We encourage future work to explore different training regimes, more tasks, and different frequency ranges.

\section*{Acknowledgments}
The authors thank Blake Bordelon, Jacob Zavatone-Veth, and Core Francisco Park for several useful references that helped concretize the claims posited in this work, and Yasaman Bahri, Surya Ganguli, Ari Holtzman, Stephanie Chan, Freya Behrens, Tom McGrath, Owen Lewis, Atticus Geiger, Jack Merullo, and Thomas Fel for fruitful conversations during the course of this project.
The authors also thank Thomas Icard for several comments on an earlier version of this draft.
This research is supported in part by a grant from Open Philanthropy (Coefficient Giving) to CP.

\bibliography{refs}

\begin{thebibliography}{140}
\providecommand{\natexlab}[1]{#1}
\providecommand{\url}[1]{\texttt{#1}}
\expandafter\ifx\csname urlstyle\endcsname\relax
  \providecommand{\doi}[1]{doi: #1}\else
  \providecommand{\doi}{doi: \begingroup \urlstyle{rm}\Url}\fi

\bibitem[Singh et~al.(2025)Singh, Fry, Perelman, Tart, Ganesh, El-Kishky, McLaughlin, Low, Ostrow, Ananthram, et~al.]{singh2025openai}
Aaditya Singh, Adam Fry, Adam Perelman, Adam Tart, Adi Ganesh, Ahmed El-Kishky, Aidan McLaughlin, Aiden Low, AJ~Ostrow, Akhila Ananthram, et~al.
\newblock Openai gpt-5 system card.
\newblock \emph{arXiv preprint arXiv:2601.03267}, 2025.

\bibitem[Anthropic(2026{\natexlab{a}})]{mythos}
Anthropic.
\newblock {System Card: Claude Mythos Preview}, 2026{\natexlab{a}}.
\newblock \url{https://www-cdn.anthropic.com/08ab9158070959f88f296514c21b7facce6f52bc.pdf}.

\bibitem[DeepMind(2026)]{gemini3}
Google DeepMind.
\newblock {Gemini 3 Pro - Model Card}, 2026.
\newblock \url{https://storage.googleapis.com/deepmind-media/Model-Cards/Gemini-3-Pro-Model-Card.pdf}.

\bibitem[Liu et~al.(2024{\natexlab{a}})Liu, Feng, Xue, Wang, Wu, Lu, Zhao, Deng, Zhang, Ruan, et~al.]{deepseekv3}
Aixin Liu, Bei Feng, Bing Xue, Bingxuan Wang, Bochao Wu, Chengda Lu, Chenggang Zhao, Chengqi Deng, Chenyu Zhang, Chong Ruan, et~al.
\newblock Deepseek-v3 technical report.
\newblock \emph{arXiv preprint arXiv:2412.19437}, 2024{\natexlab{a}}.

\bibitem[DeepSeek-AI(2026)]{deepseekv4}
DeepSeek-AI.
\newblock {DeepSeek-V4-Pro}, 2026.
\newblock \url{https://huggingface.co/deepseek-ai/DeepSeek-V4-Pro/blob/main/DeepSeek_V4.pdf}.

\bibitem[Team et~al.(2026)Team, Bai, Bai, Bao, Cai, Cao, Charles, Che, Chen, Chen, et~al.]{team2026kimi}
Kimi Team, Tongtong Bai, Yifan Bai, Yiping Bao, SH~Cai, Yuan Cao, Y~Charles, HS~Che, Cheng Chen, Guanduo Chen, et~al.
\newblock Kimi k2. 5: Visual agentic intelligence.
\newblock \emph{arXiv preprint arXiv:2602.02276}, 2026.

\bibitem[Kwa et~al.(2025)Kwa, West, Becker, Deng, Garcia, Hasin, Jawhar, Kinniment, Rush, Von~Arx, et~al.]{kwameasuring}
Thomas Kwa, Ben West, Joel Becker, Amy Deng, Katharyn Garcia, Max Hasin, Sami Jawhar, Megan Kinniment, Nate Rush, Sydney Von~Arx, et~al.
\newblock Measuring {AI} ability to complete long software tasks.
\newblock In \emph{The Thirty-ninth Annual Conference on Neural Information Processing Systems}, 2025.

\bibitem[Glazer et~al.(2024)Glazer, Erdil, Besiroglu, Chicharro, Chen, Gunning, Olsson, Denain, Ho, Santos, et~al.]{glazer2024frontiermath}
Elliot Glazer, Ege Erdil, Tamay Besiroglu, Diego Chicharro, Evan Chen, Alex Gunning, Caroline~Falkman Olsson, Jean-Stanislas Denain, Anson Ho, Emily de~Oliveira Santos, et~al.
\newblock Frontiermath: A benchmark for evaluating advanced mathematical reasoning in ai.
\newblock \emph{arXiv preprint arXiv:2411.04872}, 2024.

\bibitem[Foundation(2026)]{arc3}
ARC~Prize Foundation.
\newblock {ARC-AGI-3}, 2026.
\newblock \url{https://arcprize.org/arc-agi/3}.

\bibitem[Merrill et~al.(2026)Merrill, Shaw, Carlini, Li, Raj, Bercovich, Shi, Shin, Walshe, Buchanan, et~al.]{merrill2026terminal}
Mike~A Merrill, Alexander~G Shaw, Nicholas Carlini, Boxuan Li, Harsh Raj, Ivan Bercovich, Lin Shi, Jeong~Yeon Shin, Thomas Walshe, E~Kelly Buchanan, et~al.
\newblock Terminal-bench: Benchmarking agents on hard, realistic tasks in command line interfaces.
\newblock In \emph{The Fourteenth International Conference on Learning Representations}, 2026.
\newblock URL \url{https://openreview.net/forum?id=a7Qa4CcHak}.

\bibitem[Anthropic(2026{\natexlab{b}})]{rsp}
Anthropic.
\newblock {Responsible Scaling Policy}, 2026{\natexlab{b}}.
\newblock \url{https://www.anthropic.com/responsible-scaling-policy}.

\bibitem[OpenAI(2023)]{oairisk}
OpenAI.
\newblock {Our Approach to Frontier Risk}, 2023.
\newblock \url{https://openai.com/global-affairs/our-approach-to-frontier-risk/}.

\bibitem[Sim{\'e}oni et~al.(2025)Sim{\'e}oni, Vo, Seitzer, Baldassarre, Oquab, Jose, Khalidov, Szafraniec, Yi, Ramamonjisoa, et~al.]{simeoni2025dinov3}
Oriane Sim{\'e}oni, Huy~V Vo, Maximilian Seitzer, Federico Baldassarre, Maxime Oquab, Cijo Jose, Vasil Khalidov, Marc Szafraniec, Seungeun Yi, Micha{\"e}l Ramamonjisoa, et~al.
\newblock Dinov3.
\newblock \emph{arXiv preprint arXiv:2508.10104}, 2025.

\bibitem[Hu et~al.(2024)Hu, Liu, Han, Zhang, He, Zhao, Lin, Ding, Ou, Zeng, Liu, and Sun]{hu2023unlock}
Shengding Hu, Xin Liu, Xu~Han, Xinrong Zhang, Chaoqun He, Weilin Zhao, Yankai Lin, Ning Ding, Zebin Ou, Guoyang Zeng, Zhiyuan Liu, and Maosong Sun.
\newblock Predicting emergent abilities with infinite resolution evaluation.
\newblock In \emph{The Twelfth International Conference on Learning Representations}, 2024.
\newblock URL \url{https://openreview.net/forum?id=lDbjooxLkD}.

\bibitem[Wei et~al.(2022)Wei, Tay, Bommasani, Raffel, Zoph, Borgeaud, Yogatama, Bosma, Zhou, Metzler, Chi, Hashimoto, Vinyals, Liang, Dean, and Fedus]{wei2022emergent}
Jason Wei, Yi~Tay, Rishi Bommasani, Colin Raffel, Barret Zoph, Sebastian Borgeaud, Dani Yogatama, Maarten Bosma, Denny Zhou, Donald Metzler, Ed~H. Chi, Tatsunori Hashimoto, Oriol Vinyals, Percy Liang, Jeff Dean, and William Fedus.
\newblock Emergent abilities of large language models.
\newblock \emph{Transactions on Machine Learning Research}, 2022.
\newblock ISSN 2835-8856.
\newblock URL \url{https://openreview.net/forum?id=yzkSU5zdwD}.

\bibitem[Wei(2022)]{wei2022emergent_list_of_137}
Jason Wei.
\newblock {137 emergent abilities of large language models}, 2022.
\newblock \url{https://www.jasonwei.net/blog/emergence}.

\bibitem[Arora and Goyal(2023)]{arora2023theory}
Sanjeev Arora and Anirudh Goyal.
\newblock {A theory for emergence of complex skills in language models}.
\newblock \emph{arXiv preprint arXiv:2307.15936}, 2023.

\bibitem[Du et~al.(2024)Du, Zeng, Dong, and Tang]{du2024understanding}
Zhengxiao Du, Aohan Zeng, Yuxiao Dong, and Jie Tang.
\newblock Understanding emergent abilities of language models from the loss perspective.
\newblock In \emph{The Thirty-eighth Annual Conference on Neural Information Processing Systems}, 2024.
\newblock URL \url{https://openreview.net/forum?id=35DAviqMFo}.

\bibitem[Wei et~al.(2023)Wei, Wei, Tay, Tran, Webson, Lu, Chen, Liu, Huang, Zhou, et~al.]{wei2023larger}
Jerry Wei, Jason Wei, Yi~Tay, Dustin Tran, Albert Webson, Yifeng Lu, Xinyun Chen, Hanxiao Liu, Da~Huang, Denny Zhou, et~al.
\newblock Larger language models do in-context learning differently.
\newblock \emph{arXiv preprint arXiv:2303.03846}, 2023.

\bibitem[Kaplan et~al.(2020)Kaplan, McCandlish, Henighan, Brown, Chess, Child, Gray, Radford, Wu, and Amodei]{kaplan2020scaling}
Jared Kaplan, Sam McCandlish, Tom Henighan, Tom~B Brown, Benjamin Chess, Rewon Child, Scott Gray, Alec Radford, Jeffrey Wu, and Dario Amodei.
\newblock Scaling laws for neural language models.
\newblock \emph{arXiv preprint arXiv:2001.08361}, 2020.

\bibitem[Hestness et~al.(2017)Hestness, Narang, Ardalani, Diamos, Jun, Kianinejad, Patwary, Yang, and Zhou]{hestness2017deep}
Joel Hestness, Sharan Narang, Newsha Ardalani, Gregory Diamos, Heewoo Jun, Hassan Kianinejad, Md~Mostofa~Ali Patwary, Yang Yang, and Yanqi Zhou.
\newblock Deep learning scaling is predictable, empirically.
\newblock \emph{arXiv preprint arXiv:1712.00409}, 2017.

\bibitem[Rosenfeld et~al.(2020)Rosenfeld, Rosenfeld, Belinkov, and Shavit]{rosenfeld2019constructive}
Jonathan~S. Rosenfeld, Amir Rosenfeld, Yonatan Belinkov, and Nir Shavit.
\newblock A constructive prediction of the generalization error across scales.
\newblock In \emph{International Conference on Learning Representations}, 2020.
\newblock URL \url{https://openreview.net/forum?id=ryenvpEKDr}.

\bibitem[Henighan et~al.(2020)Henighan, Kaplan, Katz, Chen, Hesse, Jackson, Jun, Brown, Dhariwal, Gray, et~al.]{henighan2020scaling}
Tom Henighan, Jared Kaplan, Mor Katz, Mark Chen, Christopher Hesse, Jacob Jackson, Heewoo Jun, Tom~B Brown, Prafulla Dhariwal, Scott Gray, et~al.
\newblock Scaling laws for autoregressive generative modeling.
\newblock \emph{arXiv preprint arXiv:2010.14701}, 2020.

\bibitem[Hernandez et~al.(2021)Hernandez, Kaplan, Henighan, and McCandlish]{hernandez2021scaling}
Danny Hernandez, Jared Kaplan, Tom Henighan, and Sam McCandlish.
\newblock Scaling laws for transfer.
\newblock \emph{arXiv preprint arXiv:2102.01293}, 2021.

\bibitem[Rae et~al.(2021)Rae, Borgeaud, Cai, Millican, Hoffmann, Song, Aslanides, Henderson, Ring, Young, et~al.]{rae2021scaling}
Jack~W Rae, Sebastian Borgeaud, Trevor Cai, Katie Millican, Jordan Hoffmann, Francis Song, John Aslanides, Sarah Henderson, Roman Ring, Susannah Young, et~al.
\newblock Scaling language models: Methods, analysis \& insights from training gopher.
\newblock \emph{arXiv preprint arXiv:2112.11446}, 2021.

\bibitem[Alabdulmohsin et~al.(2022)Alabdulmohsin, Neyshabur, and Zhai]{alabdulmohsin2022revisiting}
Ibrahim~M Alabdulmohsin, Behnam Neyshabur, and Xiaohua Zhai.
\newblock Revisiting neural scaling laws in language and vision.
\newblock \emph{Advances in Neural Information Processing Systems}, 35:\penalty0 22300--22312, 2022.

\bibitem[Grattafiori et~al.(2024)Grattafiori, Dubey, Jauhri, Pandey, Kadian, Al-Dahle, Letman, Mathur, Schelten, Vaughan, et~al.]{grattafiori2024llama}
Aaron Grattafiori, Abhimanyu Dubey, Abhinav Jauhri, Abhinav Pandey, Abhishek Kadian, Ahmad Al-Dahle, Aiesha Letman, Akhil Mathur, Alan Schelten, Alex Vaughan, et~al.
\newblock The llama 3 herd of models.
\newblock \emph{arXiv preprint arXiv:2407.21783}, 2024.

\bibitem[Hoffmann et~al.(2022)Hoffmann, Borgeaud, Mensch, Buchatskaya, Cai, Rutherford, Casas, Hendricks, Welbl, Clark, et~al.]{hoffmann2022training}
Jordan Hoffmann, Sebastian Borgeaud, Arthur Mensch, Elena Buchatskaya, Trevor Cai, Eliza Rutherford, Diego de~Las Casas, Lisa~Anne Hendricks, Johannes Welbl, Aidan Clark, et~al.
\newblock Training compute-optimal large language models.
\newblock In Alice~H. Oh, Alekh Agarwal, Danielle Belgrave, and Kyunghyun Cho, editors, \emph{Advances in Neural Information Processing Systems}, 2022.
\newblock URL \url{https://openreview.net/forum?id=iBBcRUlOAPR}.

\bibitem[Pearce and Song(2024)]{pearce2024reconciling}
Tim Pearce and Jinyeop Song.
\newblock Reconciling kaplan and chinchilla scaling laws.
\newblock \emph{Transactions on Machine Learning Research}, 2024.
\newblock ISSN 2835-8856.
\newblock URL \url{https://openreview.net/forum?id=NLoaLyuUUF}.

\bibitem[Bordelon et~al.(2024)Bordelon, Atanasov, and Pehlevan]{bordelon2024dynamical}
Blake Bordelon, Alexander Atanasov, and Cengiz Pehlevan.
\newblock A dynamical model of neural scaling laws.
\newblock In \emph{Forty-first International Conference on Machine Learning}, 2024.
\newblock URL \url{https://openreview.net/forum?id=nbOY1OmtRc}.

\bibitem[Bordelon et~al.(2025)Bordelon, Atanasov, and Pehlevan]{bordelon2025feature}
Blake Bordelon, Alexander Atanasov, and Cengiz Pehlevan.
\newblock How feature learning can improve neural scaling laws.
\newblock In \emph{The Thirteenth International Conference on Learning Representations}, 2025.
\newblock URL \url{https://openreview.net/forum?id=dEypApI1MZ}.

\bibitem[Bahri et~al.(2024)Bahri, Dyer, Kaplan, Lee, and Sharma]{bahri2024explaining}
Yasaman Bahri, Ethan Dyer, Jared Kaplan, Jaehoon Lee, and Utkarsh Sharma.
\newblock Explaining neural scaling laws.
\newblock \emph{Proceedings of the National Academy of Sciences}, 121\penalty0 (27):\penalty0 e2311878121, 2024.

\bibitem[Lin et~al.(2024)Lin, Wu, Kakade, Bartlett, and Lee]{lin2024scaling}
Licong Lin, Jingfeng Wu, Sham~M. Kakade, Peter~L. Bartlett, and Jason Lee.
\newblock Scaling laws in linear regression: Compute, parameters, and data.
\newblock In A.~Globerson, L.~Mackey, D.~Belgrave, A.~Fan, U.~Paquet, J.~Tomczak, and C.~Zhang, editors, \emph{Advances in Neural Information Processing Systems}, volume~37, pages 60556--60606. Curran Associates, Inc., 2024.
\newblock \doi{10.52202/079017-1937}.
\newblock URL \url{https://proceedings.neurips.cc/paper_files/paper/2024/file/6fcb1afcc1e9c2c82c8ddddf03bcf0f6-Paper-Conference.pdf}.

\bibitem[Michaud et~al.(2023)Michaud, Liu, Girit, and Tegmark]{michaud2024quantization}
Eric Michaud, Ziming Liu, Uzay Girit, and Max Tegmark.
\newblock The quantization model of neural scaling.
\newblock In A.~Oh, T.~Naumann, A.~Globerson, K.~Saenko, M.~Hardt, and S.~Levine, editors, \emph{Advances in Neural Information Processing Systems}, volume~36, pages 28699--28722. Curran Associates, Inc., 2023.
\newblock URL \url{https://proceedings.neurips.cc/paper_files/paper/2023/file/5b6346a05a537d4cdb2f50323452a9fe-Paper-Conference.pdf}.

\bibitem[Maloney et~al.(2022)Maloney, Roberts, and Sully]{maloney2022solvable}
Alexander Maloney, Daniel~A Roberts, and James Sully.
\newblock A solvable model of neural scaling laws.
\newblock \emph{arXiv preprint arXiv:2210.16859}, 2022.

\bibitem[Lubana et~al.(2025)Lubana, Kawaguchi, Dick, and Tanaka]{lubana2024percolation}
Ekdeep~Singh Lubana, Kyogo Kawaguchi, Robert~P. Dick, and Hidenori Tanaka.
\newblock A percolation model of emergence: Analyzing transformers trained on a formal language.
\newblock In \emph{The Thirteenth International Conference on Learning Representations}, 2025.
\newblock URL \url{https://openreview.net/forum?id=0pLCDJVVRD}.

\bibitem[Cagnetta et~al.(2025{\natexlab{a}})Cagnetta, Kang, and Wyart]{cagnetta2025learning}
Francesco Cagnetta, Hyunmo Kang, and Matthieu Wyart.
\newblock Learning curves theory for hierarchically compositional data with power-law distributed features.
\newblock In \emph{Forty-second International Conference on Machine Learning}, 2025{\natexlab{a}}.
\newblock URL \url{https://openreview.net/forum?id=Lw0kC75dY0}.

\bibitem[Cagnetta et~al.(2025{\natexlab{b}})Cagnetta, Favero, Sclocchi, and Wyart]{cagnetta2025scaling}
Francesco Cagnetta, Alessandro Favero, Antonio Sclocchi, and Matthieu Wyart.
\newblock Scaling laws and representation learning in simple hierarchical languages: Transformers vs. convolutional architectures.
\newblock \emph{arXiv preprint arXiv:2505.07070}, 2025{\natexlab{b}}.

\bibitem[Cagnetta et~al.(2026)Cagnetta, Ravent{\'o}s, Ganguli, and Wyart]{cagnetta2026deriving}
Francesco Cagnetta, Allan Ravent{\'o}s, Surya Ganguli, and Matthieu Wyart.
\newblock Deriving neural scaling laws from the statistics of natural language.
\newblock \emph{arXiv preprint arXiv:2602.07488}, 2026.

\bibitem[Edelman et~al.(2023)Edelman, Goel, Kakade, Malach, and Zhang]{edelman2023pareto}
Benjamin Edelman, Surbhi Goel, Sham Kakade, Eran Malach, and Cyril Zhang.
\newblock Pareto frontiers in deep feature learning: Data, compute, width, and luck.
\newblock \emph{Advances in Neural Information Processing Systems}, 36:\penalty0 48021--48034, 2023.

\bibitem[Lambert et~al.(2025)Lambert, Morrison, Pyatkin, Huang, Ivison, Brahman, Miranda, Liu, Dziri, Lyu, et~al.]{lambert2024tulu}
Nathan Lambert, Jacob Morrison, Valentina Pyatkin, Shengyi Huang, Hamish Ivison, Faeze Brahman, Lester James~Validad Miranda, Alisa Liu, Nouha Dziri, Xinxi Lyu, et~al.
\newblock Tulu 3: Pushing frontiers in open language model post-training.
\newblock In \emph{Second Conference on Language Modeling}, 2025.
\newblock URL \url{https://openreview.net/forum?id=i1uGbfHHpH}.

\bibitem[H{\"u}botter et~al.(2026)H{\"u}botter, L{\"u}beck, Behric, Baumann, Bagatella, Marta, Hakimi, Shenfeld, Buening, Guestrin, et~al.]{hubotter2026reinforcement}
Jonas H{\"u}botter, Frederike L{\"u}beck, Lejs Behric, Anton Baumann, Marco Bagatella, Daniel Marta, Ido Hakimi, Idan Shenfeld, Thomas~Kleine Buening, Carlos Guestrin, et~al.
\newblock Reinforcement learning via self-distillation.
\newblock \emph{arXiv preprint arXiv:2601.20802}, 2026.

\bibitem[Bloomberg(2026)]{deepdistill}
Bloomberg.
\newblock {OpenAI Claims DeepSeek Distilled US Models to Gain an Edge}, 2026.
\newblock \url{https://www.bloomberg.com/news/articles/2026-02-12/openai-accuses-deepseek-of-distilling-us-models-to-gain-an-edge?}

\bibitem[Shao et~al.(2024)Shao, Wang, Zhu, Xu, Song, Bi, Zhang, Zhang, Li, Wu, et~al.]{shao2024deepseekmath}
Zhihong Shao, Peiyi Wang, Qihao Zhu, Runxin Xu, Junxiao Song, Xiao Bi, Haowei Zhang, Mingchuan Zhang, YK~Li, Yang Wu, et~al.
\newblock Deepseekmath: Pushing the limits of mathematical reasoning in open language models.
\newblock \emph{arXiv preprint arXiv:2402.03300}, 2024.

\bibitem[Xin et~al.(2025)Xin, Ren, Song, Shao, Zhao, Wang, Liu, Zhang, Lu, Du, , et~al.]{xin2024deepseek}
Huajian Xin, Z.Z. Ren, Junxiao Song, Zhihong Shao, Wanjia Zhao, Haocheng Wang, Bo~Liu, Liyue Zhang, Xuan Lu, Qiushi Du, , et~al.
\newblock Deepseek-prover-v1.5: Harnessing proof assistant feedback for reinforcement learning and monte-carlo tree search.
\newblock In \emph{The Thirteenth International Conference on Learning Representations}, 2025.
\newblock URL \url{https://openreview.net/forum?id=I4YAIwrsXa}.

\bibitem[Agarwal et~al.(2024)Agarwal, Vieillard, Zhou, Stanczyk, Garea, Geist, and Bachem]{agarwal2024policy}
Rishabh Agarwal, Nino Vieillard, Yongchao Zhou, Piotr Stanczyk, Sabela~Ramos Garea, Matthieu Geist, and Olivier Bachem.
\newblock On-policy distillation of language models: Learning from self-generated mistakes.
\newblock In \emph{The Twelfth International Conference on Learning Representations}, 2024.
\newblock URL \url{https://openreview.net/forum?id=3zKtaqxLhW}.

\bibitem[Zhao et~al.(2026)Zhao, Xie, Liu, Huang, Pang, Chen, and Grover]{zhao2026self}
Siyan Zhao, Zhihui Xie, Mengchen Liu, Jing Huang, Guan Pang, Feiyu Chen, and Aditya Grover.
\newblock Self-distilled reasoner: On-policy self-distillation for large language models.
\newblock \emph{arXiv preprint arXiv:2601.18734}, 2026.

\bibitem[Tang et~al.(2026)Tang, Wang, Madaan, and Munos]{tang2025beyond}
Yunhao Tang, Sid Wang, Lovish Madaan, and Remi Munos.
\newblock Beyond verifiable rewards: Scaling reinforcement learning in language models to unverifiable data.
\newblock In \emph{The Thirty-ninth Annual Conference on Neural Information Processing Systems}, 2026.
\newblock URL \url{https://openreview.net/forum?id=pc6M9h3T9m}.

\bibitem[Team et~al.(2025)Team, Du, Gao, Xing, Jiang, Chen, Li, Xiao, Du, Liao, et~al.]{team2025kimi}
Kimi Team, Angang Du, Bofei Gao, Bowei Xing, Changjiu Jiang, Cheng Chen, Cheng Li, Chenjun Xiao, Chenzhuang Du, Chonghua Liao, et~al.
\newblock Kimi k1. 5: Scaling reinforcement learning with llms.
\newblock \emph{arXiv preprint arXiv:2501.12599}, 2025.

\bibitem[Blakeney et~al.(2022)Blakeney, Forde, Frankle, Zong, and Leavitt]{blakeney2022reduce}
Cody Blakeney, Jessica~Zosa Forde, Jonathan Frankle, Ziliang Zong, and Matthew~L Leavitt.
\newblock Reduce, reuse, recycle: Improving training efficiency with distillation.
\newblock \emph{arXiv preprint arXiv:2211.00683}, 2022.

\bibitem[Qiu et~al.(2025)Qiu, Xiao, Wilson, Pennington, and Agarwala]{qiu2025scaling}
Shikai Qiu, Lechao Xiao, Andrew~Gordon Wilson, Jeffrey Pennington, and Atish Agarwala.
\newblock Scaling collapse reveals universal dynamics in compute-optimally trained neural networks.
\newblock In \emph{Forty-second International Conference on Machine Learning}, 2025.
\newblock URL \url{https://openreview.net/forum?id=Fvq9ogLnLN}.

\bibitem[Paquette et~al.(2024)Paquette, Paquette, Xiao, and Pennington]{paquette2024}
Elliot Paquette, Courtney Paquette, Lechao Xiao, and Jeffrey Pennington.
\newblock 4+3 phases of compute-optimal neural scaling laws.
\newblock \emph{Advances in Neural Information Processing Systems}, 37:\penalty0 16459--16537, 2024.

\bibitem[{Team OLMo} et~al.(2024){Team OLMo}, Walsh, Soldaini, Groeneveld, Lo, Arora, Bhagia, Gu, Huang, Jordan, et~al.]{olmo20242olmo2furious}
{Team OLMo}, Pete Walsh, Luca Soldaini, Dirk Groeneveld, Kyle Lo, Shane Arora, Akshita Bhagia, Yuling Gu, Shengyi Huang, Matt Jordan, et~al.
\newblock {2 OLMo 2 Furious}, 2024.
\newblock URL \url{https://arxiv.org/abs/2501.00656}.

\bibitem[Zhang et~al.(2026)Zhang, Saxe, and Latham]{zhang2025saddle}
Yedi Zhang, Andrew~M Saxe, and Peter~E. Latham.
\newblock Saddle-to-saddle dynamics explains a simplicity bias across neural network architectures.
\newblock In \emph{The Fourteenth International Conference on Learning Representations}, 2026.
\newblock URL \url{https://openreview.net/forum?id=Vit5M0G5Gb}.

\bibitem[Abbe et~al.(2023)Abbe, Adsera, and Misiakiewicz]{abbe2023sgd}
Emmanuel Abbe, Enric~Boix Adsera, and Theodor Misiakiewicz.
\newblock Sgd learning on neural networks: leap complexity and saddle-to-saddle dynamics.
\newblock In \emph{The Thirty Sixth Annual Conference on Learning Theory}, pages 2552--2623. PMLR, 2023.

\bibitem[Jacot et~al.(2021)Jacot, Ged, {\c{S}}im{\c{s}}ek, Hongler, and Gabriel]{jacot2021saddle}
Arthur Jacot, Fran{\c{c}}ois Ged, Berfin {\c{S}}im{\c{s}}ek, Cl{\'e}ment Hongler, and Franck Gabriel.
\newblock Saddle-to-saddle dynamics in deep linear networks: Small initialization training, symmetry, and sparsity.
\newblock \emph{arXiv preprint arXiv:2106.15933}, 2021.

\bibitem[Kunin et~al.(2026)Kunin, Marchetti, Chen, Karkada, Simon, DeWeese, Ganguli, and Miolane]{kunin2025alternating}
Daniel Kunin, Giovanni~Luca Marchetti, Feng Chen, Dhruva Karkada, James~B Simon, Michael~R DeWeese, Surya Ganguli, and Nina Miolane.
\newblock Alternating gradient flows: A theory of feature learning in two-layer neural networks.
\newblock In \emph{The Thirty-ninth Annual Conference on Neural Information Processing Systems}, 2026.
\newblock URL \url{https://openreview.net/forum?id=t7LKc0MMW6}.

\bibitem[Jagielski et~al.(2023)Jagielski, Thakkar, Tramer, Ippolito, Lee, Carlini, Wallace, Song, Thakurta, Papernot, and Zhang]{jagielski2023measuring}
Matthew Jagielski, Om~Thakkar, Florian Tramer, Daphne Ippolito, Katherine Lee, Nicholas Carlini, Eric Wallace, Shuang Song, Abhradeep~Guha Thakurta, Nicolas Papernot, and Chiyuan Zhang.
\newblock Measuring forgetting of memorized training examples.
\newblock In \emph{The Eleventh International Conference on Learning Representations}, 2023.
\newblock URL \url{https://openreview.net/forum?id=7bJizxLKrR}.

\bibitem[Carlini et~al.(2019)Carlini, Liu, Erlingsson, Kos, and Song]{carlini2019secret}
Nicholas Carlini, Chang Liu, {\'U}lfar Erlingsson, Jernej Kos, and Dawn Song.
\newblock The secret sharer: Evaluating and testing unintended memorization in neural networks.
\newblock In \emph{28th USENIX Security Symposium (USENIX Security 19)}, pages 267--284, Santa Clara, CA, August 2019. USENIX Association.
\newblock ISBN 978-1-939133-06-9.
\newblock URL \url{https://www.usenix.org/conference/usenixsecurity19/presentation/carlini}.

\bibitem[Huang et~al.(2024)Huang, Yang, and Potts]{huang2024demystifying}
Jing Huang, Diyi Yang, and Christopher Potts.
\newblock Demystifying verbatim memorization in large language models.
\newblock In \emph{Proceedings of the 2024 Conference on Empirical Methods in Natural Language Processing}, pages 10711--10732, 2024.

\bibitem[Wei et~al.(2026)Wei, Godbole, Khan, Wang, Zhu, Flemings, Kashyap, Gummadi, Neiswanger, and Jia]{wei2026hubble}
Johnny Wei, Ameya Godbole, Mohammad~Aflah Khan, Ryan~Yixiang Wang, Xiaoyuan Zhu, James Flemings, Nitya Kashyap, Krishna~P. Gummadi, Willie Neiswanger, and Robin Jia.
\newblock Hubble: a model suite to advance the study of {LLM} memorization.
\newblock In \emph{The Fourteenth International Conference on Learning Representations}, 2026.
\newblock URL \url{https://openreview.net/forum?id=ZfdnZhOP0k}.

\bibitem[Hwang and Park(2026)]{hwang2026intrinsictasksymmetrydrives}
Hyeonbin Hwang and Yeachan Park.
\newblock Intrinsic task symmetry drives generalization in algorithmic tasks, 2026.
\newblock URL \url{https://arxiv.org/abs/2603.01968}.

\bibitem[Soldaini et~al.(2024)Soldaini, Kinney, Bhagia, Schwenk, Atkinson, Authur, Bogin, Chandu, Dumas, Elazar, et~al.]{dolma}
Luca Soldaini, Rodney Kinney, Akshita Bhagia, Dustin Schwenk, David Atkinson, Russell Authur, Ben Bogin, Khyathi Chandu, Jennifer Dumas, Yanai Elazar, et~al.
\newblock {Dolma: An Open Corpus of Three Trillion Tokens for Language Model Pretraining Research}.
\newblock \emph{arXiv preprint}, 2024.
\newblock URL \url{https://huggingface.co/datasets/allenai/dolma}.

\bibitem[Groeneveld et~al.(2024)Groeneveld, Beltagy, Walsh, Bhagia, Kinney, Tafjord, Jha, Ivison, Magnusson, Wang, et~al.]{groeneveld2023olmo}
Dirk Groeneveld, Iz~Beltagy, Pete Walsh, Akshita Bhagia, Rodney Kinney, Oyvind Tafjord, Ananya~Harsh Jha, Hamish Ivison, Ian Magnusson, Yizhong Wang, et~al.
\newblock {OLM}o: Accelerating the science of language models.
\newblock In \emph{Proceedings of the 62nd Annual Meeting of the Association for Computational Linguistics (Volume 1: Long Papers)}, pages 15789--15809, Bangkok, Thailand, August 2024. Association for Computational Linguistics.
\newblock \doi{10.18653/v1/2024.acl-long.841}.
\newblock URL \url{https://aclanthology.org/2024.acl-long.841/}.

\bibitem[Power et~al.(2022)Power, Burda, Edwards, Babuschkin, and Misra]{power2022grokking}
Alethea Power, Yuri Burda, Harri Edwards, Igor Babuschkin, and Vedant Misra.
\newblock Grokking: Generalization beyond overfitting on small algorithmic datasets.
\newblock \emph{arXiv preprint arXiv:2201.02177}, 2022.

\bibitem[Nanda et~al.(2022)Nanda, Chan, Lieberum, Smith, and Steinhardt]{nanda_progress_2022}
Neel Nanda, Lawrence Chan, Tom Lieberum, Jess Smith, and Jacob Steinhardt.
\newblock {Progress measures for grokking via mechanistic interpretability}.
\newblock In \emph{The Eleventh International Conference on Learning Representations}, sep 2022.
\newblock URL \url{https://openreview.net/forum?id=9XFSbDPmdW}.

\bibitem[Zhou et~al.(2024)Zhou, Fu, Sharan, and Jia]{zhou2024pretrained}
Tianyi Zhou, Deqing Fu, Vatsal Sharan, and Robin Jia.
\newblock Pre-trained large language models use fourier features to compute addition.
\newblock In \emph{The Thirty-eighth Annual Conference on Neural Information Processing Systems}, 2024.
\newblock URL \url{https://openreview.net/forum?id=i4MutM2TZb}.

\bibitem[Feucht et~al.(2026)Feucht, Haklay, Bhalla, Wurgaft, Rager, Sarfati, Merullo, McGrath, Lewis, Lubana, Fel, and Geiger]{feucht2026arithmeticwildllamauses}
Sheridan Feucht, Tal Haklay, Usha Bhalla, Daniel Wurgaft, Can Rager, Raphaël Sarfati, Jack Merullo, Thomas McGrath, Owen Lewis, Ekdeep~Singh Lubana, Thomas Fel, and Atticus Geiger.
\newblock Arithmetic in the wild: Llama uses base-10 addition to reason about cyclic concepts, 2026.
\newblock URL \url{https://arxiv.org/abs/2605.01148}.

\bibitem[Geiger et~al.(2024)Geiger, Wu, Potts, Icard, and Goodman]{geiger2024das}
Atticus Geiger, Zhengxuan Wu, Christopher Potts, Thomas Icard, and Noah Goodman.
\newblock Finding alignments between interpretable causal variables and distributed neural representations.
\newblock In Francesco Locatello and Vanessa Didelez, editors, \emph{Proceedings of the Third Conference on Causal Learning and Reasoning}, volume 236 of \emph{Proceedings of Machine Learning Research}, pages 160--187. PMLR, 01--03 Apr 2024.
\newblock URL \url{https://proceedings.mlr.press/v236/geiger24a.html}.

\bibitem[Feldman(2020)]{feldman2020does}
Vitaly Feldman.
\newblock Does learning require memorization? a short tale about a long tail.
\newblock In \emph{Proceedings of the 52nd annual ACM SIGACT symposium on theory of computing}, pages 954--959, 2020.

\bibitem[Raffel et~al.(2020)Raffel, Shazeer, Roberts, Lee, Narang, Matena, Zhou, Li, and Liu]{raffel2020exploring}
Colin Raffel, Noam Shazeer, Adam Roberts, Katherine Lee, Sharan Narang, Michael Matena, Yanqi Zhou, Wei Li, and Peter~J Liu.
\newblock Exploring the limits of transfer learning with a unified text-to-text transformer.
\newblock \emph{Journal of machine learning research}, 21\penalty0 (140):\penalty0 1--67, 2020.

\bibitem[Xie et~al.(2023)Xie, Santurkar, Ma, and Liang]{xie2023data}
Sang~Michael Xie, Shibani Santurkar, Tengyu Ma, and Percy~S Liang.
\newblock Data selection for language models via importance resampling.
\newblock \emph{Advances in Neural Information Processing Systems}, 36:\penalty0 34201--34227, 2023.

\bibitem[Xie et~al.(2024)Xie, Pham, Dong, Du, Liu, Lu, Liang, Le, Ma, and Yu]{xie2024doremi}
Sang~Michael Xie, Hieu Pham, Xuanyi Dong, Nan Du, Hanxiao Liu, Yifeng Lu, Percy~S Liang, Quoc~V Le, Tengyu Ma, and Adams~Wei Yu.
\newblock Doremi: Optimizing data mixtures speeds up language model pretraining.
\newblock \emph{Advances in Neural Information Processing Systems}, 36, 2024.

\bibitem[Xie(2024)]{xie2024foundation}
Sang~Michael Xie.
\newblock \emph{Foundation Models from a Data-Distributional View}.
\newblock Stanford University, 2024.

\bibitem[Ramesh et~al.(2022)Ramesh, Mao, Griniasty, Yang, Teoh, Transtrum, Sethna, and Chaudhari]{ramesh2022picture}
Rahul Ramesh, Jialin Mao, Itay Griniasty, Rubing Yang, Han~Kheng Teoh, Mark Transtrum, James~P Sethna, and Pratik Chaudhari.
\newblock A picture of the space of typical learnable tasks.
\newblock \emph{arXiv preprint arXiv:2210.17011}, 2022.

\bibitem[Ramesh(2025)]{ramesh2025principles}
Rahul Ramesh.
\newblock \emph{The Principles of Learning on Multiple Tasks}.
\newblock PhD thesis, University of Pennsylvania, 2025.

\bibitem[Penedo et~al.(2023)Penedo, Malartic, Hesslow, Cojocaru, Alobeidli, Cappelli, Pannier, Almazrouei, and Launay]{penedo2023refinedweb}
Guilherme Penedo, Quentin Malartic, Daniel Hesslow, Ruxandra Cojocaru, Hamza Alobeidli, Alessandro Cappelli, Baptiste Pannier, Ebtesam Almazrouei, and Julien Launay.
\newblock The refinedweb dataset for falcon {LLM}: Outperforming curated corpora with web data only.
\newblock In \emph{Thirty-seventh Conference on Neural Information Processing Systems Datasets and Benchmarks Track}, 2023.
\newblock URL \url{https://openreview.net/forum?id=kM5eGcdCzq}.

\bibitem[Penedo et~al.(2024)Penedo, Kydl{\'\i}{\v{c}}ek, Lozhkov, Mitchell, Raffel, Von~Werra, Wolf, et~al.]{penedo2024fineweb}
Guilherme Penedo, Hynek Kydl{\'\i}{\v{c}}ek, Anton Lozhkov, Margaret Mitchell, Colin Raffel, Leandro Von~Werra, Thomas Wolf, et~al.
\newblock The fineweb datasets: Decanting the web for the finest text data at scale.
\newblock \emph{Advances in Neural Information Processing Systems}, 37:\penalty0 30811--30849, 2024.

\bibitem[Maini et~al.(2025)Maini, Dorna, Doshi, Carranza, Pan, Urbanek, Burstein, Fang, Deng, Abbas, et~al.]{maini2025beyondweb}
Pratyush Maini, Vineeth Dorna, Parth Doshi, Aldo Carranza, Fan Pan, Jack Urbanek, Paul Burstein, Alex Fang, Alvin Deng, Amro Abbas, et~al.
\newblock Beyondweb: Lessons from scaling synthetic data for trillion-scale pretraining.
\newblock \emph{arXiv preprint arXiv:2508.10975}, 2025.

\bibitem[Sam et~al.(2026)Sam, Goyal, Maini, Robey, and Kolter]{sam2026should}
Dylan Sam, Sachin Goyal, Pratyush Maini, Alexander Robey, and J~Zico Kolter.
\newblock When should we introduce safety interventions during pretraining?
\newblock \emph{arXiv preprint arXiv:2601.07087}, 2026.

\bibitem[Goyal et~al.(2024)Goyal, Maini, Lipton, Raghunathan, and Kolter]{goyal2024scaling}
Sachin Goyal, Pratyush Maini, Zachary~C Lipton, Aditi Raghunathan, and J~Zico Kolter.
\newblock Scaling laws for data filtering--data curation cannot be compute agnostic.
\newblock In \emph{Proceedings of the IEEE/CVF Conference on Computer Vision and Pattern Recognition}, pages 22702--22711, 2024.

\bibitem[Caruana(1997)]{caruana1997multitask}
Rich Caruana.
\newblock Multitask learning.
\newblock \emph{Machine learning}, 28\penalty0 (1):\penalty0 41--75, 1997.

\bibitem[Aljundi(2019)]{aljundi2019continual}
Rahaf Aljundi.
\newblock Continual learning in neural networks.
\newblock \emph{arXiv preprint arXiv:1910.02718}, 2019.

\bibitem[Liu et~al.(2021)Liu, Liu, Jin, Stone, and Liu]{liu2021conflict}
Bo~Liu, Xingchao Liu, Xiaojie Jin, Peter Stone, and Qiang Liu.
\newblock Conflict-averse gradient descent for multi-task learning.
\newblock \emph{Advances in neural information processing systems}, 34:\penalty0 18878--18890, 2021.

\bibitem[Aljundi et~al.(2019)Aljundi, Lin, Goujaud, and Bengio]{aljundi2019gradient}
Rahaf Aljundi, Min Lin, Baptiste Goujaud, and Yoshua Bengio.
\newblock Gradient based sample selection for online continual learning.
\newblock \emph{Advances in neural information processing systems}, 32, 2019.

\bibitem[Wu et~al.(2026)Wu, Samanta, Jain, Fujimoto, Kwon, Kretzu, Yu, Hassani, Vidolov, and Efroni]{wu2026imbalanced}
Runzhe Wu, Ankur Samanta, Ayush Jain, Scott Fujimoto, Jeongyeol Kwon, Ben Kretzu, Youliang Yu, Kaveh Hassani, Boris Vidolov, and Yonathan Efroni.
\newblock Imbalanced gradients in rl post-training of multi-task llms.
\newblock In \emph{Findings of the Association for Computational Linguistics: EACL 2026}, pages 3137--3150, 2026.

\bibitem[Pezeshki et~al.(2021)Pezeshki, Kaba, Bengio, Courville, Precup, and Lajoie]{pezeshki2021gradient}
Mohammad Pezeshki, Oumar Kaba, Yoshua Bengio, Aaron~C Courville, Doina Precup, and Guillaume Lajoie.
\newblock Gradient starvation: A learning proclivity in neural networks.
\newblock \emph{Advances in Neural Information Processing Systems}, 34:\penalty0 1256--1272, 2021.

\bibitem[Evron et~al.(2022)Evron, Moroshko, Ward, Srebro, and Soudry]{evron2022catastrophic}
Itay Evron, Edward Moroshko, Rachel Ward, Nathan Srebro, and Daniel Soudry.
\newblock How catastrophic can catastrophic forgetting be in linear regression?
\newblock In \emph{Conference on Learning Theory}, pages 4028--4079. PMLR, 2022.

\bibitem[Marek et~al.(2026)Marek, Cho, Qiu, Chunara, Izmailov, and Wilson]{marek2026forgettinglanguagemodelscapacity}
Martin Marek, Dongkyu Cho, Shikai Qiu, Rumi Chunara, Pavel Izmailov, and Andrew~Gordon Wilson.
\newblock Forgetting in language models: Capacity, optimization, and self-generated replay, 2026.
\newblock URL \url{https://arxiv.org/abs/2605.26097}.

\bibitem[Yu et~al.(2020)Yu, Kumar, Gupta, Levine, Hausman, and Finn]{yu2020gradient}
Tianhe Yu, Saurabh Kumar, Abhishek Gupta, Sergey Levine, Karol Hausman, and Chelsea Finn.
\newblock Gradient surgery for multi-task learning.
\newblock \emph{Advances in neural information processing systems}, 33:\penalty0 5824--5836, 2020.

\bibitem[Sener and Koltun(2018)]{sener2018multi}
Ozan Sener and Vladlen Koltun.
\newblock Multi-task learning as multi-objective optimization.
\newblock \emph{Advances in neural information processing systems}, 31, 2018.

\bibitem[Chen et~al.(2018)Chen, Badrinarayanan, Lee, and Rabinovich]{chen2018gradnorm}
Zhao Chen, Vijay Badrinarayanan, Chen-Yu Lee, and Andrew Rabinovich.
\newblock Gradnorm: Gradient normalization for adaptive loss balancing in deep multitask networks.
\newblock In \emph{International conference on machine learning}, pages 794--803. PMLR, 2018.

\bibitem[Suteu and Guo(2019)]{suteu2019regularizing}
Mihai Suteu and Yike Guo.
\newblock Regularizing deep multi-task networks using orthogonal gradients.
\newblock \emph{arXiv preprint arXiv:1912.06844}, 2019.

\bibitem[Farajtabar et~al.(2020)Farajtabar, Azizan, Mott, and Li]{farajtabar2020orthogonal}
Mehrdad Farajtabar, Navid Azizan, Alex Mott, and Ang Li.
\newblock Orthogonal gradient descent for continual learning.
\newblock In \emph{International conference on artificial intelligence and statistics}, pages 3762--3773. PMLR, 2020.

\bibitem[Chen et~al.(2026)Chen, Li, Chen, and Lin]{chen2026reward}
Peter~L Chen, Xiaopeng Li, Xi~Chen, and Tianyi Lin.
\newblock Reward-free alignment for conflicting objectives.
\newblock \emph{arXiv preprint arXiv:2602.02495}, 2026.

\bibitem[Ramasesh et~al.(2022)Ramasesh, Lewkowycz, and Dyer]{ramasesh2021effect}
Vinay~Venkatesh Ramasesh, Aitor Lewkowycz, and Ethan Dyer.
\newblock Effect of scale on catastrophic forgetting in neural networks.
\newblock In \emph{International Conference on Learning Representations}, 2022.
\newblock URL \url{https://openreview.net/forum?id=GhVS8_yPeEa}.

\bibitem[Doshi et~al.(2024)Doshi, He, Das, and Gromov]{doshi2024grokking}
Darshil Doshi, Tianyu He, Aritra Das, and Andrey Gromov.
\newblock Grokking modular polynomials.
\newblock \emph{arXiv preprint arXiv:2406.03495}, 2024.

\bibitem[Gopalani et~al.(2024)Gopalani, Lubana, and Hu]{gopalani2024abrupt}
Pulkit Gopalani, Ekdeep~S Lubana, and Wei Hu.
\newblock Abrupt learning in transformers: A case study on matrix completion.
\newblock \emph{Advances in Neural Information Processing Systems}, 37:\penalty0 55053--55085, 2024.

\bibitem[Murty et~al.(2023)Murty, Sharma, Andreas, and Manning]{murty2023grokking}
Shikhar Murty, Pratyusha Sharma, Jacob Andreas, and Christopher Manning.
\newblock Grokking of hierarchical structure in vanilla transformers.
\newblock In \emph{Proceedings of the 61st Annual Meeting of the Association for Computational Linguistics (Volume 2: Short Papers)}, pages 439--448. Association for Computational Linguistics, July 2023.
\newblock URL \url{https://aclanthology.org/2023.acl-short.38/}.

\bibitem[Kumar et~al.(2024)Kumar, Bordelon, Gershman, and Pehlevan]{kumar2023grokking}
Tanishq Kumar, Blake Bordelon, Samuel~J. Gershman, and Cengiz Pehlevan.
\newblock Grokking as the transition from lazy to rich training dynamics.
\newblock In \emph{The Twelfth International Conference on Learning Representations}, 2024.
\newblock URL \url{https://openreview.net/forum?id=vt5mnLVIVo}.

\bibitem[Stander et~al.(2024)Stander, Yu, Fan, and Biderman]{stander2023grokking}
Dashiell Stander, Qinan Yu, Honglu Fan, and Stella Biderman.
\newblock Grokking group multiplication with cosets.
\newblock In \emph{Forty-first International Conference on Machine Learning}, 2024.
\newblock URL \url{https://openreview.net/forum?id=hcQfTsVnBo}.

\bibitem[Mohamadi et~al.(2024)Mohamadi, Li, Wu, and Sutherland]{mohamadi2024you}
Mohamad~Amin Mohamadi, Zhiyuan Li, Lei Wu, and Danica~J. Sutherland.
\newblock Why do you grok? a theoretical analysis on grokking modular addition.
\newblock In \emph{Forty-first International Conference on Machine Learning}, 2024.
\newblock URL \url{https://openreview.net/forum?id=ad5I6No9G1}.

\bibitem[Varma et~al.(2023)Varma, Shah, Kenton, Kram{\'a}r, and Kumar]{varma2023explaining}
Vikrant Varma, Rohin Shah, Zachary Kenton, J{\'a}nos Kram{\'a}r, and Ramana Kumar.
\newblock Explaining grokking through circuit efficiency.
\newblock \emph{arXiv preprint arXiv:2309.02390}, 2023.

\bibitem[Morwani et~al.(2024)Morwani, Edelman, Oncescu, Zhao, and Kakade]{morwani2023feature}
Depen Morwani, Benjamin~L. Edelman, Costin-Andrei Oncescu, Rosie Zhao, and Sham~M. Kakade.
\newblock Feature emergence via margin maximization: case studies in algebraic tasks.
\newblock In \emph{The Twelfth International Conference on Learning Representations}, 2024.
\newblock URL \url{https://openreview.net/forum?id=i9wDX850jR}.

\bibitem[Chen et~al.(2024)Chen, Shwartz-Ziv, Cho, Leavitt, and Saphra]{chen2024sudden}
Angelica Chen, Ravid Shwartz-Ziv, Kyunghyun Cho, Matthew~L Leavitt, and Naomi Saphra.
\newblock Sudden drops in the loss: Syntax acquisition, phase transitions, and simplicity bias in {MLM}s.
\newblock In \emph{The Twelfth International Conference on Learning Representations}, 2024.
\newblock URL \url{https://openreview.net/forum?id=MO5PiKHELW}.

\bibitem[Cheng et~al.(2022)Cheng, Duchi, and Kuditipudi]{cheng2022memorize}
Chen Cheng, John Duchi, and Rohith Kuditipudi.
\newblock Memorize to generalize: on the necessity of interpolation in high dimensional linear regression.
\newblock In \emph{Conference on Learning Theory}, pages 5528--5560. PMLR, 2022.

\bibitem[Brown et~al.(2021)Brown, Bun, Feldman, Smith, and Talwar]{brown2021memorization}
Gavin Brown, Mark Bun, Vitaly Feldman, Adam Smith, and Kunal Talwar.
\newblock When is memorization of irrelevant training data necessary for high-accuracy learning?
\newblock In \emph{Proceedings of the 53rd annual ACM SIGACT symposium on theory of computing}, pages 123--132, 2021.

\bibitem[Mei and Montanari(2022)]{mei2022generalization}
Song Mei and Andrea Montanari.
\newblock The generalization error of random features regression: Precise asymptotics and the double descent curve.
\newblock \emph{Communications on Pure and Applied Mathematics}, 75\penalty0 (4):\penalty0 667--766, 2022.

\bibitem[Loog et~al.(2020)Loog, Viering, Mey, Krijthe, and Tax]{loog2020brief}
Marco Loog, Tom Viering, Alexander Mey, Jesse~H Krijthe, and David~MJ Tax.
\newblock A brief prehistory of double descent.
\newblock \emph{Proceedings of the National Academy of Sciences}, 117\penalty0 (20):\penalty0 10625--10626, 2020.

\bibitem[Nakkiran et~al.(2020)Nakkiran, Kaplun, Bansal, Yang, Barak, and Sutskever]{nakkiran2021deep}
Preetum Nakkiran, Gal Kaplun, Yamini Bansal, Tristan Yang, Boaz Barak, and Ilya Sutskever.
\newblock Deep double descent: Where bigger models and more data hurt.
\newblock In \emph{International Conference on Learning Representations}, 2020.
\newblock URL \url{https://openreview.net/forum?id=B1g5sA4twr}.

\bibitem[Nakkiran(2019)]{nakkiran2019more}
Preetum Nakkiran.
\newblock More data can hurt for linear regression: Sample-wise double descent.
\newblock \emph{arXiv preprint arXiv:1912.07242}, 2019.

\bibitem[Nakkiran et~al.(2021)Nakkiran, Venkat, Kakade, and Ma]{nakkiran2020optimal}
Preetum Nakkiran, Prayaag Venkat, Sham~M. Kakade, and Tengyu Ma.
\newblock Optimal regularization can mitigate double descent.
\newblock In \emph{International Conference on Learning Representations}, 2021.
\newblock URL \url{https://openreview.net/forum?id=7R7fAoUygoa}.

\bibitem[Wurgaft et~al.(2026)Wurgaft, Lubana, Park, Tanaka, Reddy, and Goodman]{wurgaft2025context}
Daniel Wurgaft, Ekdeep~Singh Lubana, Core~Francisco Park, Hidenori Tanaka, Gautam Reddy, and Noah Goodman.
\newblock In-context learning strategies emerge rationally.
\newblock In \emph{The Thirty-ninth Annual Conference on Neural Information Processing Systems}, 2026.
\newblock URL \url{https://openreview.net/forum?id=bBUUOQI0N6}.

\bibitem[Singh et~al.(2024)Singh, Moskovitz, Hill, Chan, and Saxe]{singh2024needs}
Aaditya~K Singh, Ted Moskovitz, Felix Hill, Stephanie~C.Y. Chan, and Andrew~M Saxe.
\newblock What needs to go right for an induction head? a mechanistic study of in-context learning circuits and their formation.
\newblock In \emph{Forty-first International Conference on Machine Learning}, 2024.
\newblock URL \url{https://openreview.net/forum?id=O8rrXl71D5}.

\bibitem[Singh et~al.(2023)Singh, Chan, Moskovitz, Grant, Saxe, and Hill]{singh2023transient}
Aaditya Singh, Stephanie Chan, Ted Moskovitz, Erin Grant, Andrew Saxe, and Felix Hill.
\newblock The transient nature of emergent in-context learning in transformers.
\newblock \emph{Advances in neural information processing systems}, 36:\penalty0 27801--27819, 2023.

\bibitem[Park et~al.(2025)Park, Lubana, and Tanaka]{park2024competition}
Core~Francisco Park, Ekdeep~Singh Lubana, and Hidenori Tanaka.
\newblock Competition dynamics shape algorithmic phases of in-context learning.
\newblock In \emph{The Thirteenth International Conference on Learning Representations}, 2025.
\newblock URL \url{https://openreview.net/forum?id=XgH1wfHSX8}.

\bibitem[Kandpal et~al.(2023)Kandpal, Deng, Roberts, Wallace, and Raffel]{kandpal2023large}
Nikhil Kandpal, Haikang Deng, Adam Roberts, Eric Wallace, and Colin Raffel.
\newblock Large language models struggle to learn long-tail knowledge.
\newblock In \emph{International conference on machine learning}, pages 15696--15707. PMLR, 2023.

\bibitem[Lesci et~al.(2024)Lesci, Meister, Hofmann, Vlachos, and Pimentel]{lesci-etal-2024-causal}
Pietro Lesci, Clara Meister, Thomas Hofmann, Andreas Vlachos, and Tiago Pimentel.
\newblock Causal estimation of memorisation profiles.
\newblock In \emph{Proceedings of the 62nd Annual Meeting of the Association for Computational Linguistics (Volume 1: Long Papers)}, pages 15616--15635, Bangkok, Thailand, August 2024. Association for Computational Linguistics.
\newblock \doi{10.18653/v1/2024.acl-long.834}.
\newblock URL \url{https://aclanthology.org/2024.acl-long.834/}.

\bibitem[Carlini et~al.(2023)Carlini, Ippolito, Jagielski, Lee, Tramer, and Zhang]{carlini2023quantifying}
Nicholas Carlini, Daphne Ippolito, Matthew Jagielski, Katherine Lee, Florian Tramer, and Chiyuan Zhang.
\newblock Quantifying memorization across neural language models.
\newblock In \emph{The Eleventh International Conference on Learning Representations}, 2023.
\newblock URL \url{https://openreview.net/forum?id=TatRHT_1cK}.

\bibitem[Kuditipudi et~al.(2026)Kuditipudi, Huang, Zhu, Yang, Potts, and Liang]{kuditipudi2025blackbox}
Rohith Kuditipudi, Jing Huang, Sally Zhu, Diyi Yang, Christopher Potts, and Percy Liang.
\newblock Blackbox model provenance via palimpsestic membership inference.
\newblock In \emph{The Thirty-ninth Annual Conference on Neural Information Processing Systems}, 2026.
\newblock URL \url{https://openreview.net/forum?id=VRhVS59yhP}.

\bibitem[Krasheninnikov et~al.(2026)Krasheninnikov, Turner, and Krueger]{krasheninnikov2025fresh}
Dmitrii Krasheninnikov, Richard~E. Turner, and David Krueger.
\newblock Fresh in memory: Training-order recency is linearly encoded in language model activations.
\newblock In \emph{The Fourteenth International Conference on Learning Representations}, 2026.
\newblock URL \url{https://openreview.net/forum?id=Tn6famjSxN}.

\bibitem[Duan et~al.(2025)Duan, Khona, Iyer, Schaeffer, and Fiete]{duan2025uncovering}
Sunny Duan, Mikail Khona, Abhiram Iyer, Rylan Schaeffer, and Ila~R Fiete.
\newblock Uncovering latent memories in large language models.
\newblock In \emph{The Thirteenth International Conference on Learning Representations}, 2025.

\bibitem[Chang et~al.(2024)Chang, Park, Ye, Yang, Seo, Chang, and Seo]{chang2024largelanguagemodelsacquire}
Hoyeon Chang, Jinho Park, Seonghyeon Ye, Sohee Yang, Youngkyung Seo, Du-Seong Chang, and Minjoon Seo.
\newblock How do large language models acquire factual knowledge during pretraining?
\newblock In \emph{The Thirty-eighth Annual Conference on Neural Information Processing Systems}, 2024.
\newblock URL \url{https://openreview.net/forum?id=TYdzj1EvBP}.

\bibitem[Tirumala et~al.(2022)Tirumala, Markosyan, Zettlemoyer, and Aghajanyan]{tirumala2022memorization}
Kushal Tirumala, Aram Markosyan, Luke Zettlemoyer, and Armen Aghajanyan.
\newblock Memorization without overfitting: Analyzing the training dynamics of large language models.
\newblock \emph{Advances in Neural Information Processing Systems}, 35:\penalty0 38274--38290, 2022.

\bibitem[Hernandez et~al.(2022)Hernandez, Brown, Conerly, DasSarma, Drain, El-Showk, Elhage, Hatfield-Dodds, Henighan, Hume, et~al.]{hernandez2022scaling}
Danny Hernandez, Tom Brown, Tom Conerly, Nova DasSarma, Dawn Drain, Sheer El-Showk, Nelson Elhage, Zac Hatfield-Dodds, Tom Henighan, Tristan Hume, et~al.
\newblock Scaling laws and interpretability of learning from repeated data.
\newblock \emph{arXiv preprint arXiv:2205.10487}, 2022.

\bibitem[Piantadosi(2014)]{piantadosi2014zipf}
Steven~T Piantadosi.
\newblock Zipf’s word frequency law in natural language: A critical review and future directions.
\newblock \emph{Psychonomic bulletin \& review}, 21\penalty0 (5):\penalty0 1112--1130, 2014.

\bibitem[Hyv{\"a}rinen et~al.(2009)Hyv{\"a}rinen, Hurri, and Hoyer]{hyvarinen2009natural}
Aapo Hyv{\"a}rinen, Jarmo Hurri, and Patrick~O Hoyer.
\newblock \emph{Natural image statistics: A probabilistic approach to early computational vision.}, volume~39.
\newblock Springer Science \& Business Media, 2009.

\bibitem[Atanasov et~al.(2024)Atanasov, Zavatone-Veth, and Pehlevan]{atanasov2024scaling}
Alexander Atanasov, Jacob~A Zavatone-Veth, and Cengiz Pehlevan.
\newblock Scaling and renormalization in high-dimensional regression.
\newblock \emph{arXiv preprint arXiv:2405.00592}, 2024.

\bibitem[Ren et~al.(2026)Ren, Nichani, Wu, and Lee]{ren2025emergence}
Yunwei Ren, Eshaan Nichani, Denny Wu, and Jason~D. Lee.
\newblock Emergence and scaling laws in {SGD} learning of shallow neural networks.
\newblock In \emph{The Thirty-ninth Annual Conference on Neural Information Processing Systems}, 2026.
\newblock URL \url{https://openreview.net/forum?id=kA2H90nm26}.

\bibitem[Everett et~al.(2024)Everett, Xiao, Wortsman, Alemi, Novak, Liu, Gur, Sohl-Dickstein, Kaelbling, Lee, et~al.]{everett2024scaling}
Katie Everett, Lechao Xiao, Mitchell Wortsman, Alexander~A Alemi, Roman Novak, Peter~J Liu, Izzeddin Gur, Jascha Sohl-Dickstein, Leslie~Pack Kaelbling, Jaehoon Lee, et~al.
\newblock Scaling exponents across parameterizations and optimizers.
\newblock In \emph{Forty-first International Conference on Machine Learning}, 2024.
\newblock URL \url{https://openreview.net/forum?id=0ksNeD1SJT}.

\bibitem[Michaud et~al.(2025)Michaud, Gorton, and McGrath]{michaud2025understanding}
Eric~J Michaud, Liv Gorton, and Tom McGrath.
\newblock Understanding sparse autoencoder scaling in the presence of feature manifolds.
\newblock \emph{arXiv preprint arXiv:2509.02565}, 2025.

\bibitem[Nam et~al.(2024)Nam, Fonseca, Lee, Mingard, and Louis]{nam2024exactly}
Yoonsoo Nam, Nayara Fonseca, Seok~H Lee, Chris Mingard, and Ard~A Louis.
\newblock An exactly solvable model for emergence and scaling laws in the multitask sparse parity problem.
\newblock \emph{Advances in Neural Information Processing Systems}, 37:\penalty0 39632--39693, 2024.

\bibitem[Frankle and Carbin(2019)]{frankle2018lottery}
Jonathan Frankle and Michael Carbin.
\newblock The lottery ticket hypothesis: Finding sparse, trainable neural networks.
\newblock In \emph{International Conference on Learning Representations}, 2019.
\newblock URL \url{https://openreview.net/forum?id=rJl-b3RcF7}.

\bibitem[Malach et~al.(2020)Malach, Yehudai, Shalev-Schwartz, and Shamir]{malach2020proving}
Eran Malach, Gilad Yehudai, Shai Shalev-Schwartz, and Ohad Shamir.
\newblock Proving the lottery ticket hypothesis: Pruning is all you need.
\newblock In \emph{International conference on machine learning}, pages 6682--6691. PMLR, 2020.

\bibitem[Pensia et~al.(2020)Pensia, Rajput, Nagle, Vishwakarma, and Papailiopoulos]{pensia2020optimal}
Ankit Pensia, Shashank Rajput, Alliot Nagle, Harit Vishwakarma, and Dimitris Papailiopoulos.
\newblock Optimal lottery tickets via subset sum: Logarithmic over-parameterization is sufficient.
\newblock \emph{Advances in neural information processing systems}, 33:\penalty0 2599--2610, 2020.

\bibitem[Magnusson et~al.(2025)Magnusson, Tai, Bogin, Heineman, Hwang, Soldaini, Bhagia, Liu, Groeneveld, Tafjord, Smith, Koh, and Dodge]{magnussondatadecide2025}
Ian Magnusson, Nguyen Tai, Ben Bogin, David Heineman, Jena~D. Hwang, Luca Soldaini, Akshita Bhagia, Jiacheng Liu, Dirk Groeneveld, Oyvind Tafjord, Noah~A. Smith, Pang~Wei Koh, and Jesse Dodge.
\newblock Datadecide: How to predict best pretraining data with small experiments.
\newblock In \emph{Forty-second International Conference on Machine Learning}, 2025.
\newblock URL \url{https://openreview.net/forum?id=p9YlQPF8fE}.

\bibitem[Liu et~al.(2024{\natexlab{b}})Liu, Min, Zettlemoyer, Choi, and Hajishirzi]{liu2024infinigram}
Jiacheng Liu, Sewon Min, Luke Zettlemoyer, Yejin Choi, and Hannaneh Hajishirzi.
\newblock Infini-gram: Scaling unbounded n-gram language models to a trillion tokens.
\newblock In \emph{First Conference on Language Modeling}, 2024{\natexlab{b}}.
\newblock URL \url{https://openreview.net/forum?id=u2vAyMeLMm}.

\bibitem[Wu et~al.(2024)Wu, Geiger, Arora, Huang, Wang, Goodman, Manning, and Potts]{wu-etal-2024-pyvene}
Zhengxuan Wu, Atticus Geiger, Aryaman Arora, Jing Huang, Zheng Wang, Noah Goodman, Christopher Manning, and Christopher Potts.
\newblock pyvene: A library for understanding and improving {P}y{T}orch models via interventions.
\newblock In \emph{Proceedings of the 2024 Conference of the North American Chapter of the Association for Computational Linguistics: Human Language Technologies (Volume 3: System Demonstrations)}, pages 158--165. Association for Computational Linguistics, June 2024.
\newblock URL \url{https://aclanthology.org/2024.naacl-demo.16}.

\bibitem[Fan(1949)]{fan1949theorem}
Ky~Fan.
\newblock On a theorem of weyl concerning eigenvalues of linear transformations i.
\newblock \emph{Proceedings of the National Academy of Sciences}, 35\penalty0 (11):\penalty0 652--655, 1949.

\bibitem[Vyas et~al.(2023)Vyas, Atanasov, Bordelon, Morwani, Sainathan, and Pehlevan]{vyas2023featurelearning}
Nikhil Vyas, Alexander Atanasov, Blake Bordelon, Depen Morwani, Sabarish Sainathan, and Cengiz Pehlevan.
\newblock Feature-learning networks are consistent across widths at realistic scales.
\newblock In \emph{Thirty-seventh Conference on Neural Information Processing Systems}, 2023.
\newblock URL \url{https://openreview.net/forum?id=LTdfYIvbHc}.

\end{thebibliography}

\clearpage
\appendix

\section{Related Work}
\label{app:related}

\paragraph{Multi-Task Learning.} Data distributions neural networks are trained on are often deemed as a mixture of tasks~\cite{raffel2020exploring, xie2023data, xie2024doremi, xie2024foundation, ramesh2022picture, ramesh2025principles, penedo2023refinedweb, penedo2024fineweb, dolma, maini2025beyondweb, sam2026should, goyal2024scaling}.
This motivated works analyzing both the learning dynamics of training toy models on multi-task distributions and defining methods aimed at reducing interference between updates caused by learning a task in the presence of other ones.
For example, the notion of ``catastrophic interference'' has been often characterized in the multi-task learning and continual learning literature~\cite{caruana1997multitask, aljundi2019continual}, where task gradients conflict or are imbalanced in scale, leading to learning of only a subset of tasks instead of the entire mixture.
Such phenomenology can be intuitively~\cite{liu2021conflict, aljundi2019gradient, wu2026imbalanced} and theoretically explained: e.g., Pezeshki et al.~\cite{pezeshki2021gradient} posit the idea of gradient starvation, whereby a model trained on a mixture of tasks that have different prior frequencies is unable to learn the infrequent task due to its gradient getting ``starved'' out, i.e., becoming zero; meanwhile, Evron et al.~\cite{evron2022catastrophic} characterize how tasks' observation frequency induces the forgetting of another learned task in a sequential linear regression setting.
Concurrent to our work, \citet{marek2026forgettinglanguagemodelscapacity} show that forgetting of prior tasks occur when a model has little remaining capacity.
These analyses have also motivated methods to avoid interference and enable learning of multiple tasks: e.g., methods that perform ``surgery'' on model gradients~\cite{yu2020gradient, sener2018multi, chen2018gradnorm} to make two conflicting tasks' gradients to have zero interaction by removing one's projection towards another~\cite{suteu2019regularizing, farajtabar2020orthogonal}; these methods have seen use at scale as well~\cite{chen2026reward}.

It is worth noting that our results are in a similar vein as literature above, but augment prior work by characterizing the effects of scale and showcasing that even extremely rarely observed tasks can eventually be learned if one's model is large enough---empirically, related results corroborating our claim in a vision scenario was also made in the continual learning literature by Ramasesh et al.~\cite{ramasesh2021effect}.
That said, we emphasize that neither do our results imply nor do  we claim (in fact, we say otherwise) that scale alone is the mechanism to enable the learning of a rare task in the presence of other frequent ones.
Indeed, methods discussed above from multi-task / continual learning literature show multiple tasks can be simultaneously learned by a model.

\paragraph{Memorization and Scaling.} The core mechanism posited in our work for how larger models learn rare tasks involves a model retaining some signature of observed data from a small batch of samples.
In extreme scenarios, e.g., when only a few samples are contained in the batch, such a signature cannot possibly correspond to a general, abstract task representation.
Instead, the signature can be thought of as a model (at least partially) ``memorizing'' an observation---once enough observations occur and the memories aggregate, in our simple toy settings, we find the model consolidates the memories into an abstract representation that generalizes well.
In this sense, we emphasize our core proposition suggests memorization is not an undesirable property, but instead a prerequisite to eventual generalization for rare tasks.
This mechanism is highly reminiscent of the observations~\cite{power2022grokking, doshi2024grokking, gopalani2024abrupt} and posited learning dynamics for grokking, where a model transitions from memorizing observations to generalizing to novel inputs~\cite{nanda_progress_2022, murty2023grokking, kumar2023grokking, stander2023grokking, mohamadi2024you, varma2023explaining, morwani2023feature}.
Critically, our language model pretraining results, where we use modular addition, i.e., the prototypical grokking task, and find our posited dynamics hold is suggestive that grokking-like dynamics may in fact occur in practice, especially for rarely observed tasks (closest result to this end is perhaps the syntax acquisition dynamics demonstrated by Chen et al.~\cite{chen2024sudden}).
This argument is in keeping with theoretical works on classification that have argued that memorization is necessary for generalization---for example, to handle label noise \cite{cheng2022memorize}, or to handle rare examples \cite{feldman2020does}.  \citet{brown2021memorization} provide a particularly interesting demonstration that learning rare structures effectively requires memorizing even irrelevant information about the data.
On the other hand, it is worth considering if an opposite mechanism may occur for frequently observed tasks: e.g., if a model sees too many observations of the same task, does it perhaps undergo phenomenology such as overfitting, which is generally associated with generalization to memorization dynamics; if so, does scaling help avoid this dynamic via mechanisms such as double descent~\cite{mei2022generalization, loog2020brief, nakkiran2021deep, nakkiran2019more, nakkiran2020optimal}?
Recent work on transient nature of in-context learning capabilities in toy scenarios~\cite{wurgaft2025context, singh2024needs, singh2023transient, park2024competition} is suggestive such a dynamic may occur, and we thus argue it is worth investigating what the counterpart of our work for learning dynamics of frequent tasks looks like.

Building on the above, we also note memorization and the effects of scaling have been often studied in literature; these results are in line with our claims on reduced interference over model parameters via scaling, enabling models to eventually learn rare tasks.
For example, studying memorization in the sense of verbatim match (e.g., $k$-token string match), works show larger models learn knowledge present in the tail-end of the distribution better~\cite{kandpal2023large}, larger models~\cite{lesci-etal-2024-causal,carlini2023quantifying} and later checkpoints tend to memorize more~\cite{huang2024demystifying}, not just individual data points but also training data order~\cite{kuditipudi2025blackbox,krasheninnikov2025fresh}, and these memories are retained for longer across injection events~\cite{duan2025uncovering,chang2024largelanguagemodelsacquire}. \citet{tirumala2022memorization} show that larger models memorize more, but also can memorize more of the data before they begin to overfit.

\paragraph{Generalization and Scaling.} Improved performance as a function of scaling has defined the spirit of machine learning since scaling laws first started being used for identifying training configurations~\cite{hestness2017deep, rosenfeld2019constructive, kaplan2020scaling, rae2021scaling, hoffmann2022training, henighan2020scaling, hernandez2022scaling, hernandez2021scaling}.
The precise mechanism as to how scaling helps produce better models is unclear, but a few propositions have been made.
For example, in works assessing how power-law scaling as a function of data and parameters emerges, prior work has exploited the argument that natural data statistics are heavy-tailed and follow power-law trends (e.g., Zipf-priors in language~\cite{piantadosi2014zipf} and vision~\cite{hyvarinen2009natural}); correspondingly, scaling enables access to lower-order modes of the data distribution~\cite{bahri2024explaining, maloney2022solvable, lin2024scaling, atanasov2024scaling, ren2025emergence, cagnetta2025learning, cagnetta2025scaling}, and arguments to this end have been verified in recent work by Cagnetta et al.~\cite{cagnetta2026deriving}.
Our toy setup was in fact inspired by these papers, especially Ren et al.~\cite{ren2025emergence} and Maloney et al.~\cite{maloney2022solvable}, but is a substantial simplification since, unlike these prior works, our goal was not to characterize the eventual steady state optima a model arrives at, but instead the dynamics that lead to it.
Works closer to this dynamical motivation are by Bordelon et al.~\cite{bordelon2024dynamical, bordelon2025feature}, Paquette et al.~\cite{paquette2024}, Atanasov et al.~\cite{atanasov2024scaling}, and Everett et al.~\cite{everett2024scaling}, who analyze learning dynamics of toy settings that exhibit power-law scaling curves.
However, since we primarily aimed to posit a concrete mechanism via which larger models may be able to learn tasks smaller models do not, we note the concrete results emphasized and takeaways across these works versus ours are fairly different.
In particular, these papers primarily focus on the interaction between learning dynamics and data statistics to identify different regimes of scaling, i.e., what functional form, e.g., power-law or otherwise, results in the best effective characterization of learning dynamics.
Finally, works by Michaud et al.~\cite{michaud2024quantization, michaud2025understanding} and Nam et al.~\cite{nam2024exactly} are fairly related to our paper: specifically, these works characterize an explicitly multi-task construction to posit a model for how power-law scaling can emerge in neural networks.
While still related to the data statistics argument mentioned above, these works also have an explicit notion of prior frequency and (implicitly) show scaling helps learn tasks that are rarely observed in the training distribution.
Our work makes this claim explicit, but also characterizes how, i.e., a mechanism, via which scaling aids learning of rare tasks.

\paragraph{Lottery Tickets and Scaling.} Another thread of research that partially connects the work listed above on scaling and learning of specific tasks is on the lottery ticket hypothesis~\cite{frankle2018lottery}: a lottery ticket is defined as a subnetwork identified from a larger, initial network that, even at random initialization, shows the ability to perform the task one is training their model for.
Theoretical work on lottery ticket hypothesis has characterized bounds on how much larger a model has to be in order to possess a subnetwork that can, up to some error, approximate the model eventually learned via training~\cite{malach2020proving, pensia2020optimal, edelman2023pareto}.
Especially related here is the work of Edelman et al.~\cite{edelman2023pareto}, who show that via scaling model width (the scaling axis we consider as well), the odds that a subset of representations with non-trivial alignment with true task features exist substantially increases.
Correspondingly, scaling improves sample efficiency of learning tasks that require more features (i.e., are more complex); critically, if one slightly generously interprets the authors' results, they are suggestive that a larger model will be able to learn rare tasks by virtue of already possessing features a smaller model will be unable to learn (due to sparsely observed training signal for such tasks).
While this work partially informed the intuition guiding this paper, we note the eventual results for our setting and verification on large-scale scenarios are more concrete.

\clearpage
\section{Experimental Details}
\label{appx:exp_details}

\subsection{Synthetic Experiment}
\label{app:experimental-details}

In the following, we describe experiment details and metrics relevant to results for the synthetic setup.

\paragraph{Data-generating process.} All synthetic runs use the orthogonal-block instantiation of the mixture-of-regressions setup proposed in Sec.~\ref{sec:synthetic}.
We fix the ambient dimension at $D = 1024$, the number of tasks at $K \in \{16, 32\}$ (almost all figures use $K = 32$), and a per-task block dimension $d_T$ such that $K \cdot d_T \le D$ and the task blocks are mutually orthogonal. 
Concretely, task $k$ occupies coordinates $[k\,d_T,\,(k{+}1)\,d_T)$ of $\mathbb{R}^D$, and its within-block spectrum is the power-law $\sigma_{k,j} = j^{-\alpha_k}$ for $j = 1, \dots, d_T$. 
Unless stated otherwise we use a shared exponent $\alpha_k \equiv \alpha$ across tasks ($\alpha = 1$ in the orthogonal-block experiments, making the within-block decay slow enough that capacity reliably spreads beyond the leading mode of each task). 
The task prior is the power-law $\pi_k \propto k^{-\beta}$, normalized to sum to one over $k = 1, \dots, K$; $\beta = 2$ in most experiments. 
Inputs are sampled fresh each step as $x \sim \mathcal{N}(0, \sigma_{\mathrm{in}}^2 I_D)$ with $\sigma_{\mathrm{in}} = 1$ in all orthogonal-block runs. 
The per-task targets are $y_k = \Lambda_k^{1/2} B_k^\top x$, restricted to the task's block; because each task block has rank $d_T$, the output dimension of the regressor is $d_T$ (and reduces to $1$ when $d_T = 1$, e.g., in the rank-1 specialization of App.~\ref{app:E-staircase}). 

\paragraph{Model.} The student is the linear-bottleneck regressor of Sec.~\ref{sec:synthetic}: a shared encoder $W \in \mathbb{R}^{N \times D}$ that maps the input to an $N$-dimensional hidden, followed by per-task linear decoders $D_k \in \mathbb{R}^{d_T \times N}$ selected by the ground-truth task index supplied in the batch. 
We do not explicitly constrain $W$ to have orthonormal rows: the relevant object for Theorem~\ref{thm:capacity} is the projector $P_W = W^\top (W W^\top)^{-1} W$, which is invariant to the right-multiplicative gauge of $W$ and which gradient flow drives toward the top-$N$ eigenspace of $M = \sum_k \pi_k C_k$ regardless of the parametrization. 
The encoder is initialized such that $W^\top W = I_N$ at step zero, and the per-task decoders are initialized with Kaiming-uniform fan-in / linear gain. 
The decoders are jointly optimized with the encoder rather than analytically closed-formed at each step, since learned decoders will converge to $D_k^* = \Lambda_k^{1/2} B_k^\top U$ at any stationary point, so the joint optimization does not change the encoder fixed point but does match the practical setting in which both ends of the bottleneck are learned simultaneously.

\paragraph{Optimizer.} We use AdamW with default hyperparameters and an inverse-square-root learning-rate schedule. 
Gradients are clipped at maximum norm $1.0$. 
Batches are drawn fresh each step (no fixed dataset, no replay) with batch size $B = 1\,024$ for the phase-diagram and rank-1 sweeps, and $B = 512$ for the matched-frequency retention sweeps; the smaller batch in the retention runs is required so that an injection batch with $m \le B$ rare-task slots can match the long-run frequency $\rho_r = m / (G \cdot B)$ at the $\rho_r \approx 6 \times 10^{-4}$ end of the sweep.

\paragraph{Metrics.} We track three families of metrics, all reported on freshly sampled batches separate from the training stream. 
The first is the \emph{per-task loss}, i.e., the unnormalized population MSE $\ell_k(U) = \mathbb{E}\!\left[\|y_k - D_k U^\top x\|_2^2\right]$, and its normalized counterpart $\ell_k(U) / \ell_{k,\mathrm{baseline}}$ with $\ell_{k,\mathrm{baseline}} = \|a_k\|_2^2 / d_T$ the mean-predictor MSE per task.
The second is the \emph{per-task subspace alignment}, the basis-free quantity $s_k(U) = \mathrm{Tr}(P_U C_k) / \mathrm{Tr}(C_k) = \|P_U a_k\|_2^2 / \|a_k\|_2^2$, computed via the SVD of $W$ so that it is independent of the gauge of the encoder. 
$s_k$ lies between $N/D$ at random initialization and $1$ when the task block is fully captured. 
We also report its random-baseline-corrected normalization $\tilde s_k(U) = (s_k(U) - N/D) / (1 - N/D)$, which equals $0$ at random initialization and $1$ at full capture. 
The third is the \emph{residual common-task signal}: we compute $\delta_F(U) = \sum_{k \in F} \pi_k\,(1 - s_k(U))\,\|a_k\|_2^2$, the residual energy of the frequent block.
The frequent set $F$ is the smallest top-prior set with cumulative mass at least $0.8$; under our power-law prior this yields $|F| = \{6, 3, 2, 2\}$ for $\beta \in \{0.5, 1.0, 1.5, 2.0\}$ respectively. 
Standard evaluation is performed every $1\,000$–$2\,000$ steps on a held-out probe of the same population distribution, with the final checkpoint additionally re-evaluated for end-of-training summary statistics.

\clearpage
\subsection{OLMo Pretraining Pipeline}
\label{appx:olmo_config}

\paragraph{Models.}

\begin{table}[h]
\centering
\caption{Model configurations by size.}
\begin{tabular}{lrrrrr}
\toprule
Model Name & \# Parameters & \# Layers & Hidden Dim & MLP Dim & \# Attn Heads \\
\midrule
4M   & 6,963,200     & 8  & 64   & 512    & 8  \\
20M  & 28,753,920    & 16 & 192  & 1,536  & 8  \\
300M & 371,458,048   & 16 & 1,024 & 8,192  & 16 \\
1B   & 1,279,787,008 & 16 & 2,048 & 16,384 & 16 \\
4B   & 4,707,057,664 & 16 & 4,096 & 32,768 & 32 \\
\bottomrule
\end{tabular}
\label{tab:model-config}
\end{table}

We use the OLMo model architecture~\cite{groeneveld2023olmo}. For 4M to 1B models, we follow the model configuration and naming convention of \citet{magnussondatadecide2025}. We additionally include a 4B model to further evaluate width scaling, as shown in Table~\ref{tab:model-config}.

\paragraph{Training hyperparameters.} We use the same batch size of $1024$, window size of $4096$ for \tcmp\ and $1024$ for \tadd, and a learning rate schedule with an initial learning rate of $3\times10^{-4}$ and cosine with warmup schedule for all models. For the retention window ablation experiment, we use a smaller window size of $512$ to reduce the training cost. For a full list of hyperparameters, refer to the \texttt{OLMo-7B-0724} configuration.\footnote{\url{https://github.com/allenai/OLMo/blob/main/configs/official-0724/OLMo-7B-0724.yaml}}

\paragraph{Training pipeline.} We use the OLMo code base.\footnote{\url{https://github.com/allenai/OLMo}} Our usage is in line with its Apache-2.0 license.\footnote{\url{https://github.com/allenai/OLMo?tab=Apache-2.0-1-ov-file}}

\paragraph{Compute resources.} All models are trained on a cluster of NVIDIA H200 GPUs.

\subsection{Pre-training and Injected Task Data}
\label{appx:olmo_data}

\paragraph{Pre-training data.} We use Dolma v1.7 as the pre-training corpus~\cite{dolma}. Specifically, we use the 210B tokens corresponding to the first 50K batches that \texttt{OLMo-7B-0424} and \texttt{OLMo-7B-0724} are trained on, in the exact same order.

\paragraph{Reference tasks.} To ensure the injected task frequency is comparable to the frequency of tasks learned in pre-training, we sample two reference tasks $R_{\text{cmp}}$ and $R_{\text{add}}$ from pre-training that involve similar high-level functions. $R_{\text{cmp}}$ predicts a number larger than \texttt{x} in the prompt ``it has increased from \texttt{\{x\}} to''. $R_{\text{add}}$ predicts the sum of two numbers smaller than 100 with the prompt ``\texttt{\{x\}} + \texttt{\{y\}} =''. We estimate the lower bound of their frequency in pre-training data using infini-gram~\cite{liu2024infinigram} and observe models' next token prediction loss on the task, which corresponds to the two dashed lines in Fig.~\ref{fig:olmo_loss} panel (a).

\paragraph{Injected tasks.} We elaborate the task label here. Let $val(\cdot): \mathcal{S} \mapsto [0, 99]$ be a bijective mapping that assigns an integer value between 0 and 99 to each token. For \tcmp, \texttt{LABEL} is one of two tokens randomly chosen from the vocab indicating whether $val(\texttt{TOK1})<val(\texttt{TOK2})$. For \tadd, \texttt{LABEL} is the token in $\mathcal{S}$ whose value equals $(val(\texttt{TOK1}) + val(\texttt{TOK2}))$ $mod$ $100$. Below are a few instances from the comparison task: \texttt{ address analyze pony},
 \texttt{ resort zebrafish pony},
 \texttt{ cavities misconduct provisional}, where \texttt{pony} and \texttt{provisional} are the two label tokens that represent \texttt{True} and \texttt{False}.

\newpage
\subsection{Localizing and Measuring Task Features in Sec.~\ref{sec:olmo_repr}}
\label{appx:localize_task_features}

\paragraph{The comparison task \tcmp.} We first use distributed alignment search (DAS) to verify that the model's prediction is causally dependent on a global token order feature, which is encoded in a 1-D subspace in the residual stream of the first few layers.

The task has a simple high-level causal model, namely $X \rightarrow O \leftarrow Y$, where $X, Y$ are the two inputs and $O$ is the binary output. We consider the following intervention on the input variable $X$ (or $Y$): Let a base example be $x_b, y_b, o_b$ and a source example be $x_s, y_s, o_s$, an interchange intervention on $X$ that sets the value of $x_b$ to $x_s$ should lead to a counterfactual label that corresponds to $x_s < y_b$.

In the neural model, we search for a low-dimensional subspace in the residual stream that plays the same causal role as the input variable $X$ by training on $1K$ counterfactual data pairs defined above. We use DAS to search across all layers above the input token position. We are able to find a 1-D subspace in the residual stream of the first layer that has an interchange intervention success rate of 96\%. This proves that the model not only encodes the global token order in a low-dimensional space but actually uses this feature for prediction on the task \tcmp. This allows us to use the global order feature to measure to what extent a model has learned the abstract task structure.

We use the DAS implementation from \texttt{pyvene}~\cite{wu-etal-2024-pyvene}.

\paragraph{The modular addition task \tadd.} As prior work studying modular addition has identified that grokked models use Fourier modes for addition~\cite{nanda_progress_2022}, we conduct Fourier analysis on residual stream to measure the presence of Fourier modes.

For modulus P, define a real discrete Fourier transform basis on $\mathbb{R}^P$ as follows:
\begin{align*}
\phi_k^{\cos}(n) = \frac{\cos(2\pi k n / P)}{||\cos(2\pi k \cdot / P)||}, \quad
\phi_k^{\sin}(n) = \frac{\sin(2\pi k n / P)}{||\sin(2\pi k \cdot / P)||}, \quad
k = 1,\dots,\left\lfloor \frac{P}{2} \right\rfloor.    
\end{align*}

At layer $l$, collect residual stream vectors $h^{l} \in \mathbb{R}^d$ grouped by output \texttt{c = (a+b) $\bmod$ P} and compute the mean representation as:
\begin{align*}
v_c^{(l)} = \mathbb{E}\bigl[h^{(l)} \mid c=(a+b)\bmod P\bigr] \in \mathbb{R}^d,
\quad c \in \{0,\dots,P-1\}.
\end{align*}
This yields a matrix $V^{(l)} \in \mathbb{R}^{P \times d}$.

After row-centering $V^{(l)}$, the fraction of variance captured by frequency $k$ is
\begin{align*}
  P_k^{(l)} =
\frac{\sum_{j=1}^d \left( \langle \phi_k^{\cos}, V^{(l)}_{:,j} \rangle^2
+ \langle \phi_k^{\sin}, V^{(l)}_{:,j} \rangle^2 \right)}
  {\sum_{j=1}^d ||V^{(l)}_{:,j} - \bar V^{(l)}_{:,j}||^2},
  \quad \sum_k P_k^{(l)} = 1
\end{align*}

We consider a null-baseline for detecting Fourier modes in representations. Under uniform variance allocation across frequencies,
\begin{align*}
 P^{\mathrm{null}} = \frac{2}{P-1}    
\end{align*}

Hence, a frequency $k$ at layer $l$ is identified as a Fourier mode if
\begin{align*}
    P_k^{(l)} > \theta P^{\mathrm{null}}
\end{align*}
We choose $\theta=2$ in our experiment, as we do not observe significant difference in grokking behavior for models that represent Fourier modes with stronger signals, e.g., $\theta=3$.

Finally, we define the total number of detected modes as follows, where $L$ is the total number of layers.
\begin{align*}
N_{\mathrm{features}} 
= \sum_{l=0}^{L}
\left| \left\{ k : P_k^{(l)} > 2 P^{\mathrm{null}} \right\} \right|
\end{align*}
This corresponds to the y-axis in Fig.~\ref{fig:task_feature_learning} (b) right panel.

\newpage

\begin{figure}[!t]
    \centering
        \centering
        \includegraphics[
            width=\linewidth,
            trim={0 0 0 0}, %
            clip
        ]{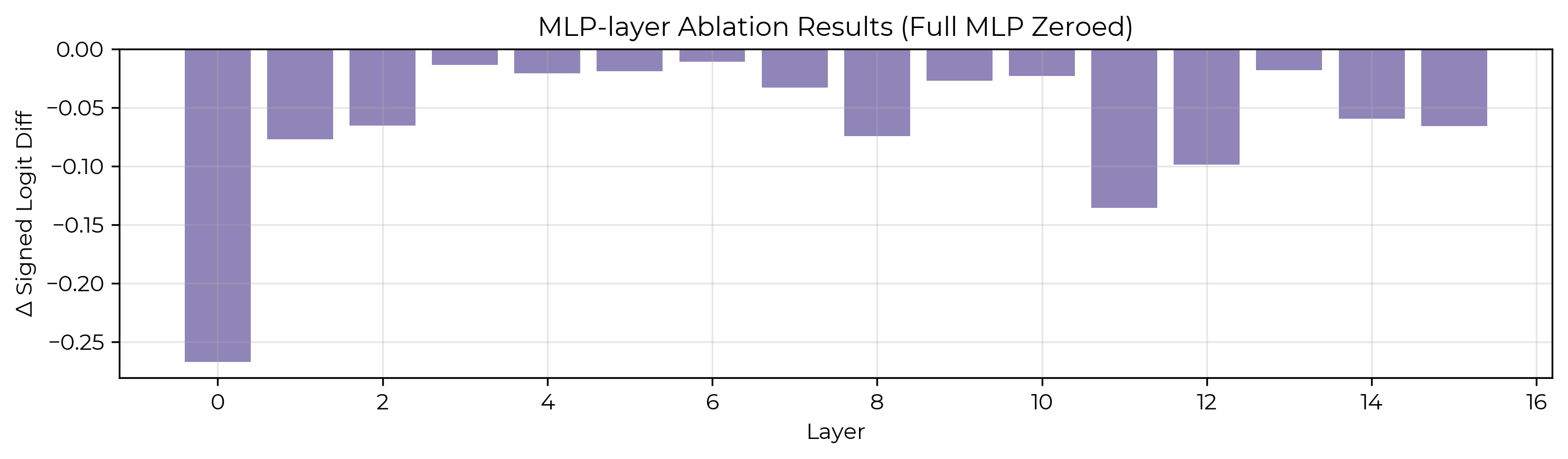}
    \caption{The first MLP layer has the strongest causal effects on the model's logits prediction.}
    \label{fig:olmo_mlp_ablation}
\end{figure}

\subsection{Localizing Task Neurons in Sec.~\ref{sec:olmo:gradients}}
\label{appx:localize_task_neurons} 

We conduct null interventions on MLP layers to identify layers that have the largest causal effects on the model output. The results are shown in Fig.~\ref{fig:olmo_mlp_ablation}. This aligns with our observation from the DAS localization experiment that first layer is the earliest layer where the global token order is causally encoded. The localization result is consistent across the three models used in Sec.~\ref{sec:olmo:gradients}.

\clearpage
\section{Proofs}
\label{app:proofs}

\subsection{Proof of Theorem~\ref{thm:capacity}}
\label{app:proof-capacity}

For fixed $U$, the task-$k$ population loss is
\begin{align}
    \ell_k(U,D_k) &= \E\big[\|\Lambda_k^{1/2}B_k^\top x - D_kU^\top x\|_2^2\big] \\
    &= \|\Lambda_k^{1/2}B_k^\top - D_kU^\top\|_F^2,
\end{align}
where the second identity uses $x\sim \mathcal{N}(0,I)$ and the standard relation $\E\|Ax\|_2^2 = \|A\|_F^2$.
This is a linear least-squares problem in $D_k$, so the minimizer is
\begin{equation}
    D_k^* = \Lambda_k^{1/2}B_k^\top U.
\end{equation}
Substituting back gives
\begin{align}
    \ell_k(U)
    &= \big\|\Lambda_k^{1/2}B_k^\top(I-P_U)\big\|_F^2 \\
    &= \tr\!\Big(\Lambda_k^{1/2}B_k^\top(I-P_U)B_k\Lambda_k^{1/2}\Big) \\
    &= \tr\!\big((I-P_U)C_k\big).
\end{align}
Summing with weights $\pi_k$ yields
\begin{equation}
    L_N(U) = \sum_{k=1}^K \pi_k\ell_k(U) = \tr(M) - \tr(U^\top M U), 
    \qquad
    M := \sum_{k=1}^K \pi_k C_k.
\end{equation}
Because $M$ is symmetric positive semidefinite with finite trace, minimizing $L_N(U)$ is equivalent to maximizing $\tr(U^\top M U)$ over all orthonormal $U$.
By Ky Fan's maximum principle~\citep{fan1949theorem},
\begin{equation}
    \max_{U^\top U = I_N} \tr(U^\top M U) = \sum_{i=1}^N \mu_i,
\end{equation}
where $\mu_1\ge \mu_2\ge \cdots$ are the eigenvalues of $M$.
Therefore any minimizer spans the top-$N$ eigenspace of $M$ and the optimal loss is
\begin{equation}
    L_N^* = \tr(M) - \sum_{i=1}^N \mu_i = \sum_{i>N} \mu_i.
\end{equation}
For our generative process, we have
\begin{equation}
    M = \sum_{k,j} \pi_k\lambda_{k,j}\, b_{k,j}b_{k,j}^\top,
\end{equation}
so the vectors $b_{k,j}$ are eigenvectors of $M$ with eigenvalues $u_{k,j}=\pi_k\lambda_{k,j}$.
Thus the width-$N$ optimum keeps the $N$ largest utilities.
If task $k$ contributes $n_k(N)$ retained coordinates, then its residual loss is
\begin{equation}
    \ell_k^*(N) = \sum_{j>n_k(N)} \lambda_{k,j}.
\end{equation}

\subsection{Proof of Theorem~\ref{thm:interference}}
\label{app:proof-interference}

Write $G_{\Fre}(U) = 2(I-P_U)M_{\Fre}^{1/2}M_{\Fre}^{1/2}U.$
Using $\|AB\|_F\le \|A\|_F\|B\|_{\mathrm{op}}$, where $\|.\|_{\mathrm{op}}$ denotes the operator norm, gives
\begin{align}
    \|G_{\Fre}(U)\|_F
    &\le 2\,\|(I-P_U)M_{\Fre}^{1/2}\|_F\,\|M_{\Fre}^{1/2}U\|_{\mathrm{op}} \\
    &\le 2\sqrt{\tr\!\big((I-P_U)M_{\Fre}\big)}\,\sqrt{\lambda_1(M_{\Fre})} \\
    &= 2\sqrt{\lambda_1(M_{\Fre})\,\delta_{\Fre}(U)}.
\end{align}

\subsection{Proof of Proposition~\ref{prop:invasion}}
\label{app:proof-invasion}

Let $U$ denote the common-task width-$N$ solution and let $u_i$ be one occupied common eigenvector with eigenvalue $\mu_i^{\Fre}$.
Replace that vector by
\begin{equation}
    v_i(\theta) = \cos\theta\,u_i + \sin\theta\,b_r,
\end{equation}
while keeping the remaining $N-1$ directions fixed.
Because $u_i$ is an eigenvector of $M_{\Fre}$ and $b_r$ is orthogonal to the common block, the contribution of this one direction to the objective $\tr(P_UM)$ is
\begin{equation}
    \langle v_i(\theta), Mv_i(\theta)\rangle
    = \mu_i^{\Fre}\cos^2\theta + \pi_r\lambda_r\sin^2\theta.
\end{equation}
Subtracting the value at $\theta=0$ gives
\begin{equation}
    \Delta \tr(P_UM) = (\pi_r\lambda_r - \mu_i^{\Fre})\sin^2\theta.
\end{equation}
Since $L = \tr(M)-\tr(P_UM)$, the loss change is
\begin{equation}
    \Delta L = (\mu_i^{\Fre} - \pi_r\lambda_r)\sin^2\theta.
\end{equation}
Hence perturbations toward $b_r$ decrease the loss if and only if $\pi_r\lambda_r > \mu_i^{\Fre}$.
The rare feature invades first through the weakest occupied common direction, which has curvature $\mu_N^{\Fre}$, proving the stated threshold.

\subsection{Microscopic competition in a one-neuron, two-task model}
\label{app:C-neuron-specialization}

\begin{figure}[h]
  \centering
  \includegraphics[width=\linewidth]{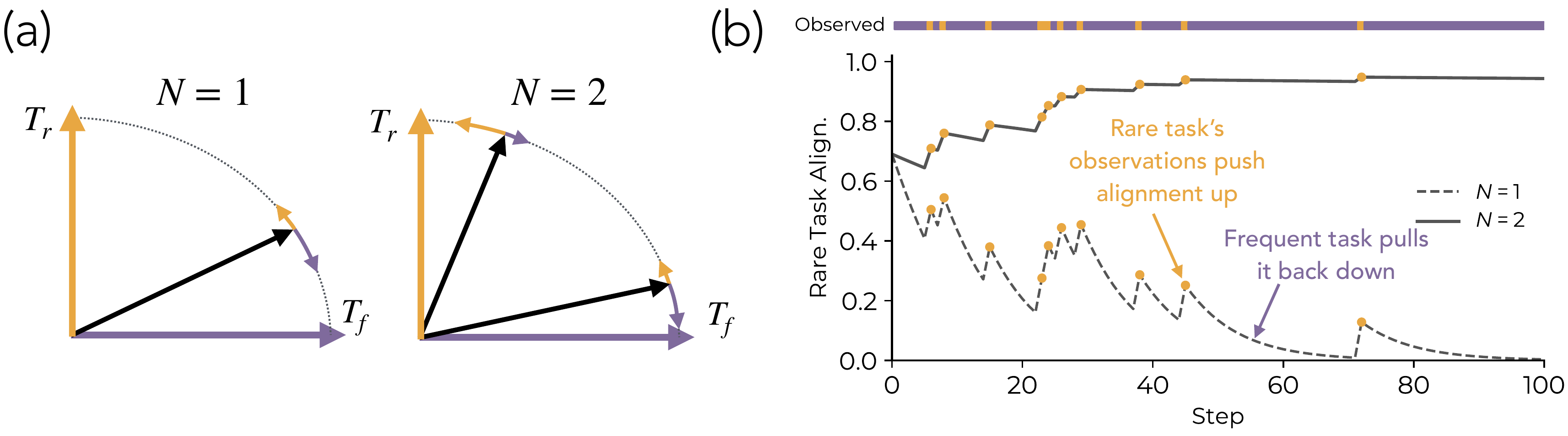}
  \caption{\textbf{Competition Dynamics over Neurons.} Rare task alignment over training for a softmax-gated model of 1 vs.\ 2 neurons. (a) Two orthogonal task directions $T_f$ (frequent, sampled with probability 0.9) and $T_r$ (rare, probability 0.1) compete for neurons. (b) With a single neuron, the frequent task dominates; with two neurons, one neuron specializes to each task, allowing rare task alignment to reach and sustain values near 1.}
  \label{fig:single_neuron}
\end{figure}

\begin{example}[One neuron, two orthogonal tasks]
\label{ex:one-neuron}
Let $a,b\in\R^d$ be orthonormal and consider rank-one tasks with covariances $C_a=aa^\top$ and $C_b=bb^\top$.
A width-1 encoder is a unit vector
\begin{equation}
    u = \cos\theta\,a + \sin\theta\,b.
\end{equation}
The task losses are
\begin{equation}
    \ell_a(u)=\sin^2\theta,
    \qquad
    \ell_b(u)=\cos^2\theta.
\end{equation}
A gradient step on task $a$ obeys $\theta^+=\theta-\eta\sin(2\theta)+O(\eta^2)$, while a step on task $b$ obeys $\theta^+=\theta+\eta\sin(2\theta)+O(\eta^2)$.
If task $a$ appears with probability $p$ and task $b$ with probability $q<p$, then
\begin{equation}
    \E[\Delta\theta\mid \theta] = \eta(q-p)\sin(2\theta)+O(\eta^2),
\end{equation}
which drives the neuron toward the common task.
Near $\theta=0$, if a rare-task update is followed by $G$ common-task updates, then
\begin{equation}
    \theta_G \approx (1-2\eta)^G\theta_0 \approx e^{-2\eta G}\theta_0.
\end{equation}
Thus rare-task alignment decays exponentially across the gap between rare observations.
\end{example}

\begin{proof}
The loss identities follow from $u^\top aa^\top u = \cos^2\theta$ and $u^\top bb^\top u = \sin^2\theta$.
Differentiating yields
\begin{equation}
    \frac{d}{d\theta}\sin^2\theta = \sin(2\theta),
    \qquad
    \frac{d}{d\theta}\cos^2\theta = -\sin(2\theta),
\end{equation}
which gives the stated updates under gradient descent.
Taking the expectation under the task mixture yields the drift formula.
Linearizing $\sin(2\theta)\approx 2\theta$ near zero gives the exponential decay estimate.
\end{proof}

The dynamics posited above are also exemplified in Fig.~\ref{fig:single_neuron}.

\clearpage
\section{Further Experimental Results: Complexity Sweeps}
\label{app:complexity_sweep}

In the main paper, we kept ``complexity'', i.e., the number of directions used for defining the target variable constant across tasks; specifically, tasks in the main paper require $5$ directions to cover $90$\% of the energy in the task spectrum (i.e., solving for $r$ in ${\arg\min}_{r} \frac{\sum_{j=1}^{j={r}} \lambda_{k,j}}{\sum_j \lambda_{k,j}} > 0.90$ gives $r=5$).
In this section, we vary this property by changing the power-law coefficient underlying the task spectrum, i.e., since $\lambda_{k,j} \propto j^{-\alpha_k}$, we vary the range of $\alpha_k$ across tasks.
We define ranges of $[\alpha_{\text{min}}, \alpha_{\text{max}}]$, split the range uniformly into $K$ values, and assign the $k^{\text{th}}$ value to $\alpha_k$.
The most frequent task is assigned the value $\alpha_{\text{max}}$, giving it the fastest decaying spectrum and hence making it the simplest task, while the rarest task is assigned the value $\alpha_{\text{min}}$, giving it the slower decaying spectrum and making it most complex.
In particular, we choose ranges (see Fig.~\ref{fig:complexity_spectra}) such that the task complexity varies between $[4,7]$ and $[2, 12]$ across $K=32$ tasks.

\begin{figure}[h]
  \centering
  \includegraphics[width=\linewidth]{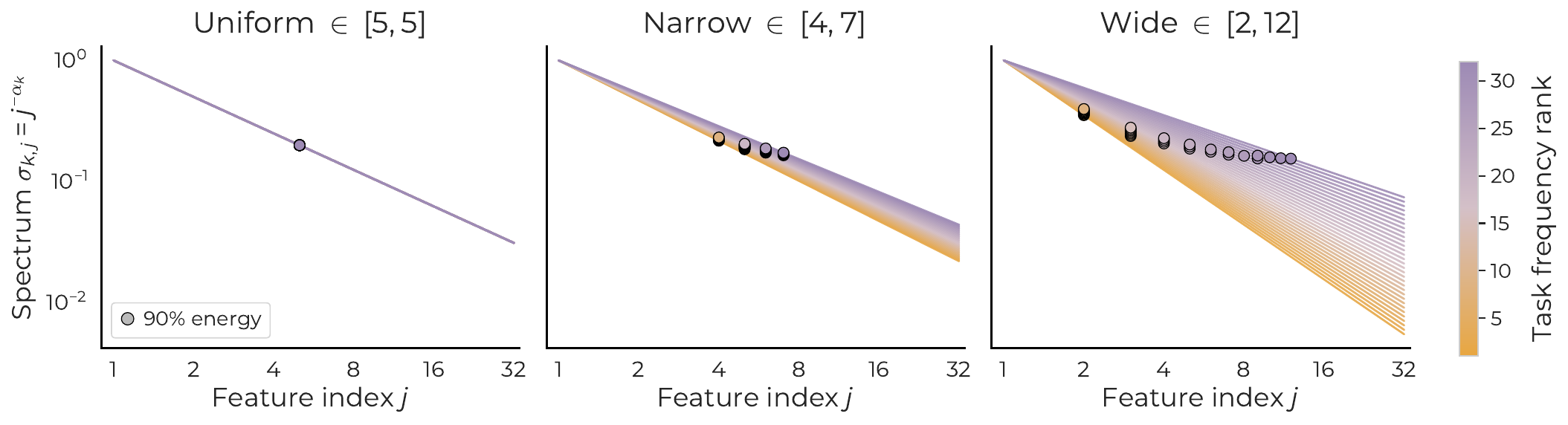}
  \caption{\textbf{Task Spectra.} We use power-law task spectra to vary the complexity of a task in our experiments, i.e., the $j^\text{th}$ feature contributes signal proportional to $j^{-\alpha_k}$ for the $k^{\text{th}}$ task. While the main paper studies the setting with uniform values for $\alpha_k$, hence making frequency the core knob for varying utility, we now vary complexity by splitting a range of $\alpha$ values; this results in task spectra such that the number of directions to cover 90\% of task signal now takes $4$--$7$ directions for the ``narrow'' range scenario, while $2$--$12$ directions for the wider range scenario.}
  \label{fig:complexity_spectra}
\end{figure}

\clearpage
\subsection{Feature Utility Predicts Learning Order}
We first reproduce Fig.~\ref{fig:phases}.
We split the figure into two parts, showing the critical width boundary as a function of task frequency in heatmaps in Fig.~\ref{fig:complexity_phases} and the per-task loss predictability based on feature utilities in Fig.~\ref{fig:complexity_utility}; for reference, we include our baseline results from the main paper, where the task spectra were uniform.

\begin{figure}[h]
  \centering
  \includegraphics[width=0.9\linewidth]{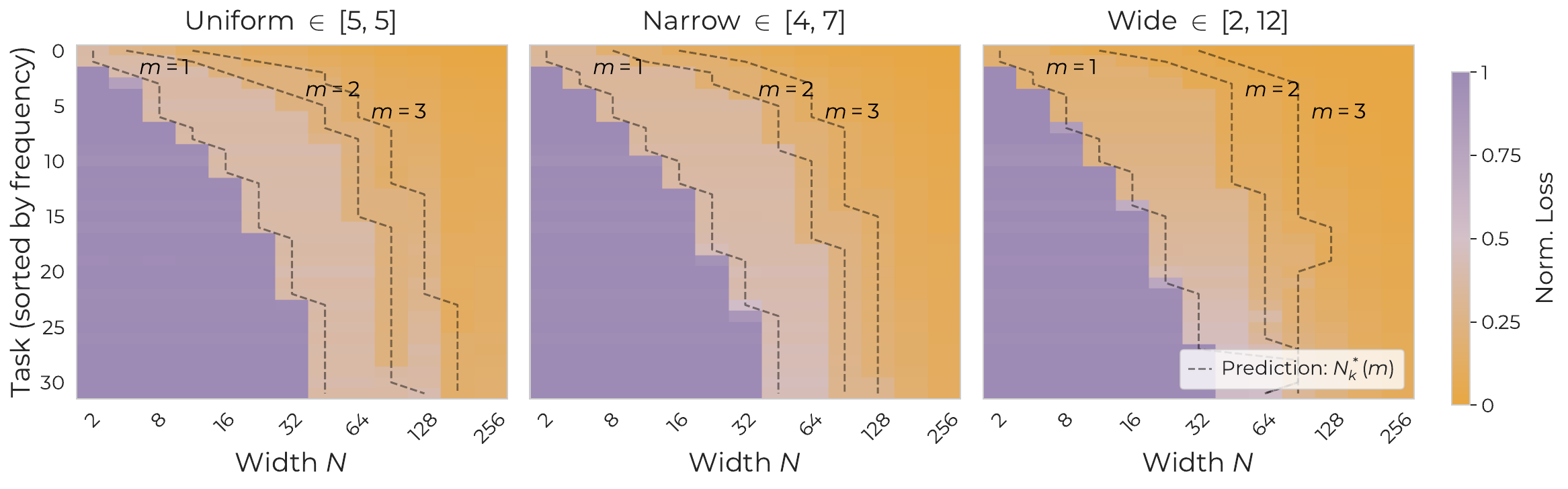}
  \caption{\textbf{Learning Phases Under Varying Complexity.} Reproduction of Fig.~\ref{fig:phases}a under varying task spectra. We see increase in the complexity gap leads to higher emphasis on the top two modes' learning, since under a power-law spectrum decay, the eigenvalue associated with larger modes will be small. More critically, learning order is now not monotonically predicted by frequency alone: this is most easily visible in the results for wide complexity range scenario, where we see the ``most complex'' task's third mode is in fact high enough utility to get learned before more frequent task's higher order modes, resulting in a non-monotonic boundary.}
  \label{fig:complexity_phases}
\end{figure}

\begin{figure}[h]
  \centering
  \includegraphics[width=0.9\linewidth]{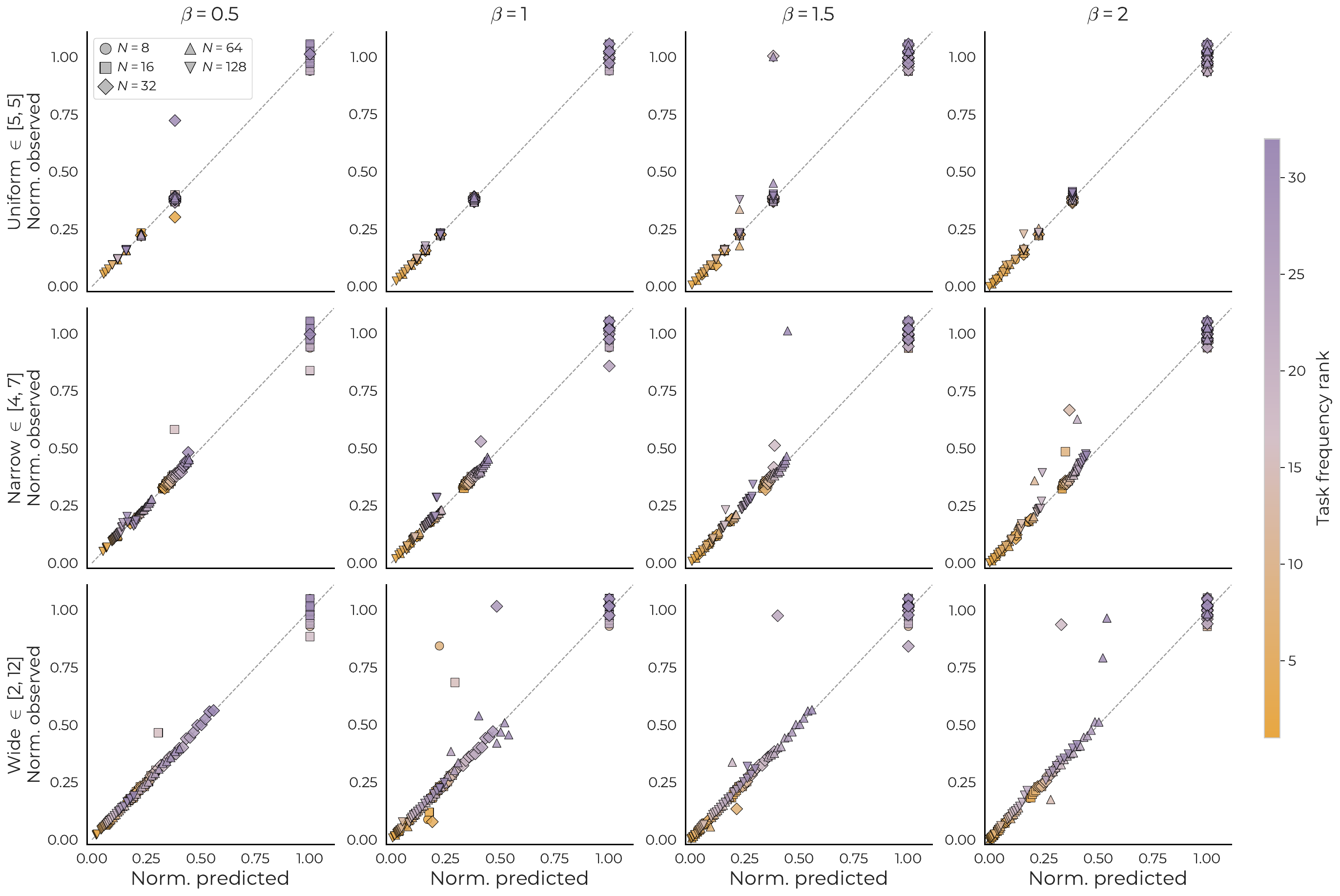}
  \caption{\textbf{Feature Utilities Continue to Predict Learning.} Reproduction of Fig.~\ref{fig:phases}b under varying task spectra. While Fig.~\ref{fig:complexity_phases} shows frequency, by itself, is insufficient to predict learning of a task, the current plot shows the empirically observed loss and the loss derived out of the assumption that $N$ neurons will learn top $N$ utility features continues to align well. This confirms the learning dynamic in the non-uniform complexity scenarios requires accounting for both frequency and complexity: higher-frequency tasks \textit{may be learned after} a lower-frequency task if the complexity of the former is more than the latter.}
  \label{fig:complexity_utility}
\end{figure}

\clearpage
\subsection{Competition Dynamics Disallow Learning of the Rarest and the Most Complex Task}
Our results above showed that learning trends under varying task complexity are modulated by both task frequency and complexity, as predicted by our account in Sec.~\ref{sec:synthetic}.
We now show the competition dynamics picture posited in that section continues to follow in these settings as well.
In particular, we plot the learning of the top-3 most frequent tasks (measured by normalized signal; see Sec.~\ref{sec:synthetic} for details) and the rare-most task as a function of residual, i.e., signal remaining to be explained in the frequent tasks.
As shown in Fig.~\ref{fig:complexity_residual}, the critical width predicted to be necessary for learning of the rarest task, by rendering the residual sufficiently small for most frequent tasks, continues to hold true in this setup as well.

\begin{figure}[h]
  \centering
  \includegraphics[width=\linewidth]{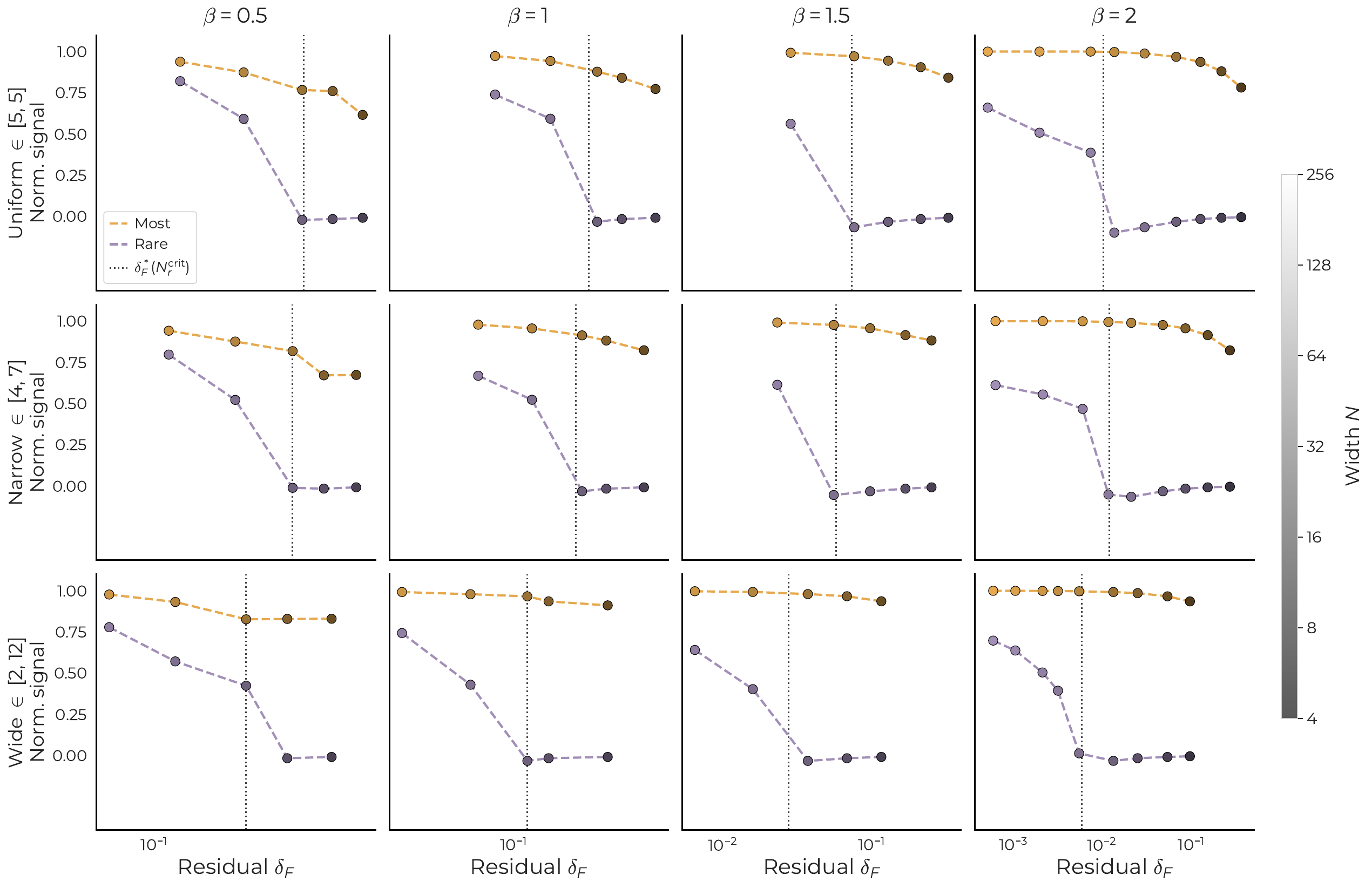}
  \caption{\textbf{Complexity Residual.} We plot the amount of signal encoded in model representation for most frequent and rarest tasks as a function of width $N$ and remaining residual $\delta_\Fre$. In line with our predictions, we see larger models perfectly capture tasks of all frequencies, while smaller models do not. Meanwhile, even for the largest models, when the residual signal remaining to explain for frequent tasks is high, rarer tasks struggle to be learned.}
  \label{fig:complexity_residual}
\end{figure}

\subsection{Reduced Interference Aids Learning of the Rarest and the Most Complex Task}

We now validate our argument for how data-centric bottlenecks, i.e., the low-frequency and high-complexity nature of a task, is circumvented by a larger model: by virtue of having more parameters, a larger model witnesses reduced interference over per-task gradients.
To this end, we redo the batch-injection experiments from Fig.~\ref{fig:competition} and plot the retention dynamics for the lowest aggregate utility task across settings.
Results are shown in Fig.~\ref{fig:complexity_retention}.
We see similar results as before: larger models show better retention of observed signal from a task, allowing them to bootstrap on these past observations and eventually learn the task.
Meanwhile, a medium width model is able to do so only when the task is observed sufficiently frequently, i.e., the gap is low.
Comparing with the case when the width is too low, we see the model never learns the task and the retention dynamics concretely show why: the model is unable to retain signal for the observed task for long enough.

\begin{figure}[h]
  \centering
  \includegraphics[width=\linewidth]{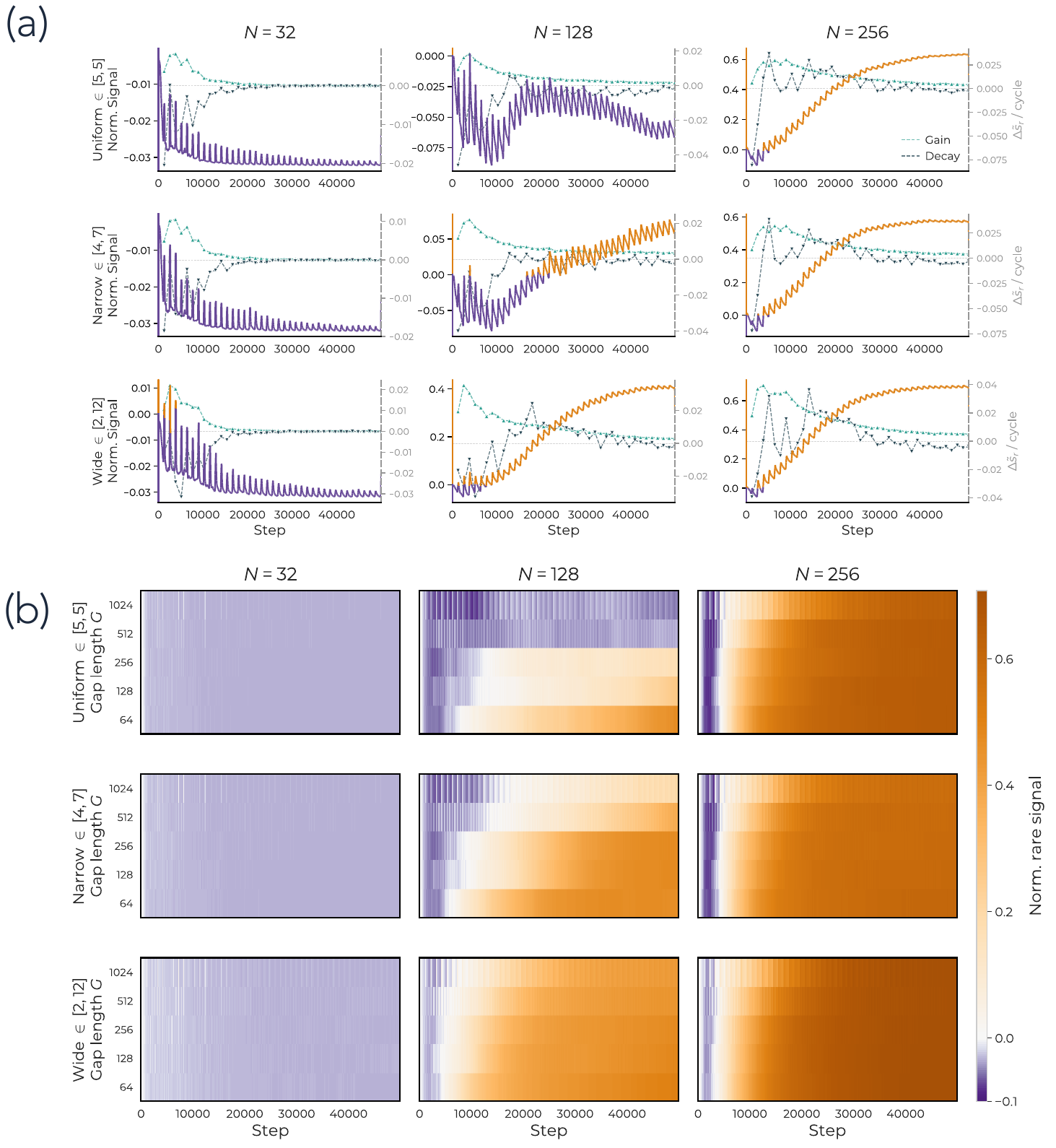}
  \caption{\textbf{Complexity Retention Phases.} We isolate retention by training with a matched-frequency injection protocol: the lowest-total utility task is withheld for $G$ steps and then reintroduced in a batch such that its overall frequency is consistent across settings. (a) Training dynamics for $G=1280$. We see small models briefly encode the rare task (Norm.\ signal $\tilde{s_r}$: left-y axis) after each injection; specifically, $\Delta \tilde{s_r}$ increases at point of injection, as shown by green dotted line (`gain'). However, as frequent-task updates resume, this signal is quickly lost (`decay': gray dotted line). Meanwhile, larger models retain more of the rare-task signal between injections and accumulate it over training. (b) Across injection gaps G and widths N, rare-task signal decays rapidly in narrow models but remains stable in wider models, while frequent-task signal is largely unaffected. These results support the reduced-interference mechanism: scaling provides enough representational capacity that updates from frequent tasks no longer overwrite rare-task features before the next rare observation arrives.}
  \label{fig:complexity_retention}
\end{figure}

\clearpage
\section{Further Experimental Results: Frequency Sweeps}
\label{app:frequency_sweeps}

\subsection{Features and Tasks are Learned in Order of Utility}

\subsubsection{Extending Phase Diagram}
\label{app:E-phase-diagram-betas}

We sweep the power-law exponent $\beta$ defining the task prior by varying it over $\{0.5, 1.0, 1.5, 2.0\}$.
Unlike Fig.~\ref{fig:phases}, we only sweep $5$ values of widths---specifically,
$N \in \{8, 16, 32, 64, 128\}$. 
Correspondingly, the staircase is rendered at lower resolution. Results are shown in Fig.~\ref{fig:E-phase-diagram-betas}.

\begin{figure}[h]
    \centering
    \includegraphics[width=\linewidth]{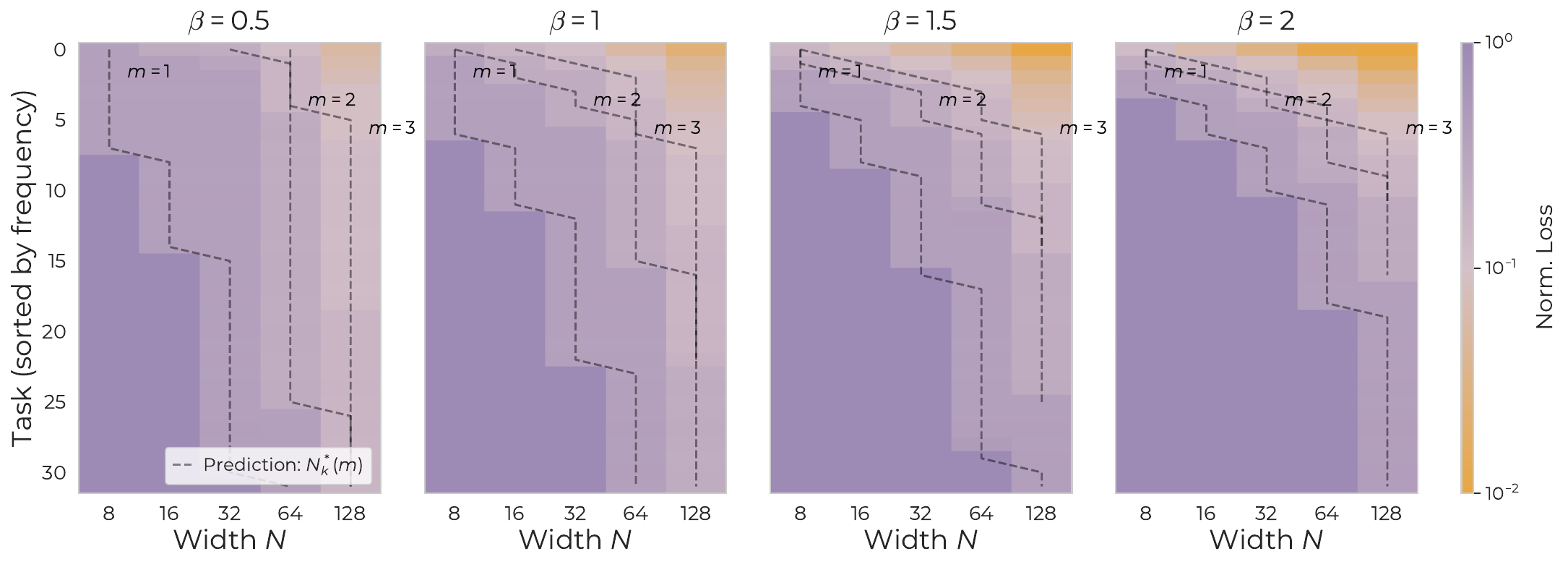}
    \caption{\textbf{Feature Utility Predicts Order of Learning.} 
    We extend results from Fig.~\ref{fig:phases} by analyzing values $\beta \in \{0.5, 1.0, 1.5, 2.0\}$.
    Each panel reports the normalized per-task loss, i.e., $\ell_k(N) / \ell_{k,\mathrm{baseline}}$ as a function of width $N$. 
    Tasks are sorted top-to-bottom by descending prior frequency so the most-frequent task occupies the top row.
    Dashed staircases are the theoretical thresholds $N_k^*(m)$ for $m = 1, 2, 3$ computed from the per-direction utility ordering of Theorem~\ref{thm:capacity}. 
    The empirical learned region (orange) tracks the $m = 1$ staircase across all four prior skews; deeper-orange cells in the steeper-prior panels ($\beta = 1.5, 2$) reflect the model spending its width budget on additional directions of the leading tasks rather than on rarer tasks, in agreement with the account posited in the main paper for how scaling interacts with data properties.
    }
    \label{fig:E-phase-diagram-betas}
\end{figure}

\subsubsection{Simplified Case: Rank-1 Tasks}
\label{app:E-staircase}

\begin{figure}[h]
  \centering
  \includegraphics[width=\linewidth]{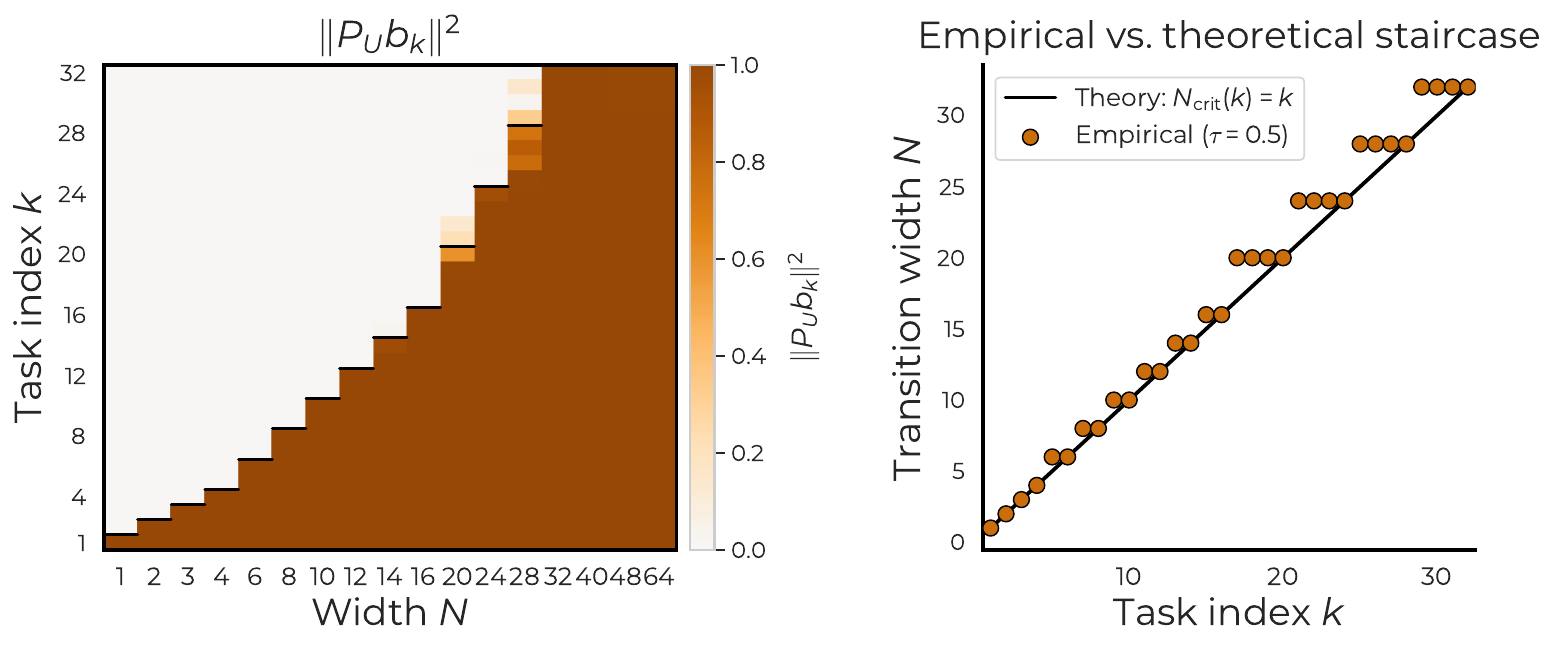}
  \caption{\textbf{Rank-1 Verification of Utility Predicting Learning Order.} \textbf{Left:} Per-task subspace alignment \(\|P_U b_k\|^2\) at the end of training as a function of width \(N\) and task index \(k\). 
  By our account, we expect tasks \(1\dots N\) to be retained, while tasks \(N+1\dots K\) are not  retained. 
  The black step segments in the heatmap mark the predicted retention horizon \(k = N\) per width. 
  Results align well with our expectations.
  \textbf{Right:} Empirical transition width \(N_{\mathrm{emp}}(k)\) (markers) sits on the theoretical staircase \(N_{\mathrm{crit}}(k) = k\) (line) within the resolution of the width sweep. 
  Plateaus at \(k = 5, 7, 9, \dots\) reflect the gap in sampled widths between adjacent grid points and are not deviations from theory.
  }
  \label{fig:E-Ncrit-staircase}
\end{figure}

Theorem~\ref{thm:capacity} predicts that a width-\(N\) minimizer retains the \(N\) task-features with largest utility $u_{kj} = \pi_k\,\lambda_{kj}$.
We obtain a sharp quantitative test of this claim by collapsing the orthogonal-block setup to its rank-1 specialization: setting \(d_T = 1\) and \(\alpha_k = 1\) makes every task rank-1 with \(\lambda_k = 1\), so the utility ordering reduces to the prior ordering, and the predicted critical width for task \(k\) becomes
\begin{equation}
  N_{\mathrm{crit}}(k)
  \;=\; \#\{j \neq k : \pi_j\lambda_j > \pi_k\lambda_k\}
  \;=\; k,
  \label{eq:E-staircase-prediction}
\end{equation}
i.e., a perfectly linear staircase in task index.

\paragraph{Setup.} We train the linear-bottleneck student described in Sec.~\ref{sec:synthetic} on a mixture of \(K = 32\) rank-1 orthogonal tasks, ambient dimension $D = 1024$, and a power-law prior with exponent $\beta = 2$. 
We sweep the encoder width \(N \in \{1, 2, 3, 4, 6, 8, 10, 12, 14, 16, 20, 24, 28, 32, 40, 48, 64\}\) and read out the per-task subspace alignment \(s_k(U) = \|P_U b_k\|^2\) at the end of training (rank-1 specialization of the per-task signal of Sec.~\ref{sec:synthetic}, so \(s_k(U)\) lies between \(N/D\) at random initialization and \(1\) when \(b_k\) is fully captured by the encoder subspace).

\paragraph{Result.} See Figure~\ref{fig:E-Ncrit-staircase}. We overlay the empirical transition width \(N_{\mathrm{emp}}(k) = \min\{N \in \text{grid} : s_k(U) > 0.5\}\) on the theoretical staircase~\eqref{eq:E-staircase-prediction}, finding almost perfect alignment (minimal disparities are an artifact of the sampled width grid).

\subsection{Residual Controls Learning}
\label{app:E-residual-scatter-betas}

\begin{figure}[h]
  \centering
  \includegraphics[width=\linewidth]{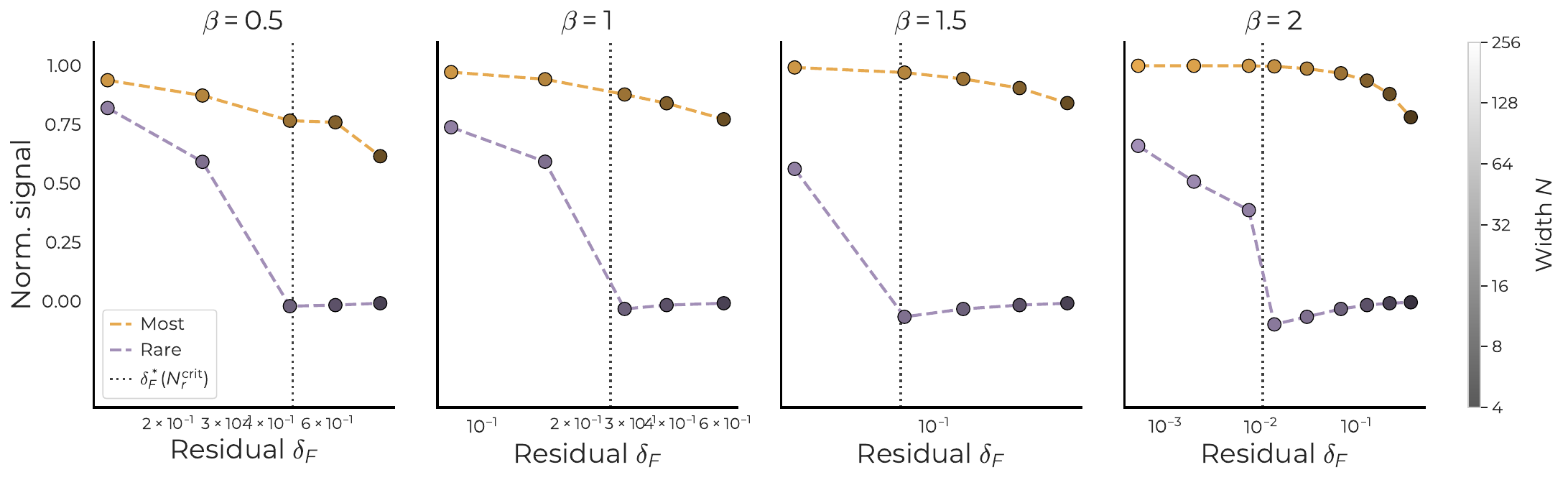}
  \caption{\textbf{Residual Controls Rare-Task Learning}. We vary $\beta \in \{0.5, 1.0, 1.5, 2.0\}$. 
  Each panel reports the normalized rare-task and most-frequent-task signal as a function of the frequent-task residual \(\delta_F\), with width \(N\) encoded by marker brightness (dark = small \(N\), bright = large \(N\); see grayscale colorbar on the right). 
  Dashed vertical line marks the analytic threshold \(\delta^*_F(N_r^{\rm crit})\) computed from Corollary~\ref{cor:finite-rank-main} under that panel's $\beta$.
  The two-phase dynamic predicted by Theorem~\ref{thm:interference} and seen in Fig.~\ref{fig:residual_controls_learning} is preserved across all four prior skews; the shift in the threshold's location with $\beta$ is exactly what theory predicts. 
  }
  \label{fig:E-residual-scatter-betas}
\end{figure}

We next provide further validation for results in Fig.~\ref{fig:residual_controls_learning} by repeating experiments across different values of exponents in the power-law task prior. 
Specifically we vary $\beta$ from the set $\{0.5, 1.0, 1.5, 2.0\}$.
For each run we compute the per-task signal $s_k(U) = \mathrm{Tr}(P_U C_k) / \mathrm{Tr}(C_k)$, i.e., how well the model representation encodes information about the $k^{\text{th}}$ task, and the residual $\delta_F(U) = \sum_{k \in F} \pi_k\,(1 - s_k(U))\,\|a_k\|^2$ from the final checkpoint.
The frequent set $F$ is defined as the smallest set of tasks whose cumulative mass meets $0.8$; under \(\beta \in \{0.5, 1, 1.5, 2\}\) this
yields $|F| = \{6, 3, 2, 2\}$ respectively, reflecting how a
flatter prior spreads the loss budget across more frequent tasks.
Results are reported in Figure~\ref{fig:E-residual-scatter-betas}.
We see a precise kink similar to Fig.~\ref{fig:residual_controls_learning} when rare-task signal drops to zero once $\delta_F$ exceeds the analytic threshold \(\delta^*_F(N_r^{\rm crit})\), and rises steeply to near-unity once \(\delta_F\) falls below it. 
We also see the threshold itself shifts left as $\beta$ grows: a steeper prior makes the rare task's leading utility \(\pi_r \lambda_r\) much smaller, so a smaller residual is required to ``free up'' encoder directions that can then capture the rare task.

\subsection{Per-gap dynamics: Reproducing retention results across different injection gaps and widths}
\label{app:E-lines-multi-gap}

\begin{figure}[h]
  \centering
  \includegraphics[width=\linewidth]{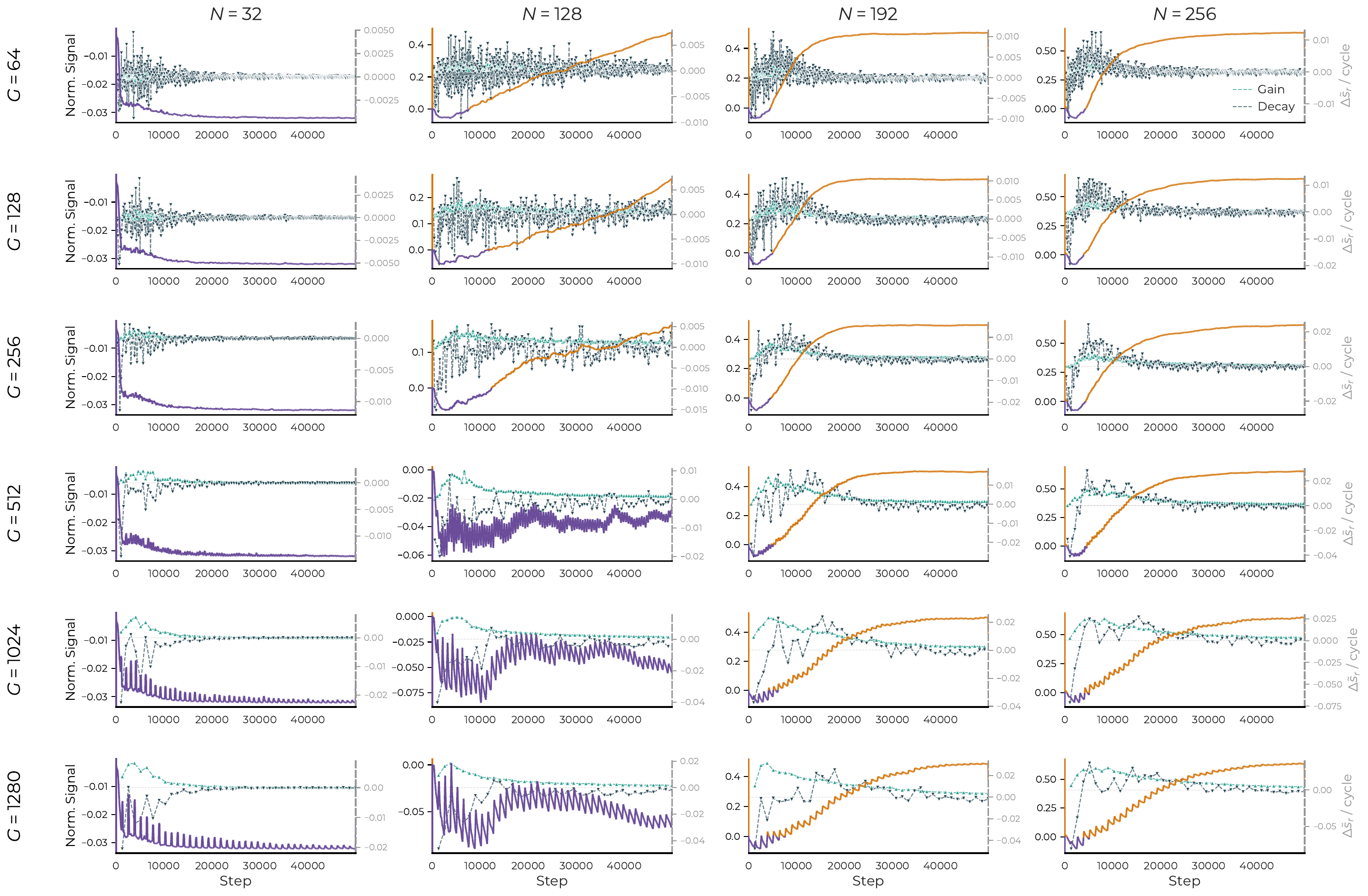}
  \caption{\textbf{Per-gap dynamics: Reproducing retention results across different injection gaps and widths.} 
  We vary the injection gap $G$ in the set $\{64, 128, 256, 512, 1024, 1280\}$ (top to bottom) and reproduce the results shown in Fig.~\ref{fig:competition}a for widths $N \in \{32, 96, 128, 192, 256\}$. 
  In each cell, the left y-axis reports the normalized rare-task signal \(\tilde s_r(U_t)\), while the right y-axis (gray) is the gain / decay curves reporting how much the signal for rare task grows vs.\ decays as a function of time. 
  We see analogous results as the main paper: larger models retain and preserve the learned signal, while smaller models require the gaps to be sufficiently small if learning is to occur at all.}
  \label{fig:E-lines-multi-gap}
\end{figure}

We reproduce results from Fig.~\ref{fig:competition}a by reporting the joint dynamics of the normalized rare-task signal \(\tilde s_r(U_t)\) and its gain / decay dynamics as a function of rare-task injection events, as shown in Fig.~\ref{fig:E-lines-multi-gap}.
Similar to results seen in the main paper, we find larger models retain and preserve the learned signal, while smaller models require the gaps to be sufficiently small if learning is to occur at all.

\subsection{Effects of Scaling Data: Learning Bottleneck Persists at Long Training Horizon}
\label{app:E-long-horizon}

In the main paper, especially Sec.~\ref{sec:categorization}, we distinguish between finite vs.\ asymptotic training.
However, most of our training runs use a budget of $100$K training iterations.
To contextualize that this budget is sufficient for the claims made in the paper, we extend training runs to $1$M steps for $6$ values of model widths $N = \{8, 16, 32, 64, 128, 256\}$, subsampling the range of widths analyzed in the main paper; the setup remains the same otherwise as Fig.~\ref{fig:phases}.
We find that results (see Fig.~\ref{fig:E-long-horizon-phase}) are stable at \emph{much} longer horizons: above-capacity tasks do not slowly close the gap given more training; instead they remain at or below the random-projection baseline indefinitely.

\begin{figure}[ht]
  \centering
  \includegraphics[width=\linewidth]{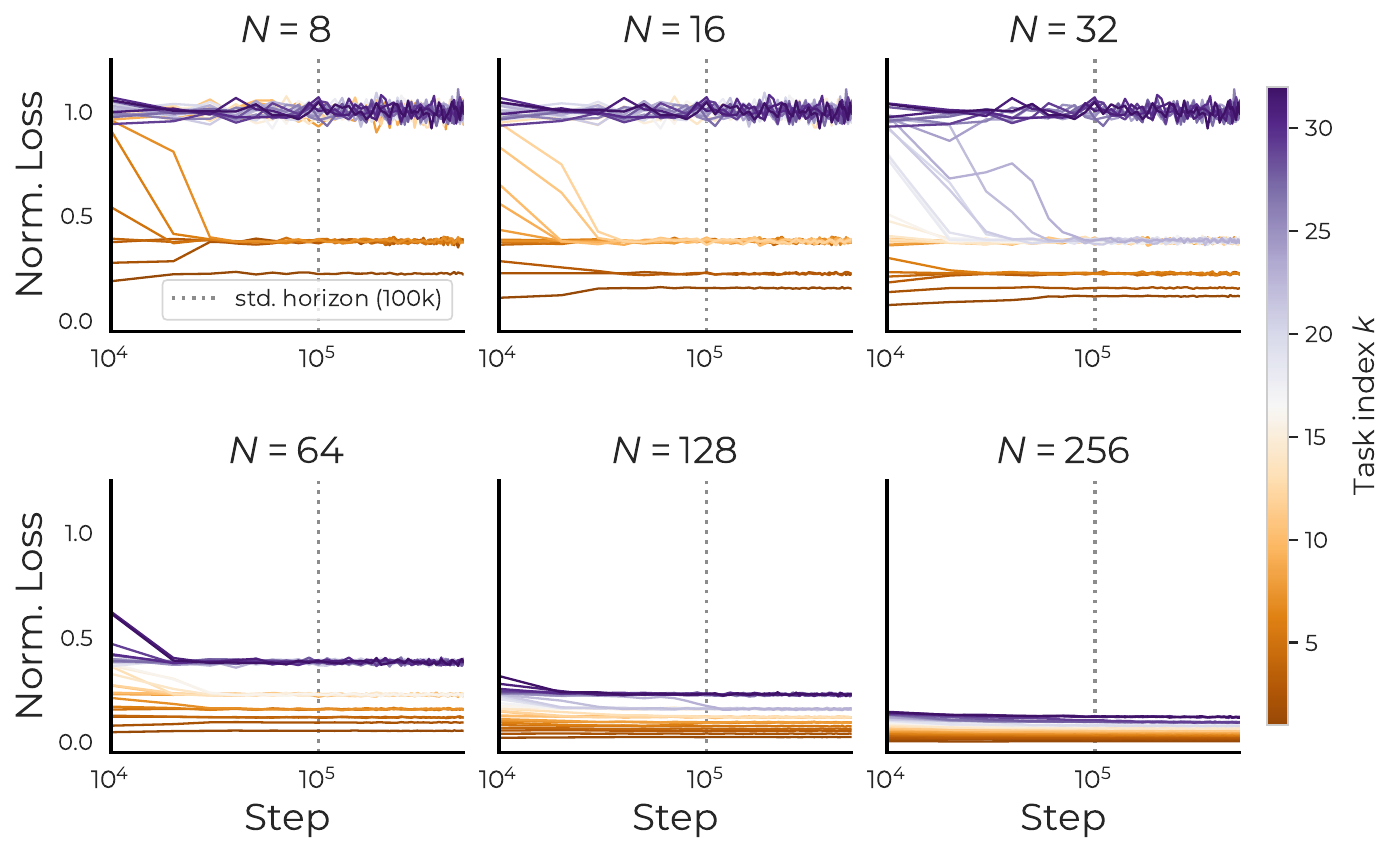}
  \caption{\textbf{Persistence of the multi-rank phase diagram at $1$M steps.} Per-task normalized loss \(\ell_k / \ell_{k,\mathrm{baseline}}\) versus training step (log-x, linear-y) for six widths $N \in \{8, 16, 32, 64, 128, 256\}$; \(\ell_{k,\mathrm{baseline}} = \|a_k\|^2 / D_t\) is the mean-predictor MSE per task. 
  Tasks colored by index from orange ($k = 1$, most frequent) to purple (\(k = 32\), rarest); vertical dotted line marks the training budget used in main paper, i.e., $100$K steps. 
  We clearly see that at every width, tasks that can fit model capacity (top-by-utility) drop near zero by the standard horizon and stay there; above-capacity tasks remain near the mean-predictor baseline and do not bend downward across longer training.
  }
  \label{fig:E-long-horizon-phase}
\end{figure}

\clearpage
\section{Further Experimental Results in OLMo Setting}
\label{app:scaling_olmo}

\subsection{Task Loss vs. General Language Modeling Loss}
\begin{figure}[!h]
    \centering
        \centering
        \includegraphics[
            width=0.5\linewidth,
            trim={0 0 0 0}, %
            clip
        ]{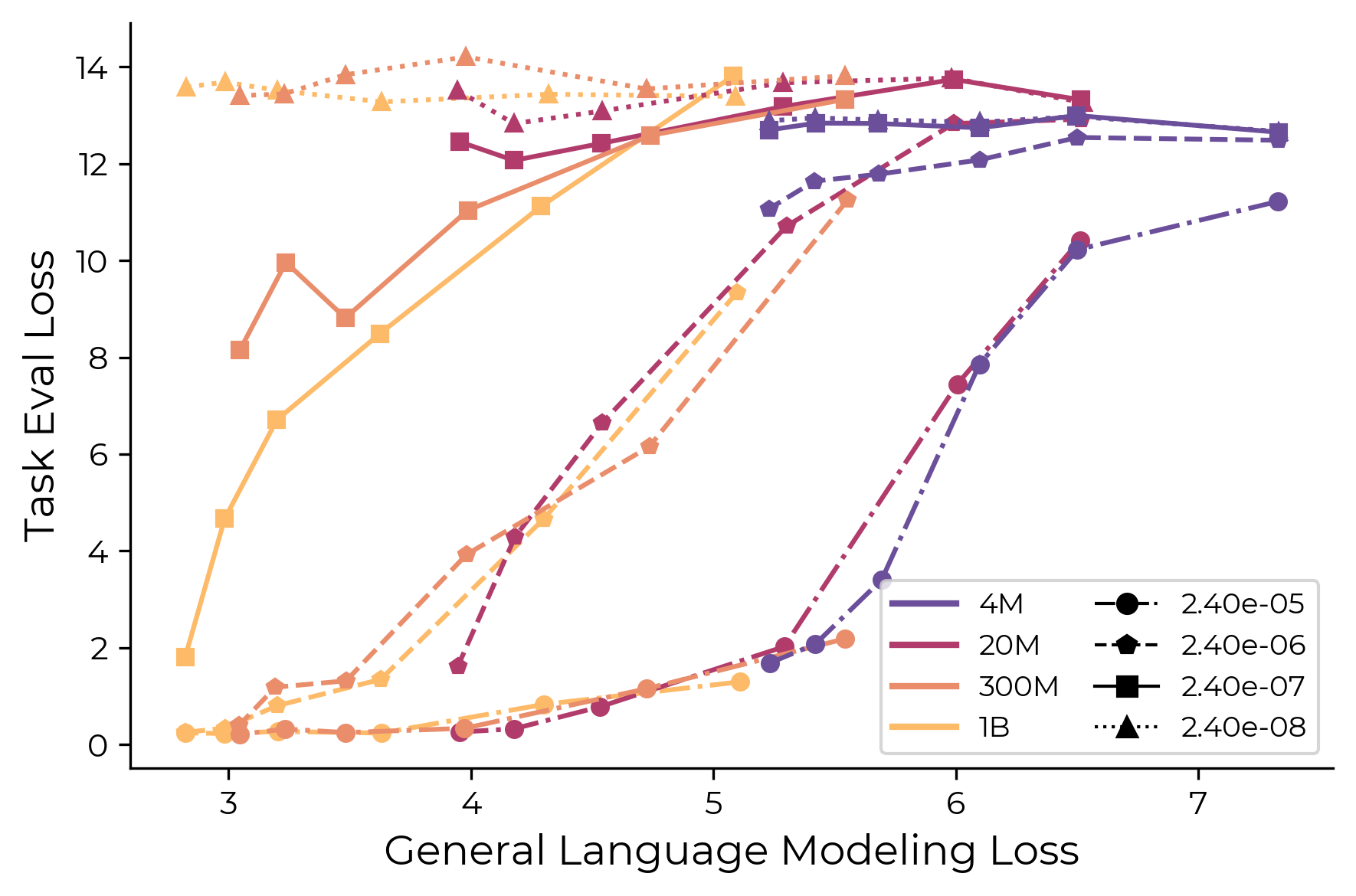}
    \caption{Task loss vs. general language modeling loss.}
    \label{fig:olmo_task_vs_lm_loss}
\end{figure}

\paragraph{Given the same language modeling loss, larger models can achieve lower task-specific loss.} We analyze the relationship between learning ``frequent tasks'' and  learning the injected task. For evaluating ``frequent tasks'', we measure the language modeling loss on the C4 validation set using a context window size of 256. As shown in Fig.~\ref{fig:olmo_task_vs_lm_loss}, when the task frequency is relatively high, i.e., from $2.4\times10^{-6}$ to $2.4\times10^{-5}$, loss curves from all model sizes follow roughly the same trajectory. This suggests that in this frequency range, model size only improves sample efficiency, but these models still have similar training dynamics~\cite{vyas2023featurelearning}. However, when the task frequency gets lower, i.e., $2.4\times10^{-7}$, larger models achieve lower task loss given the same language modeling loss. Moreover, smaller models diverge from larger models one by one, with their injected-task loss  plateauing at different values. This supports our hypothesis that for rare tasks, larger models have different learning dynamics that unlock the ability to learn rare tasks. 

\subsection{Compute-optimal Comparison}
\begin{figure}[!h]
    \centering
        \centering
        \includegraphics[
            width=\linewidth,
            trim={0 0 0 30}, %
            clip
        ]{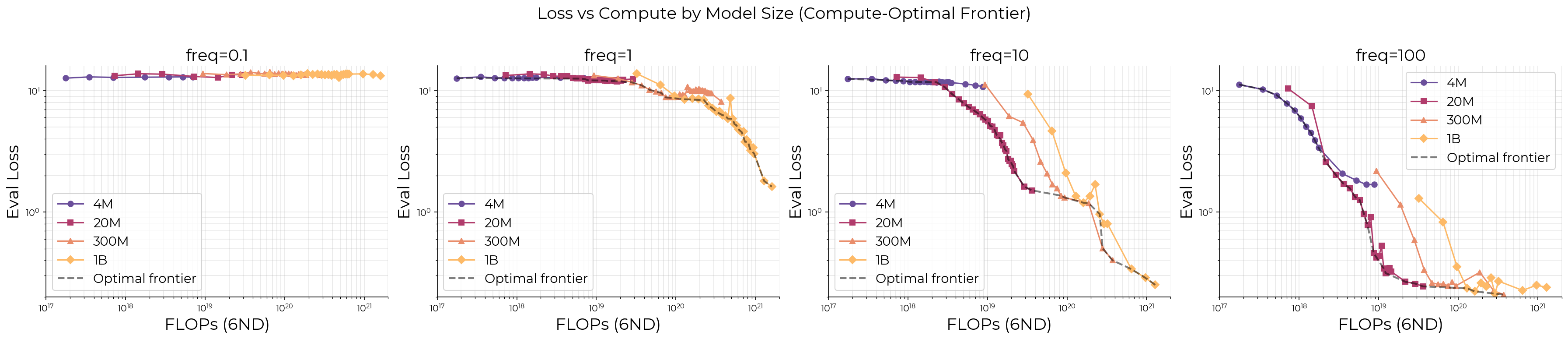}
    \caption{\tcmp\ task eval loss vs. compute by model size. Dashed black line shows the compute-optimal frontier.}
    \label{fig:olmo_compute_optimal}
\end{figure}

In Fig.~\ref{fig:olmo_freq_vs_model_size}, we compare models trained on the same amount of data. In Fig.~\ref{fig:olmo_compute_optimal}, we further show that larger models are more compute-efficient at learning low-frequency tasks. When the frequency is one task instance per batch, i.e., $2.4\times10^{-7}$, given the same compute budget, estimated as $6\times$ the number of model parameters $\times$ the number of training tokens following Chinchilla scaling laws, larger models achieve lower task loss. Moreover, consistent with our observation in Fig.~\ref{fig:olmo_task_vs_lm_loss}, smaller models initially follow the learning dynamics of larger models, but after a certain point their injected-task loss curves plateau and deviate from the larger models' curves.

\end{document}